\documentclass[journal]{IEEEtran}

\usepackage{cite}
\ifCLASSINFOpdf
   \usepackage[pdftex]{graphicx}
\else
\fi
\usepackage{amsmath}
\usepackage{mathtools,amssymb,amsthm}
\usepackage{amsfonts}
\theoremstyle{remark}
\newtheorem{remark}{\indent Remark}

\theoremstyle{definition}
\newtheorem{definition}{\indent Definition} 

\theoremstyle{plain}
\newtheorem{lemma}{\indent Lemma}
\newtheorem{theorem}{\indent Theorem}
\newtheorem{corollary}{\indent Corollary}
\usepackage{algorithm}
\usepackage[noend]{algpseudocode}
\floatname{algorithm}{Algorithm} 
\makeatletter
\let\OldStatex\Statex
\renewcommand{\Statex}[1][3]{%
  \setlength\@tempdima{\algorithmicindent}%
  \OldStatex\hskip\dimexpr#1\@tempdima\relax}
\makeatother

\ifCLASSOPTIONcompsoc
 \usepackage[caption=false,font=normalsize,labelfont=sf,textfont=sf]{subfig}
\else
 \usepackage[caption=false,font=footnotesize]{subfig}
\fi
\usepackage{url}
\usepackage{color,xcolor}
\usepackage{comment}
\usepackage{multirow} 
\usepackage{booktabs}
\usepackage{float}

\usepackage{siunitx} % \SIrange{10}{100}{\hertz}
\usepackage{microtype} % after font selection
\usepackage{hyperref}   % at last \cite{name} \href{name}{link} 
\hypersetup{
    colorlinks=true,
    linkcolor=cyan,
    linkbordercolor={0 0 1},
    anchorcolor=green,
    filecolor=gray,      
    urlcolor=red,
    citecolor=gray,
}
\usepackage[capitalise]{cleveref}

\begin{document}
% at, but, by, for, in, nor, of, on, or, the, to and 
%up, which are usually

% \title{Continuous Control for Multi-aircraft Guidance and Separation Assurance}

% efficient/low conservative 
\title{Fine-Tuned Convex Approximations of Probabilistic Reachable Sets under Data-driven Uncertainties} 

%Using Kernel Density Estimation and Mixed Integer Programming}

\author{Pengcheng~Wu,
        Sonia~Martinez,~\IEEEmembership{Fellow,~IEEE},
        Jun~Chen,~\IEEEmembership{Member,~IEEE}% <-this % stops a space
\thanks{Pengcheng Wu is with the Department of Mechanical and Aerospace Engineering, University of California San Diego, La Jolla, CA 92093, and also with the Department of Aerospace Engineering, San Diego State University, San Diego, CA 92182
        {\tt\small pcwupat@ucsd.edu,\,pwu@sdsu.edu}}%
        
\thanks{Sonia Martinez is with the Department of Mechanical and Aerospace Engineering, University of California San Diego, La Jolla, CA 92093
        {\tt\small soniamd@eng.ucsd.edu}}%
       
\thanks{Jun Chen is with the Department of Aerospace Engineering, San Diego State University, San Diego, CA 92182
        {\tt\small jun.chen@sdsu.edu}}
}

\maketitle

\begin{abstract}
This paper proposes a mechanism to fine-tune convex approximations of probabilistic reachable sets (PRS) of uncertain dynamic systems. We consider the case of unbounded uncertainties, for which it may be impossible to find a bounded reachable set of the system. Instead, we turn to find a PRS that bounds system states with high confidence. Our data-driven approach builds on a kernel density estimator (KDE) accelerated by a fast Fourier transform (FFT), which is customized to model the uncertainties and obtain the PRS efficiently. However, the non-convex shape of the PRS can make it impractical for subsequent optimal designs. Motivated by this, we formulate a mixed integer nonlinear programming (MINLP) problem whose solution result is an optimal $n$ sided convex polygon that approximates the PRS. Leveraging this formulation, we propose a heuristic algorithm to find this convex set efficiently while ensuring accuracy. The algorithm is tested on comprehensive case studies that demonstrate its near-optimality, accuracy, efficiency, and robustness. The benefits of this work pave the way for promising applications to safety-critical, real-time motion planning of uncertain dynamic systems. \\

\textbf{\textit{Note to Practitioners}---}This study is motivated by the realization of safety-critical real-time motion planning for a dynamic system under uncertainties. A popular method used to guarantee the safe operation of uncertain dynamic systems is reachability analysis. However, this method may not work well in the face of the following challenges: unbounded uncertainties, unknown distributions, generality, convexity, and efficiency. To address these issues, we first present a data-driven approach to model arbitrary unknown uncertainties and obtain a set encompassing the system states with high confidence; Then we propose an algorithm to efficiently find a tight convex polygon approximation for the set. This clearly benefits real motion planning. When considering collision avoidance in motion planning, a tight convex approximation allows a larger feasible search area which may provide a better-planned trajectory. Also, the efficiency of the algorithm ensures that the motion planning can be realized in real-time. 
\end{abstract}


% \Comment{{\color{magenta}The first keyword determines the type of reviewer and AE we get. Because of this, I'd vote for using Probabilistic Reachability. }}

% \fi
\IEEEpeerreviewmaketitle

\section{Introduction}
\medskip

% \subsection{Motivation}

The deployment of autonomous systems in urban ground and air traffic environments requires the development of new real-time collision detection and resolution algorithms~\cite{lew2022safe,xing2023collision,wen2022path}. However, widespread uncertainties in such systems raise important concerns about their safety \cite{wu2022risk}. Such uncertainties may have various causes, such as epistemic, given a lack of understanding of the underlying mechanism, the gap between simple assumptions and complex reality, or aleatoric, given the inherent randomness in the systems, or the errors of sensor measurements \cite{paudel2022higher,wu2022safety,meng2023learning,zhang2023adaptive}. A popular method that is used to guarantee the safe operation of uncertain dynamic systems is based on reachability analysis \cite{li2021comparison}. This set-based method computes the reachable set of states that are accessible by a stochastic dynamic system from all initial conditions and all admissible inputs and parameters \cite{althoff2021set}. Reachability analysis has increasingly attracted significant attention from researchers in the past decades \cite{devonport2021data}. In prior literature, accounting for the knowledge of uncertainties, many algorithms have been proposed to utilize detailed system information to figure out a reachable set. However, existing algorithms may not work well or even fail in the face of the following challenges:

1) Unboundedness: When uncertainties are unbounded, it may be impossible to figure out a bounded reachable set for an uncertain dynamic system. However, most real-world applications need a bounded reachable set due to physical limits.

% \Comment{{\color{magenta}  Why doesn't it make sense in real applications? because it does not happen? or why?}}

2) Unknown Distributions: Uncertainties are often assumed to obey Gaussian distributions in most literature. However, in practice, the uncertainties may be non-Gaussian or even empirical due to unknown system interconnections and nonlinearity \cite{dai2021conflict,han2022non}.
 
% \Comment{{\color{magenta} Actually, some people may disagree why we need uncertainties other than Gaussian! I found this in reviews.... we have to be ready to give a brief example of why this is the case: e.g. when uncertainties are due to unknown system interconnections, then these are mostly non-Gaussian and nonlinear. I'm not sure if this can be added here; it may be too much detail}}

3) Generality: To compute reachable sets, many algorithms require detailed system information given \textit{a priori}. However, in some applications like complex cyber-physical systems that are only accessible through experiments or simulations, this detailed information is unavailable. Many algorithms also rely on strong assumptions about the dynamics or uncertainties, which are often unrealistic.

4) Convexity: The reachable set computed by most algorithms is irregular or non-convex, which is hard to handle when it comes to the resolution of collision avoidance in practice.

5) Efficiency: In order to realize real-time motion planning and control of uncertain dynamic systems, the reachable set must be computed efficiently. However, many algorithms are computationally expensive.

To address the challenges, in what follows, we will first review some related literature, and then introduce the contributions of our paper.

\subsection{Literature Review}

%Researchers may come across various uncertainties when it comes to the motion planning and control of dynamic systems considering collision avoidance. For the sake of operation safety, it is critical to perform a reachability analysis, which has been widely discussed. 

A classical way to compute reachable sets for bounded uncertainties is given by performing a Hamilton-Jacobi (HJ) reachability analysis. It formulates the problem as a partial differential equation which can be solved by numerical methods. While computationally tractable in lower dimensions, the method scales poorly to higher dimensions \cite{chen2018decomposition,lew2022safe}. Also, HJ reachability analysis requires detailed information on dynamics or uncertainties and relies on strong assumptions like there are no time correlations of parameter uncertainties, which makes it unavailable in some realistic scenarios \cite{lew2022safe}.  
%yang2016multi analytical

%For bounded uncertainties, the reachable set completely contains all possible states of an uncertain system. However, for unbounded ones,
When uncertainties are unbounded, it is often impossible to guarantee collision avoidance with $100\%$ certainty as the reachable set becomes unbounded. Instead, researchers seek to identify a trajectory with a probability of collision bounded by a certain risk bound. One direct way to realize this goal is to evaluate the probability of collision of a generated candidate trajectory. The candidate trajectory will not be adopted until the probability of collision is below the risk bound. However, there is no closed-form expression to evaluate the probability of collision, which poses difficulties to this approach \cite{wu2023online, zou2021collision}. Another indirect but effective method consists of computing a probabilistic reachable set (PRS) instead of an exact one\cite{yang2016multi,hewing2020recursively}. 
A PRS is a set of possible states that a system can reach with a certain probability. By means of this, the stochastic constraint on the system, which sets a limit on the collision probability, can be converted to a deterministic constraint by which the generated trajectory does not intersect with the PRS \cite{blackmore2010probabilistic}. An advantage of this method is that it is not only applicable for unbounded uncertainties but also for bounded ones. As the risk bound approaches $0\%$, the PRS becomes a superset of the reachable set \cite{yang2016multi}. In this paper, we adopt the PRS approach. 
% probreachset: hewing2020 yang2016multi,devonport2021data, lew21,liebenwein2018sampling

% probabilistic collision detection

Many existing works assume a Gaussian distribution of the uncertainty \cite{zhu2019chance,lathrop2021distributionally}. Blackmore~et~al. \cite{blackmore2010probabilistic} use convex polygons (or polytopes) determined by the mean and covariance of the Gaussian distribution as a proxy of PRS. Wu~and~Chen~et~al. \cite{wu2022safety, wu2022risk} directly use the level sets of the Gaussian distribution to serve as PRS. However, in many scenarios uncertainties are of non-Gaussian type \cite{dai2021conflict}. Even in that case, some researchers still make use of Gaussian distributions to approximate non-Gaussian ones \cite{boone2022spacecraft}. Unfortunately, this approach can not guarantee that the risk of a generated trajectory is still within a given bound \cite{han2022non}. Instead, Aoude~et~al. \cite{aoude2013probabilistically} apply a particle method to evaluate the probability of collision in the process of motion planning. Calafiore~and~Campi \cite{calafiore2006scenario} use a scenario approach by which stochastic constraints are sampled to obtain a convex optimization problem whose solution is feasible for the original stochastic constraints with high confidence. However, all of the above typically require a large number of samples for probabilistic assurance, which is neither feasible nor computationally efficient. Alternatively, Han~and~Jasour~et~al. \cite{han2022non} employ a moment-based method to characterize non-Gaussian PRS. However, moment information is assumed to be known \textit{a priori} in their works, which is not realistic as this information can only be learned online from sensor measurements \cite{lefkopoulos2019using}. Moreover, the PRS is often irregular, which is challenging to handle in practice.

% {\color{blue} Additionally, they are often applied without the guarantee of satisfying stochastic constraints \cite{han2022non, lefkopoulos2019using,devonport2020estimating}. }

% \Comment{{\color{magenta} I would just focus on the previous probabilistic approaches, not work that uses it without any guarantee satisfaction, this way we save some space. (So I would omit the blue part)}}

Such drawbacks motivate the study of data-driven reachability analysis, using data obtained from experiments or simulations \cite{kehoe2015survey,lefkopoulos2019using,han2022data}. Devonport~et~al. \cite{devonport2021data} use the level sets of a Christoffel function to estimate the reachable sets. The Christoffel function is obtained from the samples of the arbitrary unknown uncertainties, without requiring any detailed information \textit{a priori}. However, the level sets estimating the reachable sets may be irregular or non-convex, which is hard to deal with in practice. Motivated by this, Lew~et~al. \cite{lew2022safe} employ a conservative convex hull of data samples for reachable set approximation. The number of sides of the convex hull is unspecified, leading to an arbitrary number of constraints in subsequent design problems. A bounding box can enclose data samples with a fixed number of sides but is more conservative \cite{ericson2004real}. All the works above are intended to estimate reachable sets given bounded uncertainties and are not accurate approximations of the reachable sets because lacking uncertainty quantification. For arbitrary probability distributions, including unbounded ones, a natural way of approximating reachable sets is via empirical or data-fitted probability density functions (PDF), whose level sets serve as the reachable sets. In this paper, we employ this approach considering kernel density estimators (KDE) \cite{wand1994kernel}. Our goal is to address important limitations of adopting such a method such as the lack of closed-form expressions for KDEs, the irregular shapes of their level sets, and the conservativeness and inaccuracy of taking convex hulls or bounding boxes of the level sets. Some researchers have considered dynamic systems with time delays \cite{ding2023intermittent}, which is not done in this manuscript. The extension of the proposed methodology to scenarios where time delay matters is left for future work.

\subsection{Contributions and Organization}
%Motivated by the requirement of realizing safety-critical real-time motion planning for uncertain systems, 
This paper proposes an efficient convex approximation of PRS of uncertain dynamic systems while ensuring their accuracy. Our approach  consists of two parts: (i) A data-driven uncertainty quantification; (ii) An optimization problem formulation to find the convex approximation. 

The major contributions of this paper are the following:

% 1) The probabilistic reachable set is introduced to address unbounded uncertainties. As an extension of the traditional reachable set, it provides a unified way to work for both bounded and unbounded uncertainties. Through the introduction of the probabilistic reachable set, the connection between reachability analysis and chance constraint formulation can be revealed. 

1) We present an algorithm that efficiently finds the PRS via KDE accelerated by a fast Fourier transform (FFT). FFT-based KDE is customized to estimate the PDF under arbitrary unknown uncertainties and therefore its level sets can serve as a proxy for the PRS of the uncertainties. As an efficient data-driven approach, FFT-based KDE can learn from data online, and thus it neither requires \textit{a priori} information nor relies on strong assumptions; 

% require relatively small samples than pure sampling
% no closed-form expression KDE
% safety reliablity, uncertainty non-gaussian distribution
% available accessible

2) We formulate an optimization problem of mixed integer nonlinear programming (MINLP), resulting in an optimal $n$ sided convex polygon that approximates the PRS. As opposed to a convex hull, users can arbitrarily customize the number of sides. The optimization problem is established not by using data samples directly but by weighted grid points from KDE, which makes it tractable. Compared with bounding boxes, the result obtained by this approach is more tight and accurate; 

% optimal convex and connected, others nonlinear nonconvex nonoptimal, 

% 3) When it comes to MINLP formulation, some important concepts are created independently and rigorous proofs are presented detailedly to show that particular constraints restrict all the formed graphs to enclosing $n$ sided convex polygons which contain a particular point.

% \Comment{{\color{magenta} expand on details on this paragraph: what important concepts are created independently? say we introduce the concept of... and be more precise on what results you have}}

% {\magenta{Let's tone down this novelty just in case: we don't know all of the graph theory. This may be in some advanced book or have been defined somewhere.}}

% compliance with guaratee

3) We develop a heuristic algorithm to efficiently solve the formulated MINLP problem. This algorithm performs weighted sampling to select the representative grid points from all the grid points from KDE, which significantly reduces computational complexity. Comprehensive case studies demonstrate that this algorithm enjoys the benefits of accuracy, efficiency, near-optimality, and robustness while providing a convex approximation for the PRS.

The rest of this paper is organized as follows. In \cref{Problem_Statement}, we formally define the concept of PRS and state the problem as providing a convex approximation for it. In \cref{Evaluation_of_Probabilistic_Reachable_Set}, we use KDE accelerated by FFT to model arbitrary probability distributions and present an algorithm to find the PRS. In \cref{MINLP_Formulation_for_Probabilistic_Reachable_Set_Approximation}, we formulate an MINLP optimization framework, the solution of which leads to a convex polygon approximation of the PRS. In \cref{Solution_Method}, we develop a heuristic algorithm to efficiently solve the formulated MINLP. In \cref{Main_Results}, we conduct comprehensive case studies to demonstrate the performance of our proposed algorithm. In \cref{Conclusion}, we draw conclusions and make suggestions for future research. In addition, rigorous proofs of some propositions in this paper are given in \cref{Appendix} Appendix.

% end of introduction section

% new section problem statement
\section{Problem Statement} 
\label{Problem_Statement}
\medskip

In this section, we take into account an uncertain discrete-time dynamic system and introduce the concept of reachable set and PRS of the system. Then, we present the goal of this paper.

Consider a general discrete-time dynamic system of the form
\begin{equation}
    \boldsymbol{x}_{k+1}=\boldsymbol{f}\left(\boldsymbol{x}_k, \boldsymbol{u}_k, \boldsymbol{\theta}_k, \boldsymbol{w}_k\right),
    \label{eqn:dynamic_system}
\end{equation}
where $\boldsymbol{x}_k \in \mathbb{R}^n$ is the system state, $\boldsymbol{u}_k \in \mathcal{U}_k$ is an uncertain control input, $\boldsymbol{\theta}_k \in \Theta$ is an uncertain parameter, $\boldsymbol{w}_k \in \mathbb{W}$ is an uncertain disturbance, and $\boldsymbol{x}_0 \in \mathcal{X}_0$ is an initial state. Unlike most prior works \cite{lew2022safe}, in this paper, the sets $\mathcal{U}_k \subset \mathbb{R}^m, \Theta \subset \mathbb{R}^p, \mathbb{W} \subset \mathbb{R}^q$, and $\mathcal{X}_0 \subset \mathbb{R}^n$ can be unbounded. Given the uncertainties of $\boldsymbol{x}_k, \boldsymbol{u}_k, \boldsymbol{\theta}_k, \boldsymbol{w}_k$, then $\boldsymbol{x}_{k+1}$ is a random vector which obeys an unknown probability distribution. % nonlinear general

Formally, the concept of reachable set is introduced as follows.

% \textbf{Definition 1 (\textit{Reachable Set})}: 

\begin{definition}[Reachable Set \cite{alanwar2021data}]
    At time $T$, the \textit{reachable set} $\mathcal{X}_{T}$ of the dynamic system in \cref{eqn:dynamic_system} is defined to be
    \begin{equation*}
    \begin{aligned}
    \mathcal{X}_{T} := \big \{ & \boldsymbol{x}_{T} \in \mathbb{R}^n:  \boldsymbol{x}_{k}=\boldsymbol{f}\left(\boldsymbol{x}_{k-1}, \boldsymbol{u}_{k-1}, \boldsymbol{\theta}_{k-1}, \boldsymbol{w}_{k-1} \right), \\
    & \boldsymbol{x}_{0} \in \mathcal{X}_{0},  \boldsymbol{u}_{k-1} \in \mathcal{U}_{k-1}, \boldsymbol{\theta}_{k-1} \in \Theta, \boldsymbol{w}_{k-1} \in \mathbb{W}, \\
    & k \in \{1,\ldots,T\} \big \}, 
    \end{aligned}
    \end{equation*}
    where $\boldsymbol{x}_{0}$ is the initial state and $\mathcal{X}_{0}$ is the initial set.
    \label{def:rset}
    \end{definition}

% \Comment{{\color{magenta} Shouldn't the previous set be equal to $\mathcal{X}_k$? (you are defining "a reachable set") if so, start by defining it via an equality, as "$\mathcal{X}_k =$". Otherwise, according to this definition, trivially $\mathbb{R}^n$ is a reachable set. Is this what we want? This is for $k \ge 1$, so indicate so in the equation above, and say what is $\mathcal{X}_0$.}}

% \footnote{\color{magenta} I'm not sure why the reachable set is not unique. How can there be more than one?}}

% $\mathcal{X}_{K}$ contains all possible reachable states $\boldsymbol{x}_{K}$ at time $k$, for any possible states $\boldsymbol{x}_{k-1}$ at time $k-1$, any possible control inputs $\boldsymbol{u}_{k-1}$ without violating actuator constraints, any model parameters $\boldsymbol{\theta}_{k-1}$, and any disturbances $\boldsymbol{w}_{k-1}$. 

%The goal of performing reachability analysis is to figure out a reachable set $\mathcal{X}_{k}$. 

Observe that, since the sets $\mathcal{U}_{k-1}, \Theta, \mathbb{W}$, and $\mathcal{X}_0$ are not limited to bounded sets, the reachable set $\mathcal{X}_{T}$ may be unbounded. Hence, it may be infeasible to find a bounded set in which the state is guaranteed to lie. Instead, we expect to find a bounded set such that the probability of the state lying in the bounded set is greater than a confidence level. To this end, we formally introduce the concept of the PRS as follows.

\begin{definition}[Probabilistic Reachable Set (PRS) \cite{hewing2020recursively}]
    At time $k$, a bounded set $\Tilde{\mathcal{X}}_{k}$  is defined to be a  \textit{probabilistic reachable set} (PRS) of the dynamic system in \cref{eqn:dynamic_system} at confidence level $\alpha$ if and only if
\begin{equation*}
    \operatorname{Pr}(\boldsymbol{x}_{k} \in \Tilde{\mathcal{X}}_{k}) \ge \alpha. 
\end{equation*} 
% \label{def:prset}
\end{definition}

A PRS is a set of possible states that a system can reach with a certain probability. Note that $\Tilde{\mathcal{X}}_{k}$ may not exist. When $\alpha = 100\%$, if there exists a PRS $\Tilde{\mathcal{X}}_{k}$, then $\Tilde{\mathcal{X}}_{k}$ is a superset of the reachable set.

% \footnote{\magenta{Maybe this is why you have nonuniqueness? A reachable set can differ from others by a nonzero measure set. If that is the case, it would be preferable to } }

In this paper, we assume the dynamic system in \cref{eqn:dynamic_system} is two-dimensional, and leave the general dimension case for future work. The goal of this paper is twofold: 1) We aim to put forward a data-driven approach to model the arbitrary unknown uncertainties and obtain a PRS $\Tilde{\mathcal{X}}_{k}$ of the dynamic system; 2) Motivated by the realization of safety-critical real-time motion planning for uncertain systems, we aim to approximate the PRS $\Tilde{\mathcal{X}}_{k}$ by means of a bounded set that satisfies: i) \textbf{Convexity}: The boundary of the set is an $n$ sided convex polygon; ii) \textbf{Efficiency}: The computation of this set is tractable; iii) \textbf{Accuracy}: The probability of the state $\boldsymbol{x}_{k}$ lying in the set is close to the prescribed confidence level $\alpha$; iv) \textbf{Optimality}: The area of the set is as small as possible (not too conservative) while ensuring accuracy.

% end of section

\section{PRS Identification} 
\label{Evaluation_of_Probabilistic_Reachable_Set}
\medskip

In this section, we model an arbitrary unknown probability distribution through FFT-based KDE, and develop an online algorithm to find a PRS at confidence level $\alpha$.

\subsection{FFT-Based KDE}

We first briefly review FFT-based KDE. The concept of KDE is introduced below.

\begin{definition}[Kernel Density Estimator (KDE) \cite{wand1994kernel}]
    Let $\boldsymbol{x}_k \in \mathbb{R}^2, k \in \{1,\ldots,M\}$ be $M$ data samples drawn from a $2$-variate probability distribution given by a PDF $f: \mathbb{R}^2 \mapsto \mathbb{R}$. The \textit{kernel density estimator} (KDE) $\hat{f}: \mathbb{R}^2 \mapsto \mathbb{R}$ to approximate the PDF $f$ is defined to be 
\begin{equation*}
  \hat{f}(\boldsymbol{x}) = \frac{1}{M} \sum_{k=1}^{M} K(\boldsymbol{x} - \boldsymbol{x}_k),
  % \label{ori_kde}
\end{equation*} 
where the function $K: \mathbb{R}^2 \mapsto \mathbb{R}$ is the \textit{kernel function} which is a symmetric $2$-variate density function. 
    \label{ori_kde}
\end{definition}

The choice of the kernel function $K$ is not critical to the accuracy of KDE, so we use the standard $2$-variate Gaussian kernel in this paper
\begin{equation}
    K(\boldsymbol{x}) = (2 \pi)^{-1}\text{det}(\mathbf{H})^{-\frac{1}{2}} \exp(-\frac{1}{2}\boldsymbol{x}^{\top}\mathbf{H}^{-1}\boldsymbol{x}),
    \label{kernel_gaussian}
\end{equation}
where $\mathbf{H}$ is a symmetric and positive definite bandwidth matrix. The choice of bandwidth matrix is crucial to the accuracy of KDE approximating PDF. In this paper, we choose the bandwidth matrix using Silverman's rule \cite{silverman2018density}.

Consider a mesh of $N^2$ \textit{grid points} $\{ (x_i, y_j) \in \mathbb{R}^2: i,j \in \{1,\ldots,N\} \}$ equally spaced on the plane $\mathbb{R}^2$, the boundary of which encompasses all $M$ data samples $\boldsymbol{x}_k \in \mathbb{R}^2, k \in \{1,\ldots,M\}$. The mesh of grid points can be viewed as a map $\boldsymbol{g}: \{(i,j): i,j \in \{1,\ldots,N\}\} \mapsto  \{(x_i, y_j) \in \mathbb{R}^2: i,j \in \{1,\ldots,N\} \} $ given by 
\begin{equation}
    (x_i, y_j) = \boldsymbol{g}(i,j), \; i,j \in \{1,\ldots,N\},
\label{mesh_grids}
\end{equation}
where $(i,j)$ is the index of a grid point whose coordinate is $(x_i, y_j)$. Throughout the rest of this paper, a coordinate $\boldsymbol{g}(i,j)$ is abbreviated as $\boldsymbol{g}_{ij}$.

The KDE $\hat{f}$ given in \cref{ori_kde} can be evaluated on the mesh of $N^2$ grid points $\boldsymbol{g}$ in \cref{mesh_grids}, which yields
\begin{equation*}
  \hat{f}_{ij} := \hat{f}(\boldsymbol{g}_{ij})  = \frac{1}{M} \sum_{k=1}^{M} K(\boldsymbol{g}_{ij} - \boldsymbol{x}_k), \; i,j \in \{1,\ldots,N\} .
  % \label{wei_kde}
\end{equation*}

One way to significantly accelerate the evaluation of KDE is \textit{linear binning}. Instead of placing the kernel functions $K$ on $M$ data samples, we can place them on $N^2$ grid points $\boldsymbol{g}$ weighted by grid counts $c_{uv}, u,v \in \{1,\ldots,N\}$. A grid count $c_{uv}$ is a weight chosen to represent the amount of data samples near a grid point $\boldsymbol{g}_{uv}$. To obtain $c_{uv}$, we can go through every data sample and assign weights to its neighbouring grid points, as shown in \cref{fig:binned_kde}.

\begin{figure}[H] %[!h]  [hbt!]
\centering
\includegraphics[width=.35\textwidth]{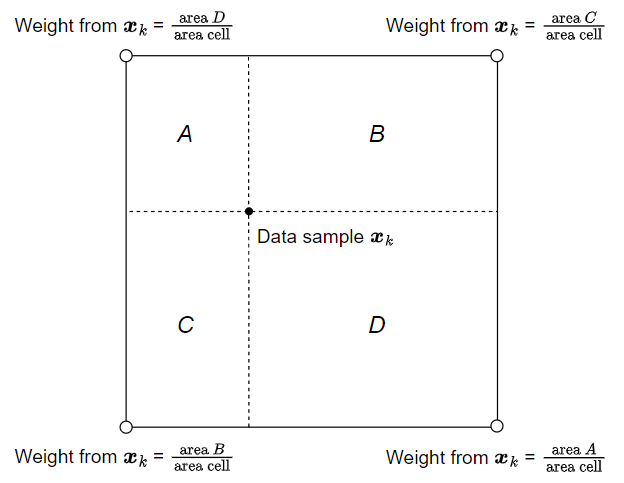}
\caption{Graphical representation of $2$-variate linear binning} 
\label{fig:binned_kde}
\end{figure}

Using the strategy of linear binning, the \textit{binned} KDE $\tilde{f}: \mathbb{R}^2 \mapsto \mathbb{R}$ to approximate the KDE $\hat{f}$ in \cref{ori_kde} is given by 
\begin{equation}
    \tilde{f}(\boldsymbol{x}) = \frac{1}{M} \sum_{u=1}^{N} \sum_{v=1}^{N} K(\boldsymbol{x} - \boldsymbol{g}_{uv})c_{uv},
    \label{binned_kde_formula}
\end{equation}
where $c_{uv}$ is the grid count of the grid point $\boldsymbol{g}_{uv}$ indexed by $(u,v)$.

Evaluating the binned KDE $\tilde{f}$ in \cref{binned_kde_formula} on the mesh of $N^2$ grid points $\boldsymbol{g}$ yields
\begin{equation}
\begin{aligned}
        & \tilde{f}_{ij} := \tilde{f}(\boldsymbol{g}_{ij}) = \frac{1}{M} \sum_{u=1}^{N} \sum_{v=1}^{N} K(\boldsymbol{g}_{ij} - \boldsymbol{g}_{uv})c_{uv}, \\ 
        & i,j \in \{1,\ldots,N\} ,
\end{aligned}
    \label{discretized_binned_kde}
\end{equation}
which shows that the matrix $\tilde{\boldsymbol{f}}:=[\tilde{f}_{ij}]$ is the discrete convolution of $\boldsymbol{c}:= [c_{uv}]$ and $\boldsymbol{k}:= [K(\boldsymbol{g}_{uv})]$. 

The discrete convolution can be computed by FFT, which significantly saves computational time. Let $\boldsymbol{C}$ and $\boldsymbol{K}$ be FFT of $\boldsymbol{c}$ and $\boldsymbol{k}$, and $\tilde{\boldsymbol{F}}$ be the element-wise product of $\boldsymbol{C}$ and $\boldsymbol{K}$. Then $\tilde{\boldsymbol{f}}$ can be extracted from the inverse FFT of $\tilde{\boldsymbol{F}}$. By doing so, we can obtain the binned KDE $\tilde{f}_{ij}, \; i,j \in \{1,\ldots,N\}$ in \cref{discretized_binned_kde} to approximate the PDF $f$ of an arbitrary unknown probability distribution in real-time \cite{wu2023online, wu2022online}.

% impact of N_ds, alpha and efficiency, accuracy demonstration

\subsection{Online Computation of PRS}

For real-time motion planning of uncertain dynamic systems, we propose an efficient data-driven algorithm using FFT-based KDE to capture PRS of an unknown static uncertainty whose data samples are gradually observed. See \textbf{\cref{alg:kde}}.

% We hope to develop a data-driven algorithm to capture an unknown uncertainty. This algorithm is intended for real-time motion planning of uncertain dynamic systems. In practice, the data samples of the unknown uncertainty is gradually observed. Thus, the information about uncertainty is changing over time and the distribution captured by the algorithm is also updated with more data samples coming in. This poses a great challenge requiring that the proposed algorithm must run in real-time to carry out re-computation every time newly observed data samples are added.

% To address this challenge, we propose an algorithm based on FFT-based KDE to compute in real-time a PRS at confidence level $\alpha$ for an arbitrary unknown probability distribution $f$. 
In the pseudo-code of \cref{alg:kde}, the symbol $\epsilon$ represents the tolerance of confidence error. From Line 10 through 17, FFT is employed to accelerate the computation of the discrete convolution of $\boldsymbol{c}$ and $\boldsymbol{k}$ to obtain the binned KDE $\tilde{\boldsymbol{f}}$ evaluated on the mesh of $N^2$ grid points $\boldsymbol{g}$. From Line 18 through 33, a bisection search is implemented to find the critical binned KDE value $C^{\text{kde}}$ at which the confidence level $\alpha$ is achieved. The level set of the KDE corresponding to $C^{\text{kde}}$ is the PRS that we aim to find. 

As the number of the collected data samples increases, KDE captured from those data samples will be convergent to the true PDF of the unknown uncertainties. However, considering more data samples requires a longer evaluation time. \cref{fig:alg1} shows how increasing data samples influence the evaluation time of identifying a PRS at confidence level $95\%$ on a mesh of $128^2$ grid points through implementing \cref{alg:kde}. As the number of data samples increases, the proposed \cref{alg:kde} can rapidly recompute. For example, when the number of data samples is within the range of $10^3$ to $10^4$, the evaluation time of implementing \cref{alg:kde} is always under \SI{0.01}{s}. This result demonstrates the computational power of our proposed \cref{alg:kde} and thus it can handle streaming data in a reasonable way. However, it cannot guarantee to deal with uncertainties that can evolve with time, and this will be part of our future work.

Moreover, note that the shape of a PRS obtained in this way is often irregular, which precludes its use in subsequent applications requiring an optimal design. This will be solved in the next section.

\begin{figure}[H] %[!h]  [hbt!]
\centering
\includegraphics[width=.4\textwidth]{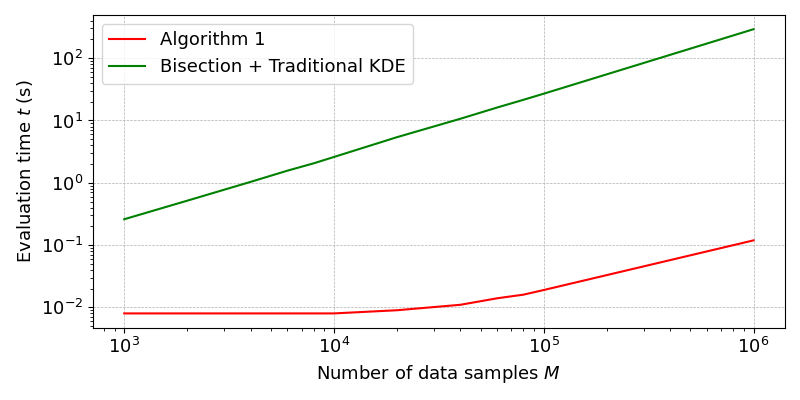}
\caption{Evaluation time of Algorithm 1 with respect to the increasing number of data samples} 
\label{fig:alg1}
\end{figure}

% \begin{minipage}{10cm}
\begin{algorithm} [!h]
\caption{PRS Identification Algorithm}  
\label{alg:kde}
\begin{algorithmic}[1] 
    \Function{GenDS}{$M$}
        \State Generate $M$ data samples $ \boldsymbol{x}_k, k \in \{1,\ldots,M\} $
        \State \Return $\boldsymbol{x}_k$
    \EndFunction
    \Statex
    
    \Function{MeshGrid}{$\boldsymbol{x}_k$, $N$}
        \State Get $x_{\min}, x_{\max}, y_{\min}, y_{\max}$ from data samples $\boldsymbol{x}_k$
        \State $\boldsymbol{h} = \operatorname{linspace}(x_{\min}, x_{\max}, N)$
        \State $\boldsymbol{v} = \operatorname{linspace}(y_{\min}, y_{\max}, N)$
        \State $\boldsymbol{g} =$  Cartesian product of $\boldsymbol{h}$ and $\boldsymbol{v}$
        \State \Return $\boldsymbol{g}$
    \EndFunction
    \Statex
    
    \Function{FFTKDE}{$\boldsymbol{x}_k$, $\boldsymbol{g}$}
        \State Obtain grid counts $\boldsymbol{c}$ for all grid points $\boldsymbol{g}$ using $\boldsymbol{x}_k$
        \State Evaluate kernel functions $\boldsymbol{k}$ on all grid points $\boldsymbol{g}$
        \State $ \boldsymbol{C} = \operatorname{FFT}(\boldsymbol{c})$
        \State $ \boldsymbol{K} = \operatorname{FFT}(\boldsymbol{k}) $
        \State $\boldsymbol{F} = $ the element-wise product of $\boldsymbol{C}$ and $\boldsymbol{K}$
        \State $ \tilde{\boldsymbol{f}} = \operatorname{iFFT}(\boldsymbol{F})$
        \State \Return $\tilde{\boldsymbol{f}}$
    \EndFunction
    \Statex
    
    \Function{BisecSearch}{$\tilde{\boldsymbol{f}}$, $\alpha$, $\epsilon$}
        \State $low = \min(\tilde{\boldsymbol{f}})$
        \State $up = \max(\tilde{\boldsymbol{f}})$
        \While{$low < up$}
            \State $mid = (low + up) / 2$
            \State $\boldsymbol{z}^{\text{bin}} = (\tilde{\boldsymbol{f}} \ge mid) \cdot 1$
            \State $\boldsymbol{z}^{\text{mix}} = $ element-wise product of $\boldsymbol{z}^{\text{bin}}$ and $\tilde{\boldsymbol{f}}$
            \State $pr = \operatorname{sum}(\boldsymbol{z}^{\text{mix}}) / \operatorname{sum}(\tilde{\boldsymbol{f}})$
            \If{ $\operatorname{abs}(pr - \alpha) \le \epsilon$ }
                \State $C^{\text{kde}} = mid$
                \State \Return $\boldsymbol{z}^{\text{bin}}$, $C^{\text{kde}}$
            \ElsIf{$pr < \alpha$}
                \State $up = mid$
            \Else
                \State $low = mid$
            \EndIf
        \EndWhile
        \State \Return $Failure$
    \EndFunction
    \Statex

    \Function{FindPRS}{$M$, $N$, $\alpha$, $\epsilon$}
        \State $\boldsymbol{x}_k =$ \Call{GenDS}{$M$}
        \State $\boldsymbol{g} = $ \Call{MeshGrid}{$\boldsymbol{x}_k$, $N$}
        \State $\tilde{\boldsymbol{f}} = $ \Call{FFTKDE}{$\boldsymbol{x}_k$, $\boldsymbol{g}$}
        \State $\boldsymbol{z}^{\text{bin}}$, $C^{\text{kde}} = $ \Call{BisecSearch}{$\tilde{\boldsymbol{f}}$, $\alpha$, $\epsilon$}
        \State \Return $\boldsymbol{g}$, $\tilde{\boldsymbol{f}}$, $\boldsymbol{z}^{\text{bin}}$, $C^{\text{kde}}$
    \EndFunction
    \Statex
    
    \State \Call{FindPRS}{$M$, $N$, $\alpha$, $\epsilon$}
            
\end{algorithmic}  
\end{algorithm}
% \end{minipage}

% \footnote{\magenta{You should give a short name to the algorithm. In any case, refer to it as "A"lgorithm, in capitals.}}

% bisection search, continuous, ordered data structure, monotone of pdf function

% \begin{figure}[!h]
% \centering
% \includegraphics[width=.35\textwidth]{figs/binned_kde.png}
% \caption{Assign weights to grid points \cite{wand1994kernel}}
% \label{graph_binned_kde}
% \end{figure}

%To do this, we develop an algorithm to evaluate the probabilistic reachable set at confidence level $\alpha$ online, as shown in \textbf{\cref{alg:kde}}. 

% $M$ is the number of data samples, $N$ is the number of grid points in each dimension, $\alpha$ represents the confidence level, $\epsilon$ represents the tolerance of confidence error, and $g$ represent grid points, which can also be understood as a bijective map or assignment between the index of grid points and their coordinates. In this way, for a bivariate probability distribution, we have $g: (i,j) \mapsto (x_i, y_j), \;  i,j \in \{1,\ldots,N\}$. 

% condition continuous, ordered

%  end section

% \newpage

\section{MINLP Formulation for PRS Approximation}
\label{MINLP_Formulation_for_Probabilistic_Reachable_Set_Approximation}
\medskip

In this section, we formulate the convex approximation problem stated in \cref{Problem_Statement} as an MINLP optimization problem for the PRS of a two-dimensional system subject to arbitrary uncertainties. % given a confidence level

Recall that, after evaluating \cref{discretized_binned_kde} through the implementation of \cref{alg:kde}, we have obtained that each grid point $\boldsymbol{g}_{ij}, \; i,j \in \{1,\ldots,N\}$ has a binned KDE value $\tilde{f}_{ij}, \; i,j \in \{1,\ldots,N\}$,  respectively. We define a normalized weight matrix $\boldsymbol{w}:= [w_{ij}]$ whose element is
\begin{equation}
    w_{ij} := \frac{\tilde{f}_{ij}}{\sum\limits_{i=1}^N \sum\limits_{j=1}^N \tilde{f}_{ij}}, \quad i,j \in \{1,\ldots,N\}.
    \label{eqn:normw}
\end{equation}
% which satisfies
% \begin{align*}
%     & 0 \le w_{ij} \le 1, \\
%     & \sum\limits_{i=1}^N \sum\limits_{j=1}^N w_{ij} = 1.
% \end{align*}

The MINLP will be established using these weights, and not with the original data samples. This makes it more tractable because the number of grid points is far less than the number of data samples, significantly reducing the number of decision variables and constraints.

\subsection{Problem Description and Objective Function Formulation}
\label{subsecdp}
The optimization goal is to find a minimum $n$ sided convex polygon to approximate the PRS efficiently and accurately, by determining which weighted grid points should lie inside the polygon. 

% The following \cref{fig:digminlp} helps illustrate this.
% \begin{figure}[H] %[!h]  [hbt!]
% \centering
%     \includegraphics[width=.25\textwidth]{figs/digminlp.png}
% \caption{Solution to the MINLP optimization problem} 
% \label{fig:digminlp}
% \end{figure}

To this end, we introduce three types of decision variables to formulate the MINLP optimization problem:
% \begin{enumerate}
\begin{itemize}
    \item $2n$ continuous variables $ a_k, b_k \in \mathbb{R}, \;  k \in \{ 1,\ldots,n \} $;  
    
    \item $N^2n$ binary variables $l_{ij}^k \in \{0,1\}, \,  i, j \in \{1,\ldots,N\}, \,  k \in \{ 1,\ldots,n \}$;
    
    \item $N^2$ binary variables $z_{ij} \in \{0,1\}, \;  i, j \in \{1,\ldots,N\}$.
\end{itemize}

Here the integer $n \ge 3$ is the number of convex polygon edges; The integer $N$ is the number of grid points in each dimension; An ordered pair of real numbers $(a_k, b_k)$ indicates the coefficients of a line $l^k$ which is the extension of an edge of the convex polygon; The binary variable $l_{ij}^k = 1$ if and only if the grid point $\boldsymbol{g}_{ij}$, indexed by $(i,j)$, lies on the specified side of the line $l^k$; The binary variable $z_{ij} = 1$ if and only if the grid point $\boldsymbol{g}_{ij}$ lies inside the  polygon. We will fully discuss the roles of these decision variables in what follows.

Our goal is to make the convex polygon as small as possible without violating the constraints. This can be realized by minimizing the number of grid points lying inside the polygon. Formally speaking, our optimization objective function can be formulated as 
\begin{equation}
    \min \sum_{i=1}^N \sum_{j=1}^N z_{ij},
    \label{eqn:obj}
\end{equation}
where $\sum\limits _{i=1}^N \sum \limits _{j=1}^N z_{ij}$ is a linear function of binary variables $z_{ij}$.

% Next, we are going to formulate constraints for all the aforementioned decision variables.

\subsection{Constraints on Continuous Variables}
\label{subsecconti}
All the planar lines not passing through the origin can be represented by 
\begin{equation*}
    \{ax + by - 1 = 0: a,b \in \mathbb{R} \land  a^2+b^2 \ne 0 \},
    % \label{representation-1}
\end{equation*}
and different pairs of coefficients $(a,b)$ determine different lines.

A line $l^k := a_kx + b_ky - 1 = 0$ always divides the entire plane into two half-planes. We refer to the half plane $\{(x,y): a_kx + b_ky - 1 < 0 \}$ where the origin $(0,0)$ lies as the \textit{strict inner side} of the line, and $\{(x,y): a_kx + b_ky - 1 \le 0 \}$ as the \textit{inner side} of the line. A strict inner side is a proper subset of the inner side of the same line.

For any $n \ge 3$ planar lines not passing through the origin, the region $S$ bounded by these lines is defined as
\begin{equation}
    S := \{(x,y): \bigwedge_{k=1}^n a_kx + b_ky - 1 \le 0 \},
    \label{eqn:ss}
\end{equation}
where $a_k,b_k \in \mathbb{R}$, $a_k^2+b_k^2 \ne 0$, and $n \ge 3$. On the one hand, since $S$ is the intersection of $n$ inner sides, then $S$ is a convex set and the origin $(0,0) \in S$; On the other hand, the region $S$ may be not a bounded set (In other words, the boundary of $S$ may be not enclosing). Even though bounded, the boundary of $S$ may be a polygon with the number of edges less than $n$, or even not a polygon but an object consisting of line segments or a point only. 

%What we are concerned with is: for which $n \ge 3$ lines, the boundary of the region $S$ generated by these lines is an enclosing $n$ sided convex polygon with the origin inside? 

We look for conditions that ensure $S$ is a bounded $n$ sided convex polygon with the origin in its interior. To do this, we first introduce the following concepts. % Before answering the question, let's first introduce rigorous definitions of several important concepts created by us independently.

\begin{definition}[Digraph \cite{bullo2009distributed}]
\label{def:Digraph}
A finite directed graph (in short, a \textit{digraph}) is a $2$-tuple $(V, E)$, where $V \subseteq \mathbb{R}^2$ is a finite set of planar points called nodes and $E \subseteq V \times V$ is a set of ordered pairs of nodes called (directed) edges. Self-loops are allowed in a digraph. For $u,v \in V$, if the ordered pair $(u,v) \in E$, it denotes an edge from $u$ to $v$.
\end{definition} 

\begin{definition}[Path \cite{bullo2009distributed}]
A finite directed path (in short, a \textit{path}) given by a finite sequence of planar points $S_g := (S_g[0], \ldots, S_g[m])$ where $S_g[i] \in \mathbb{R}^2,  i \in \{0, \ldots, m\}$, is a digraph whose set of nodes $V=\{S_g[i]: i \in \{0, \ldots, m\} \}$ and set of edges $E= \{ (S_g[i], S_g[i+1]): i \in \{0, \ldots, m-1\} \}$. In a path $S_g$, one node $S_g[i]$ \textit{immediately precedes} another node $S_g[j]$ if and only if there is an edge $(S_g[i], S_g[j])$; \textit{Two nodes are consecutive} if and only if one immediately precedes the other; \textit{Three nodes are consecutive} if and only if the first one immediately precedes the second and the second immediately precedes the third.
\end{definition} 

% A path is simple if no node appears more than once in the sequence, except possibly for the initial and final node.

Note that if $S_g[i] = S_g[j], i \ne j$, then $\{ S_g[i], S_g[j]  \} = \{ S_g[i] \}$, according to the Axiom of Extensionality in axiomatic set theory \cite{zach2021sets}. For a node $S_g[i]$ in 
$S_g$, the node which $S_g[i]$ immediately precedes (resp.~which immediately precedes $S_g[i]$) may not exist, and even if it exists, it may be not unique. If there is an edge $(S_g[i],S_g[i])$, then $S_g[i]$ immediately precedes itself. It is possible that  $S_g[i]$ immediately precedes $S_g[j]$ and $S_g[j]$ immediately precedes $S_g[i]$ at the same time if $(S_g[i], S_g[j])$ and $(S_g[j], S_g[i])$ are both edges in $S_g$.

\begin{definition}[Graph and Undirected Version of Digraph \cite{bullo2009distributed}]
\label{def:Graph}
A finite undirected graph (in short, a \textit{graph}) is a $2$-tuple $(V, E)$, where $V \subseteq \mathbb{R}^2$ is a finite set of planar points called nodes and $E$ is a set of unordered pairs of nodes called (undirected) edges. Self-loops are not allowed in a graph. For $u,v \in V, u \ne v$, if the set $\{u,v\} \in E$, it denotes an edge between $u$ and $v$. A graph $(V,E)$ is the \textit{undirected version} of a digraph $(V,E')$ if and only if $(u,v) \in E' \lor (v,u) \in E' \iff \{u,v\} \in E$ for $u,v \in V, u \ne v$.
\end{definition}

In this paper, these definitions emphasize the planar geometric properties. For example, the positions of nodes are 2D coordinates; The shape of an edge connecting two distinct nodes is a straight line segment and the shape of a self-loop is a point. The operation of taking undirected version preserves geometric properties. For example, for any three nodes collinear in a digraph, these nodes must also be collinear in the graph which is the undirected version of the digraph.

\begin{definition}[Enclosing Path]
\label{def:Digraph_enclosing}
A path $S_g$ is \textit{enclosing} if and only if: 1) There are $n \ge 3$ different nodes; 2) No nodes appear more than once in the sequence $S_g$ except for the case of the initial and final node, which should be identical; 3) For any edge connecting two nodes, no other nodes lie on that edge.
\end{definition} 

%\magenta{and no edge lines can strictly intersect}. \blue{in fact, I discuss this case later on when addressing self-intersecting polygons in fig6b. We don't need to add it here. The concept of enclosing graph here is still stronger than traditional cycle-graph. A key issue is since the graph/digraph discussed in this paper is endowed with geometrical property, we cannot tell which node in the graph defined here is first, second,...(i.e., we cannot tell the order of nodes in a graph in this paper, and indeed we cannot decide which line segments are sides of the graph), this is why we first introduce the definition of enclosing digraph and then we define enclosing graph through the bridge of the definiton of enclosing digraph. See fig1，(c) sequence (1,2,3,4,1) is a traditional cycle digraph because only first and last node are the same and all the edges are different, but it is not enclosing in this paper.}

% finite ensures we can say words like first last

For example, all the paths in \cref{fig:dignonenclose} are not enclosing. The path $(1,2,3,4,2,1)$ in \cref{sub-fig-C2} 
violates Condition 2) in \cref{def:Digraph_enclosing} because Node 2 appears more than once; The path $(1,2,3,4)$ in \cref{sub-fig-C3} also violates Condition 2) because the first node 1 and last node 4 are not identical; The path $(1,2,3,4,1)$ in \cref{sub-fig-C4} satisfies first two conditions, but violates Condition 3) because Node 4 is on the edge $(1,2)$.

\begin{figure}[H]
    \centering
    \subfloat[A path violating Condition 2)]{
    \label{sub-fig-C2}
    %\begin{minipage}{ }
    %\centering
    \includegraphics[width=0.11\textwidth]{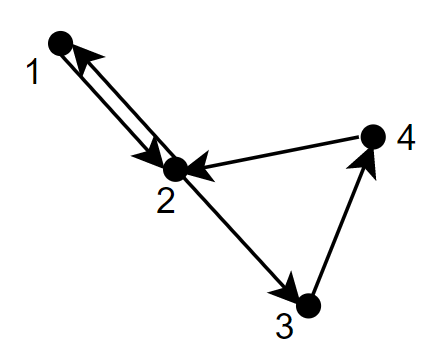}
    %\end{minipage}
    } \hspace{0.5cm}
    \subfloat[A path violating Condition 2)]{
    \label{sub-fig-C3}
    %\begin{minipage}{ }
    %\centering
    \includegraphics[width=0.1\textwidth]{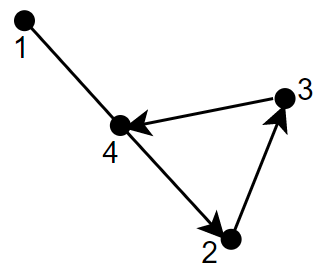}
    %\end{minipage}
    } \hspace{0.5cm}%\quad
    \subfloat[A path violating Condition 3)]{
    \label{sub-fig-C4}
    %\begin{minipage}{ }
    %\centering
    \includegraphics[width=0.1\textwidth]{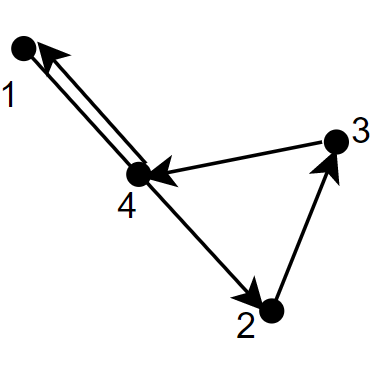}
    %\end{minipage}
    }
    \caption{Three cases of paths that are not enclosing}
    \label{fig:dignonenclose}  %
\end{figure} % end 

Through the definition of an enclosing path, we can introduce a definition of an enclosing graph.

%Roughly speaking, a graph is enclosing only if all of its vertices and edges can be traversed without repeat.

\begin{definition}[Enclosing Graph]
\label{def:Graph_enclosing}
    A graph $G$ is \textit{enclosing} if and only if there is an enclosing path $S_g$ whose undirected version is $G$.
\end{definition}

For example, the graph $G$ in \cref{fig:ugenclosing} is not enclosing because any path whose undirected version is $G$ is not an enclosing path. 

\begin{figure}[H] %[!h]  [hbt!]
\centering
\includegraphics[width=.1\textwidth]{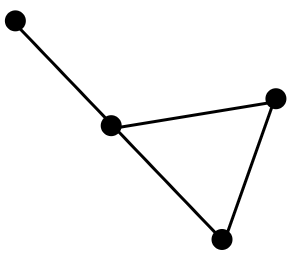}
\caption{A graph that is not enclosing} 
\label{fig:ugenclosing}
\end{figure}

It is straightforward to see that if a path $S_g$ is enclosing, then its undirected version $G$ is also enclosing. Conversely, if a path is not enclosing, its undirected version may still be enclosing.

% \begin{lemma}
% \footnote{\magenta{Given the definition above, this lemma is obvious. Please remove it and remove the proof.}}
% \label{lemma:1}
% \end{lemma}

% \textbf{Proof}: 
% See \cref{subsec:lemma1}.  \hfill \textbf{Q.E.D}

% \begin{proof}
% See \cref{subsec:lemma1}. 
% \end{proof}

Inspired by \cref{def:Digraph_enclosing} but subtly different from it, we define the following notion. 

\begin{definition}[Formally Enclosing Sequence of Symbols]
\label{def:Formallyenclosing}
    Let $S_s$ be a finite sequence of symbols. The sequence $S_s$ is \textit{formally enclosing} if and only if: 1) There are $n \ge 3$ formally different symbols; 2) No symbols appear more than once except for the case in which the initial and final symbol are formally identical. 
    
    % one symbol $S_s[i]$ immediately precedes the other symbol $S_s[j]$ if and only if 
    % $\exist x \in \{ m : S_g[m] = S_g[i]\}$
    % $j = i + 1$; Two symbols $S_s[i]$ and $S_s[j]$ are consecutive if and only if $j = i + 1 \lor i = j+1$. % 3) The initial symbol is formally identical to the terminal symbol.
\end{definition}

In a finite sequence of symbols $S_s$, one symbol \textit{immediately precedes} another symbol if and only if there is an appearance $S_s[i]$ of the first one and an appearance $S_s[j]$ of the second such that $j = i + 1$; \textit{Two symbols are consecutive} if and only if one immediately precedes the other; \textit{Three symbols are consecutive} if and only if the first one immediately precedes the second and the second immediately precedes the third. 

A finite sequence of symbols is just an ordered structure, which neither has to be enclosing nor has geometric meaning. To address this, we introduce an isomorphism map, as defined below. %It is inspired by the concept of isomorphism in Abstract Algebra or Mathematical Logic \cite{hodel2013introduction}, but adapted to our need here. An isomorphism can connect (symbols) and interpretation (planar points).

\begin{definition}[Isomorphism Type A]
\label{def:isomorphism}
    Let $f_{\text{tag}}$ be a map from a finite sequence of symbols $S_{s}$ to a path $S_{g}$. The map $f_{\text{tag}}$ is an \textit{isomorphism type A} if and only if the following properties hold for $f_{\text{tag}}$: 1) \textbf{Injective}: formally different symbols in $S_{s}$ refer to different nodes in $S_{g}$; 2) \textbf{Surjective}: every node in $S_{g}$ is the image of a symbol in $S_{s}$; 3) \textbf{Order-Preserving}: for any two symbols $S_{s}[i], S_{s}[j]$ in $S_{s}$, the symbol $S_{s}[i]$ immediately precedes $S_{s}[j]$ if and only if $(f_{\text{tag}}(S_{s}[i]), f_{\text{tag}}(S_{s}[j]))$ is an edge in $S_g$; %(one element in the sequence precedes another one if and only if the index of the former element is less than that of the latter one); 
     4) \textbf{Not on the edge}: for any two consecutive symbols $S_{s}[i], S_{s}[i+1]$ and any third symbol $S_{s}[k]$ in $S_{s}$, if $S_{s}[k]$ is formally different from $S_{s}[i]$ and $S_{s}[i+1]$, then $f_{\text{tag}}(S_{s}[k])$ is not on the edge $(f_{\text{tag}}(S_{s}[i]), f_{\text{tag}}(S_{s}[i+1]))$ in $S_g$. 
\end{definition}

 % \footnote{\magenta{Well, we said the map is bijective, so this in principle can't happen (images of consecutive symbols can't be the same). So I would rephrase as "unless the images of the two consecutive symbols coincide with the initial and terminal nodes of $S_g$".... And another question: here, we are not assuming that $S_g$ is enclosing, right? However, to be able to find an isomorphism, this definition clearly implies this needs to be the case (at least with the original definition you had without the intersection of edges). On the other hand, if we want to add the edge intersection condition on the definition of enclosing/enclosing, then we may have to add something in here, such as that for any two pairs of consecutive nodes it holds that their image edges do not strictly intersect.}}) 
 
%In fact, if a map $f_{tag}$ is injective, surjective and order-preserving, then the images of any two consecutive symbols must also be two consecutive points (namely, a directed edge of the digraph). Therefore, an isomorphism $f_{tag}$ also preserves consecutiveness between the sequence of symbols and the digraph.

%The connection between a sequence of symbols being formally enclosing and a graph being enclosing, through the introduction of an isomorphism, is shown as follows. 

\begin{remark}
\label{remark_enclose}
Observe that if there is an isomorphism type A $f_{\text{tag}}$ from a finite sequence of symbols $S_s$ to a path $S_g$, then $S_s$ is formally enclosing if and only if $S_g$ is enclosing. We have seen that if $S_g$ is enclosing, then its undirected version $G$ is also enclosing. Therefore, we conclude that if $f_{\text{tag}}$ is an isomorphism type A from $S_s$ to $S_g$, and $S_s$ is formally enclosing, then the graph $G$, as the undirected version of $S_g$, is also enclosing.
\end{remark}

% \begin{lemma}
%    
%     \label{lemma:2}
% \end{lemma}

% \begin{proof}
%     See \cref{subsec:lemma2}. 
% \end{proof}

In addition to the discussions of enclosing paths and graphs above, we would also like to explore under which conditions it suffices to guarantee non-degenerate paths and graphs. %. Before that, let's first give a definition of the non-degeneration of a graph.

\begin{definition}[Non-Degenerate Path and Graph] Let $S_g$ (resp.~$G$) be a path (resp.~a graph) with nodes given by a set of $n \ge 3$ points in $\mathbb{R}^2$.  The 
 path $S_g$ is \textit{a non-degenerate path} if and only if any three consecutive nodes in the sequence $S_g$ are not collinear. The graph $G$ is \textit{a non-degenerate graph} if and only if there is a non-degenerate path $S_g$ whose undirected version is $G$. A path $S_g$ (resp.~a graph $G$) is \textit{degenerate} if and only if it is not non-degenerate.
\end{definition}

Requiring that the number of nodes is $n=4$, the following \cref{fig:gdeg} gives two degenerate paths. In \cref{sub-fig-degC1}, the path $(1,2,2,3,1)$ is degenerate, because the number of nodes in the path is three less than $n=4$ and there are three consecutive nodes $(1,2,2)$ that are collinear. The path $(1,2,3,4,1)$ in \cref{sub-fig-degC2} is also degenerate, because there are three consecutive nodes $(2,3,4)$ that are collinear.

\begin{figure}[H]
    \centering
    \subfloat[A  degenerate path $(1,2,2,3,1)$ with three nodes]{
    \label{sub-fig-degC1}
    %\begin{minipage}{ }
    %\centering
    \includegraphics[width=0.12\textwidth]{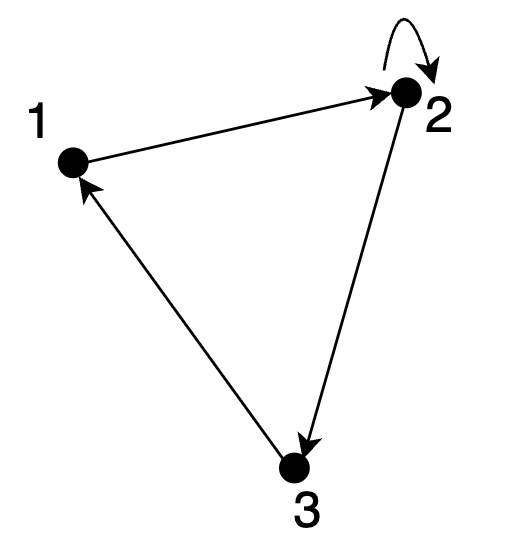}
    %\end{minipage}
    } \hspace{1cm}
    \subfloat[A degenerate path $(1,2,3,4,1)$ with four nodes]{
    \label{sub-fig-degC2}
    %\begin{minipage}{ }
    %\centering
    \includegraphics[width=0.12\textwidth]{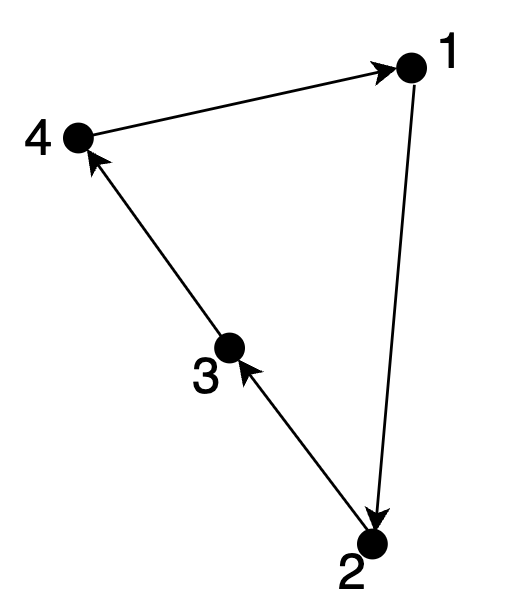}
    %\end{minipage}
    } 
    \caption{Two cases of degenerate paths}
    \label{fig:gdeg}  %
\end{figure} % end 

Requiring that the number of nodes is $n=4$, the following \cref{fig:gdeg_un} gives two degenerate graphs, which are the undirected version of the two paths shown in \cref{fig:gdeg}, respectively.

\begin{figure}[H]
    \centering
    \subfloat[A degenerate graph with three nodes]{
    \label{sub-fig-degC1-un}
    %\begin{minipage}{ }
    %\centering
    \includegraphics[width=0.12\textwidth]{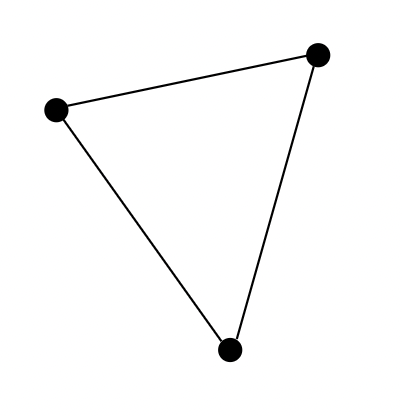}
    %\end{minipage}
    } \hspace{1cm}
    \subfloat[ A degenerate graph with four nodes]{
    \label{sub-fig-degC2-un}
    %\begin{minipage}{ }
    %\centering
    \includegraphics[width=0.12\textwidth]{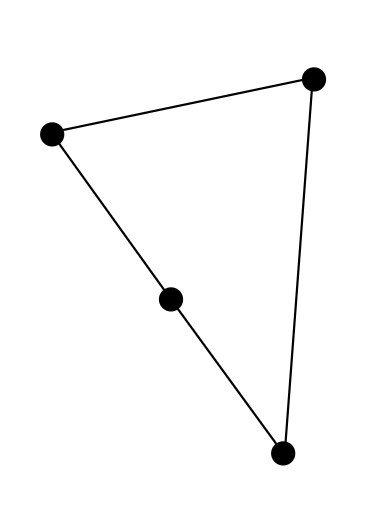}
    %\end{minipage}
    } 
    \caption{Two cases of degenerate graphs }
    \label{fig:gdeg_un}  %
\end{figure} % end 

Note that non-degenerate graphs (resp.~paths) and enclosing graphs (resp.~paths) are two independent concepts. The following \cref{fig:example_enc_nond} gives examples to illustrate the difference between those two concepts.

\begin{figure}[H] %[!h]  [hbt!]
\centering
\includegraphics[width=.4\textwidth]{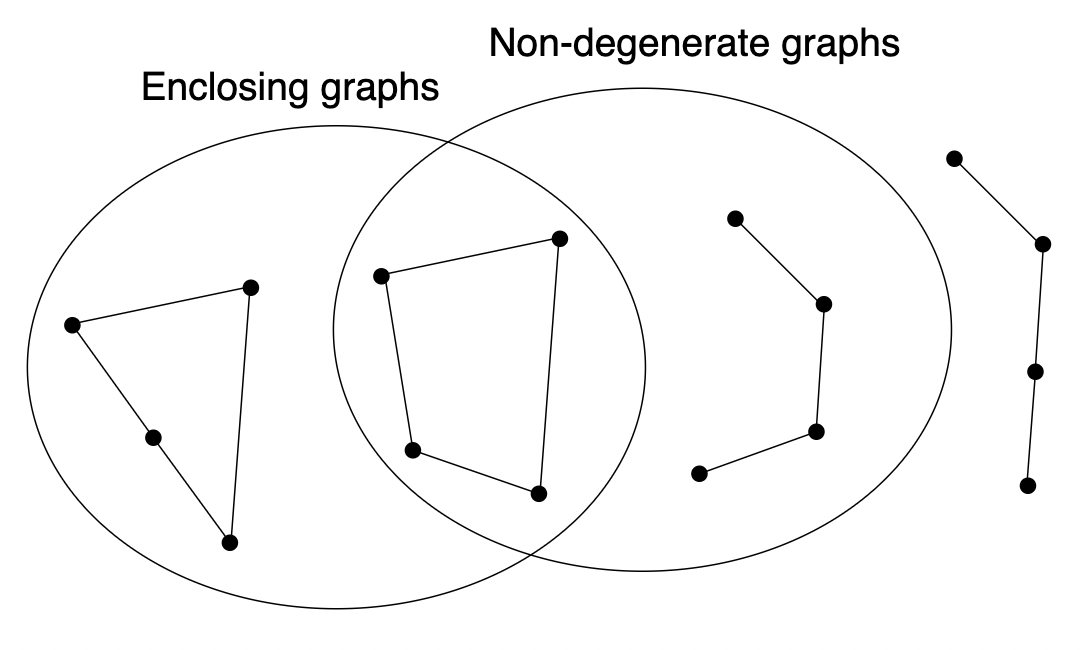}
\caption{Examples of enclosing or non-degenerate graphs} 
\label{fig:example_enc_nond}
\end{figure}

\begin{definition}[Isomorphism Type B]
\label{def:isomorphism_type_b}
    Let $f_{\text{tag}}$ be a map from a finite sequence of symbols $S_{s}$ to a path $S_{g}$. The map $f_{\text{tag}}$ is an \textit{isomorphism type B} if and only if the following properties hold for $f_{\text{tag}}$: 1) \textbf{Injective}: formally different symbols in $S_{s}$ refer to different nodes in $S_{g}$; 2) \textbf{Surjective}: every node in $S_{g}$ is the image of a symbol in $S_{s}$; 3) \textbf{Order-Preserving}: for any two symbols $S_{s}[i], S_{s}[j]$ in $S_{s}$, the symbol $S_{s}[i]$ immediately precedes $S_{s}[j]$ if and only if $(f_{\text{tag}}(S_{s}[i]), f_{\text{tag}}(S_{s}[j]))$ is an edge in $S_g$; %(one element in the sequence precedes another one if and only if the index of the former element is less than that of the latter one); 
    4) \textbf{Noncollinearity}: for any three consecutive symbols in $S_{s}$, their images are not collinear in $S_g$. 
\end{definition}

\begin{remark}
\label{remark_degenerate}
Observe that if there is an isomorphism type B $f_{\text{tag}}$ from a finite sequence of symbols $S_s$ to a path $S_g$, and there are $n \ge 3$ formally different symbols in $S_s$, then $S_g$ and its undirected version $G$ are both non-degenerate.
\end{remark}

Note that Condition 4) in \cref{{def:isomorphism_type_b}} of isomorphism type B is different from Condition 4) in \cref{def:isomorphism} of isomorphism type A. For example, the label shown in \cref{sub-fig-type-a-y} gives a map $f_{\text{tag}}$ from the sequence of symbols $(``a",``b",``c",``d",``b")$ to the path $(a,b,c,d,b)$. The map $f_{\text{tag}}$ is an isomorphism type A but not an isomorphism type B, because there are three consecutive symbols $``a",``b",``c"$ whose images $a,b,c$ are collinear in the path. This violates Condition 4) in \cref{def:isomorphism_type_b}. In contrast, the label shown in \cref{sub-fig-type-b-y} gives another map $f_{\text{tag}}'$ from the sequence of symbols $(``a",``b",``c",``d")$ to the path $(a,b,c,d)$. The map $f_{\text{tag}}'$ is an isomorphism type B but not an isomorphism type A, since although the symbol $``d"$ is formally different from each of the two consecutive symbols $``a",``b"$, Node $d$ is on the edge $(a,b)$ in the path. This violates Condition 4) in \cref{def:isomorphism}.

\begin{figure}[H]
    \centering
    \subfloat[An isomorphism type A ]{
    \label{sub-fig-type-a-y}
    %\begin{minipage}{ }
    %\centering
    \includegraphics[width=0.12\textwidth]{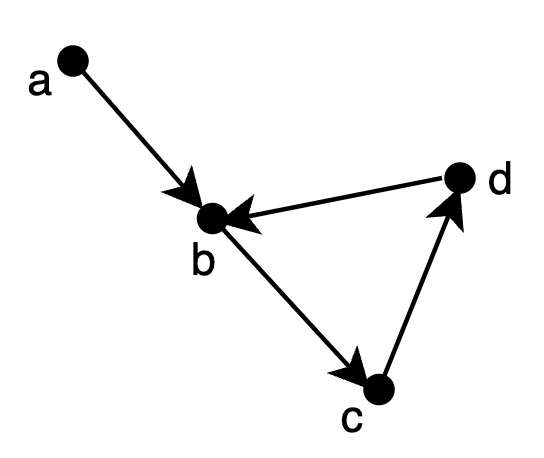}
    %\end{minipage}
    } \hspace{2cm}
    \subfloat[An isomorphism type B ]{
    \label{sub-fig-type-b-y}
    %\begin{minipage}{ }
    %\centering
    \includegraphics[width=0.12\textwidth]{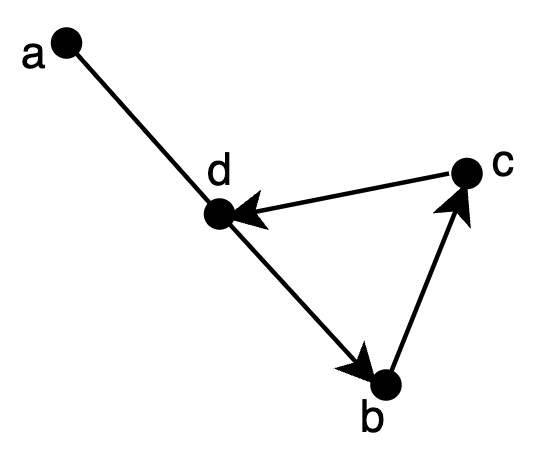}
    %\end{minipage}
    } 
    \caption{Examples of isomorphisms type A or type B}
    \label{fig:diff_two_iso}  %
\end{figure} % end 

% \begin{lemma}
%     Let $f_{tag}$ be a map from a finite sequence $S_{s}$ of symbols to a digraph $S_{g}$ which is a finite sequence of planar points. If 1) The number of formally different symbols in $S_{s}$ is $n \ge 3$; then the graph $G$, which is the undirected version of $S_{g}$, doesn't degenerate.
%     \label{lemma:3}
% \end{lemma}

% \begin{proof}
%     See \cref{subsec:lemma3}.
% \end{proof}

We go back to look for the conditions that ensure $S$ in \cref{eqn:ss} is a bounded $n$ sided convex polygon with the origin in its interior.

% Now let's go back to answering the question asked before \cref{def:Digraph_enclosing}. 

For any two lines $l^i := a_ix + b_iy - 1 = 0$ and $l^j := a_jx + b_jy - 1 = 0$, if $D_{ij} := a_ib_j-b_ia_j \ne 0$, then there is a unique intersection point between those two lines, which is
\begin{equation*}
    V_{ij} := l^i \cap l^j = \left(-\frac{b_i-b_j}{D_{ij}}, \frac{a_i-a_j}{D_{ij}}\right).
\end{equation*}
Note that $V_{ij} = V_{ji}$. The fact $D_{ij} \ne 0$ implies that $a_i^2 + b_i^2 \ne 0 \land a_j^2 + b_j^2 \ne 0$, and also implies that $a_i \ne a_j \lor b_i \ne b_j$, namely, $V_{ij} \ne (0,0)$. This makes sense since neither lines passes through the origin and therefore the origin is definitely not the intersection point. The sign of $D_{ij}$ depends on the order of $i$ and $j$, namely, $D_{ij} = - D_{ji}$. For two lines $l^i, l^j$ sharing a unique intersection point, if we can rotate $l^i$'s normal vector $(a_i, b_i)$ to $l^j$'s normal vector $(a_j, b_j)$ counterclockwise about the origin by an angle $ \theta \in (0, \pi) $, then $D_{ij}>0$; otherwise, $D_{ij}<0$. When $D_{ij} < 0$, we can swap the value of $(a_i, b_i)$ with that of $(a_j, b_j)$ so that $D_{ij} > 0$. Hence, without loss of generality, we can assume that $D_{ij} > 0$.

% $
% \left\\{
% \begin{aligned}
% x = & \cos(t) \\\
% y = & \sin(t) \\\
% z = & \frac xy \\\
% \end{aligned}
% \right.
% $

% \begin{align*}
% y  = 1 + & \left(  \frac{1}{x} + \frac{1}{x^2} + \frac{1}{x^3} + \ldots \right. \\
%   &\left. \quad + \frac{1}{x^{n-1}} + \frac{1}{x^n} \right)
% \end{align*}

For simplicity of notation, we define a unary operation $ i^{\oplus}$ for $  i \in \{ 1,\ldots,n \}$
\begin{equation*}
    i^{\oplus} : =
    \begin{cases}
    i + 1, & i \in \{ 1,\ldots,(n-1) \},\\
    1, & i = n, \\
    \end{cases}
\end{equation*}
where $i+1$ is the ordinary arithmetic operation of addition. Inversely, we can define another unary operation $i^{\ominus}$ for $  i \in \{ 1,\ldots,n \}$, that is, $j= i^{\ominus}$ if and only if $j^{\oplus} = i$. 

Consider $n \ge 3$ planar lines not passing through the origin, denoted by $l^k, \;  k \in \{ 1,\ldots,n \}$. We hope the boundary of the region $S$, generated by these lines in the way \cref{eqn:ss} gives, is an enclosing $n$ sided convex polygon with the origin in its interior. A necessary condition to make it is that there are at least $n$ different line intersection points serving as $n$ nodes of the polygon. Without loss of generality, we require the existence of such $n$ formally different intersection points $V_{ij}, \;  i \in \{ 1,\ldots,n \}, j = i^{\oplus}$ (which also implicitly requires $l^i \ne l^j, \;  i \in \{ 1,\ldots,n \}, j = i^{\oplus}$). To enforce this, we formulate the following constraints
\begin{equation}
    D_{ij}:=a_ib_j-b_ia_j > 0, \quad  i \in \{ 1,\ldots,n \}, \;j = i^{\oplus}. 
    \label{eqn:detcon}
\end{equation}
Nevertheless, satisfying \cref{eqn:detcon} can not ensure that there are indeed $n$ different intersection points (resp.~lines), because two formally different representations of the intersection points (resp.~lines) may refer to an identical intersection point (resp.~line), as illustrated in \cref{fig:occurrence_identical}.

\begin{figure}[H]
    \centering
    \subfloat[The point $V_{12}$ and $V_{23}$ are identical]{
    \label{sub-fig-24n1}
    %\begin{minipage}{ }
    %\centering
    \includegraphics[width=0.18\textwidth]{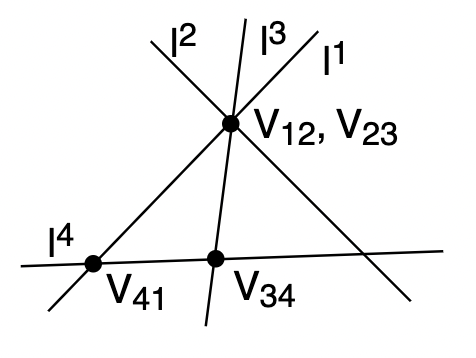}
    %\end{minipage}
    } \hspace{0.5cm}
    \subfloat[The line $l^3$ and $l^6$ are identical]{
    \label{sub-fig-24n2}
    %\begin{minipage}{ }
    %\centering
    \includegraphics[width=0.18\textwidth]{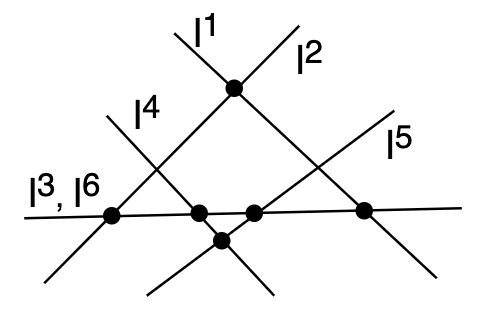}
    %\end{minipage}
    } \hspace{0.5cm}%\quad
    \subfloat[The line $l^1$ and $l^3$, the point $V_{12}$ and $V_{23}$, as well as the point $V_{34}$ and $V_{41}$ are identical]{
    \label{sub-fig-24n3}
    %\begin{minipage}{ }
    %\centering
    \includegraphics[width=0.18\textwidth]{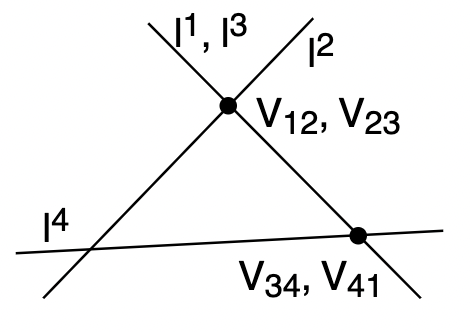}
    %\end{minipage}
    }
    \caption{Scenarios of identical intersection points or lines while satisfying \cref{eqn:detcon}}
    \label{fig:occurrence_identical}  %
\end{figure} % end 

The boundary of a region $S$ generated by a finite set of lines can be defined as the undirected version $G$ of a path $S_g$. Throughout the rest of this paper, when a graph $G$ is the undirected version of a path $S_g$, we say the graph $G$ is \textit{formed} by $S_g$. For any $n \ge 3$ planar lines $l^k, \;  k \in \{ 1,\ldots,n \}$ not passing through the origin, if there are $n$ formally different intersection points $V_{ij}, \;  i \in \{ 1,\ldots,n \}, j = i^{\oplus}$, we can construct a path given by a finite sequence of intersection points $S_g:=(V_{12}, V_{23}, \ldots, V_{(n-1)n}, V_{n1}, V_{12})$. If the graph $G$ formed by $S_g$ is an enclosing $n$ sided convex polygon with the origin in its interior, then the boundary of the region $S$ is identical to the graph $G$, and therefore the boundary of the region $S$ is also an enclosing $n$ sided convex polygon with the origin in its interior. Hence, to make the boundary of the region $S$ an enclosing $n$ sided convex polygon with the origin in its interior, we just need to ensure the graph $G$ formed by $S_g$ is an enclosing $n$ sided convex polygon with the origin in its interior. 

However, for certain $n \ge 3$ planar lines not passing through the origin that have $n$ formally different intersection points $V_{ ij}, \;  i \in \{ 1,\ldots,n \}, j = i^{\oplus}$, the graph $G$ formed by $S_g$ may be not an enclosing $n$ sided convex polygon with the origin in its interior. Specifically, 1) the graph $G$ may be not enclosing. For example, the graph $G$ in \cref{sub-fig-24n3} is a line segment and is not enclosing; 2) The graph $G$ may degenerate to a polygon whose number of edges less than $n$, or even not a polygon but line segments or a point. For example, the graph $G$ in \cref{sub-fig-24n1} is a triangle, not a quadrilateral as required; 3) The graph $G$ may be not convex (may be concave or self-intersecting), as illustrated in \cref{sub-fig-concave} and \cref{sub-fig-complex}; 4) The origin $(0,0)$ may be not in the interior of the graph $G$, as shown in \cref{sub-fig-originnotin}.

\begin{figure}[H] %[!h]  [hbt!]
    \centering
    \subfloat[A concave quadrilateral formed by $(V_{12}, V_{23}, V_{34}, V_{41}, V_{12})$ ]{
    \label{sub-fig-concave}
    %\begin{minipage}{ }
    %\centering
    \includegraphics[width=0.18\textwidth]{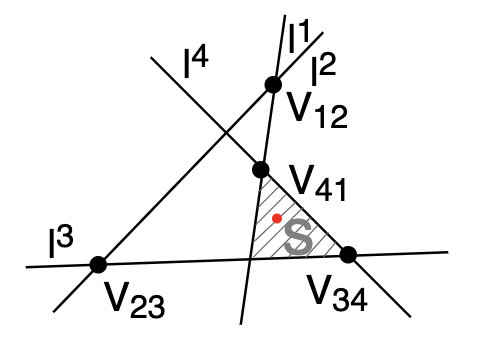}
    %\end{minipage}
    } \hspace{1cm}
    \subfloat[A self-intersecting quadrilateral formed by $(V_{12}, V_{23}, V_{34}, V_{41}, V_{12})$]{
    \label{sub-fig-complex}
    %\begin{minipage}{ }
    %\centering
    \includegraphics[width=0.18\textwidth]{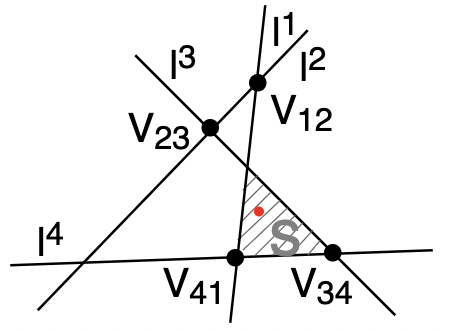}
    %\end{minipage}
    } \hspace{1cm}%\quad
    \subfloat[A convex quadrilateral formed by $(V_{12}, V_{23}, V_{34}, V_{41}, V_{12})$ with the origin in its exterior]{
    \label{sub-fig-originnotin}
    %\begin{minipage}{ }
    %\centering
    \includegraphics[width=0.18\textwidth]{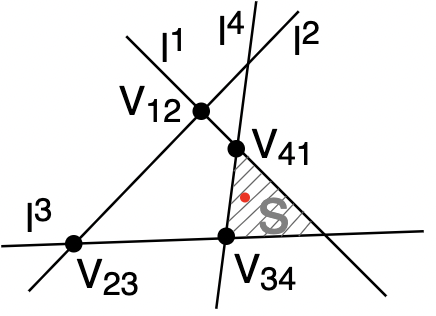}
    %\end{minipage}
    }
    \caption{Cases of graphs that are not enclosing $n$ sided convex polygons with the origin in its interior (The shadowed area is the region $S$ defined in \cref{eqn:ss} and the red point is the origin)}
    \label{fig:scenotbeing}  %
\end{figure} % end 

As discussed above, the constraints \cref{eqn:detcon} are not sufficient to guarantee that the graph $G$ formed by $S_g$ is an enclosing $n$ sided convex polygon with the origin in its interior. To this end, we provide additional constraints in \cref{theorem:1}.

Before proceeding to \cref{theorem:1}, we first introduce \cref{lemma:4}. This lemma is useful in the sense that once constraints \cref{eqn:detcon} and \cref{eqn:no1cons} are satisfied, we no longer need to explicitly enforce that all $n$ intersection points (resp.~ lines) are indeed different from each other.

% \newpage

\begin{lemma}
\label{lemma:4}
    Assume that there are $n \ge 3$ planar lines not passing through the origin $(0,0)$, denoted by $l^k := a_kx + b_ky - 1 = 0,  k \in \{ 1,\ldots,n \}$. If there are $n$ intersection points $V_{ij} := l^i \cap l^j, \;   i \in \{ 1,\ldots,n \}, j = i^{\oplus}$, and the coefficients of these lines satisfy the following inequalities
\begin{equation}
    \begin{aligned}
    & -a_k\left(b_i-b_j \right)+b_k\left(a_i-a_j\right) - (a_ib_j-b_ia_j) < 0, \\
    &  i \in \{ 1,\ldots,n \}, j = i^{\oplus},  k \in \{ 1,\ldots,n \}-\{i, j\}, 
    \label{eqn:no1cons}
    \end{aligned}
\end{equation}
then, any two formally different representations of intersection points (resp.~lines) refer to different intersection points (resp.~lines), and therefore all $n$ intersection points (resp.~lines) are different from each other.
\end{lemma}

\begin{proof}
    See \cref{subsec:lemma4}.  
\end{proof}

% Induction: 
% We know that each line divides the plane into two half-planes. $\mathrm{Fig.\space 1.}a$ lets us hypothesize that, if each of the three points of intersection (i.e. $V_{jk}$) and the origin lie in one of the half-planes generated by the line opposite to it (i.e.  $L_i$), the three lines form a close region (i.e. triangle) which has the origin in its interior. 
% $\mathrm{Fig.\space 1.}b$ confirms that this is in fact is the case, because one of the points of intersection, $V_{jk}$ to be precise, and the origin lie in the two different half-planes created by the line opposite to it (i.e. $L_i$).
% This supposition can be generalized to $n$ number of lines, none of which passes through the origin, with two slight adjustments. Firstly, when there are $n$ lines at our disposal, we need to deal with more than $n$ ($\dfrac{n\left(n-1\right)}{2}$ to be precise) number of points of intersection. Secondly, each of these points must be checked against more than one ($n-2$ to be precise) opposite line. Last but not least, not all points of intersection need to satisfy all $n-2$ mandatory constrains they are subjected to. It is suffice to find $n$ all-obliging points, because we are here looking for an $n-$gon with origin in its interior.

Next, the following \cref{theorem:1} gives the conditions which suffice to ensure the boundary of the region $S$ in \cref{eqn:ss} is an enclosing $n$ sided convex polygon with the origin in its interior. Intuitively, constraints \cref{eqn:no1cons} check each intersection point $V_{ij},  i \in \{ 1,\ldots,n \}, j = i^{\oplus} $ against all its opposite $(n-2)$ lines $l^k,  k \in \{ 1,\ldots,n \}-\{i, j\}$.

% \blue{yinhao "" symbol problem }

% \begin{equation}
%     i^{\oplus} : =
%     \begin{cases}
%     i + 1 & i \in \{ 1,\ldots,(n-1) \} \\
%     1 & i = n \\
%     \end{cases}
% \end{equation}

\begin{theorem}
\label{theorem:1}
Assume that there are $n \ge 3$ planar lines not passing through the origin $(0,0)$, denoted by $l^k := a_kx + b_ky - 1 = 0, \;  k \in \{ 1,\ldots,n \}$. If there are $n$ intersection points $V_{ij} := l^i \cap l^j, \;   i \in \{ 1,\ldots,n \}, j = i^{\oplus}$, and the coefficients of these lines satisfy \cref{eqn:no1cons}, then the boundary of the region $S := \{(x,y): \bigwedge\limits_{k=1}^n a_kx + b_ky - 1 \le 0 \}$ generated by these lines is the graph $G$ formed by the sequence of intersection points $S_g:=(V_{12}, V_{23}, \ldots, V_{(n-1)n}, V_{n1}, V_{12})$, which defines an enclosing $n$ sided convex polygon with the origin in its interior.
\end{theorem}

\begin{proof}
    See \cref{subsec:theorem1}.
\end{proof}

So far, we have just derived constraints \cref{eqn:detcon} and \cref{eqn:no1cons} to ensure that the polygon of interest contains the origin $(0,0)$. The following result extends the applicability of these conditions to contain an arbitrary point. 

%Consider an arbitrary point $(\Bar{x},\Bar{y})$ and $n \ge 3$ planar lines not passing through the point $(\Bar{x},\Bar{y})$  represented by $l^k := a_k(x-\Bar{x}) + b_k(y-\Bar{y}) - 1 = 0, \;  k \in \{ 1,\ldots,n \} $. Then, it can be shown that \cref{eqn:detcon} and \cref{eqn:no1cons} guarantee that the boundary of the region $S := \{(x,y): \bigwedge\limits_{k=1}^n a_k(x-\Bar{x}) + b_k(y-\Bar{y}) - 1 \le 0 \}$ generated by these lines is an enclosing $n$ sided convex polygon containing $(\Bar{x},\Bar{y})$ inside. We will show this in \cref{corollary}.

\begin{corollary}
\label{corollary}
Assume that there are $n \ge 3$ planar lines not passing through the point $(\Bar{x},\Bar{y})$, denoted by $l^k := a_k(x-\Bar{x}) + b_k(y-\Bar{y}) - 1 = 0, \;  k \in \{ 1,\ldots,n \} $. If constraints \cref{eqn:detcon} and \cref{eqn:no1cons} hold, then the boundary of the region $S := \{(x,y): \bigwedge\limits_{k=1}^n a_k(x-\Bar{x}) + b_k(y-\Bar{y}) - 1 \le 0 \}$ generated by these lines is the graph $G$ formed by the sequence of intersection points $S_g:=(V_{12}, V_{23}, \ldots, V_{(n-1)n}, V_{n1}, V_{12})$, which defines an enclosing $n$ sided convex polygon with the point $(\Bar{x},\Bar{y})$ inside.
\end{corollary}

\begin{proof}
    See \cref{subsec:corollary1}. 
\end{proof}

We finally conclude that, according to \cref{corollary}, \cref{eqn:detcon} and \cref{eqn:no1cons} are the constraints that we formulate for $2n$ continuous variables $ a_k, b_k \in \mathbb{R}, \;  k \in \{ 1,\ldots,n \} $. For any $n \ge 3$ planar lines not passing through an arbitrary point, if satisfying \cref{eqn:detcon} and \cref{eqn:no1cons}, we can guarantee that the boundary of the region $S := \{(x,y): \bigwedge\limits_{k=1}^n a_k(x-\Bar{x}) + b_k(y-\Bar{y}) - 1 \le 0 \}$ generated by these lines is an enclosing $n$ sided convex polygon with that point in its interior.

\subsection{Constraints Relating Continuous and Discrete Variables}
\label{subsecbridge}

Recall that, a mesh of grid points is a map $\boldsymbol{g}$ given by $(x_i, y_j) = \boldsymbol{g}_{ij}, \; i,j \in \{1,\ldots,N\}$ as defined in \cref{mesh_grids}, and we have obtained a normalized weight matrix $\boldsymbol{w}:= [w_{ij}]$ such that every grid point $\boldsymbol{g}_{ij}, \; i,j \in \{1,\ldots,N\}$ has a normalized weight $w_{ij}$, respectively, in \cref{eqn:normw}.
The index of the grid point whose normalized weight is the greatest is denoted by $(\Bar{i},\Bar{j}) := \arg \max \limits_{i,j} (\boldsymbol{w})$, and its coordinate is $(x_{\Bar{i}},y_{\Bar{j}}) := \boldsymbol{g}_{\Bar{i}\,\Bar{j}} $, which plays the role of $(\Bar{x},\Bar{y})$ in \cref{subsecconti}.

The following binary variables $z_{ij} \in \{0,1\}$ are defined so that $z_{ij}=1$ if and only if the grid point $\boldsymbol{g}_{ij}$ lies inside the polygon obtained in \cref{subsecconti}. In other words, 
\begin{align}
    & z_{ij}=1\implies \bigwedge_{k=1}^n a_k(x_i-x_{\Bar{i}})+b_k(y_j-y_{\Bar{j}}) - 1\le 0, \label{eq:zab1} \\
    &  i,j \in \{1,\ldots,N\}, \notag \\[0.1cm]
    & z_{ij}=0\implies \bigvee_{k=1}^n a_k(x_i-x_{\Bar{i}})+b_k(y_j-y_{\Bar{j}}) - 1 > 0, \label{eq:zab2}\\
    &  i,j \in \{1,\ldots,N\}. \notag
\end{align}

To convert the logical constraints above into equivalent algebraic constraints, we introduce new binary variables $l_{ij}^k \in \{0,1\}$ such that $l_{ij}^k = 1$ if and only if the grid point $\boldsymbol{g}_{ij}$ lies on the inner side of the line $l^k$. That is, 
\begin{align*}
    & l_{ij}^k = 1 \implies a_k(x_i-x_{\Bar{i}})+b_k(y_j-y_{\Bar{j}}) - 1 \le 0,  \\
    &  i,j \in \{1,\ldots,N\}, \quad  k \in \{1,\ldots,n\}, \notag \\[0.2cm]
    & l_{ij}^k = 0 \implies a_k(x_i-x_{\Bar{i}})+b_k(y_j-y_{\Bar{j}}) - 1 > 0,  \\
    &  i,j \in \{1,\ldots,N\}, \quad  k \in \{1,\ldots,n\}, \notag
\end{align*}
where $n \in \{x \in \mathbb{N}: x \ge 3 \} $. These can be formulated as algebraic constraints via big-$M$ representation \cite{griva2009linear} as 
\begin{align}
    & a_k(x_i-x_{\Bar{i}})+b_k(y_j-y_{\Bar{j}}) - 1 \le M_{ijk}^1(1-l_{ij}^k) \label{eqn:lab1}\\
    &  i,j \in \{1,\ldots,N\}, \quad  k \in \{1,\ldots,n\}, \notag \\[0.2cm]
    & -a_k(x_i-x_{\Bar{i}}) - b_k(y_j-y_{\Bar{j}}) + 1  < M_{ijk}^2 l_{ij}^k \label{eqn:lab2}\\
    &  i,j \in \{1,\ldots,N\}, \quad  k \in \{1,\ldots,n\},  \notag 
\end{align}
where the parameters $M_{ijk}^1, M_{ijk}^2, \; i,j \in \{1,\ldots,N\}, \; k \in \{1,\ldots,n\}$ are  sufficiently large constants. They are often set specifically across different constraints, but in this paper, we just need to choose a common constant for them.

% the parameter $M$ a sufficiently large constant is often set specifically across different constraints\footnote{\magenta{how do we choose the parameter $M$ here? }}. 

The logical relationship between $z_{ij}$ and $l_{ij}^k$ is 
\begin{align*}
    &z_{ij}=1 \implies \sum_{k=1}^n l_{ij}^k = n \quad  i, j \in \{1,\ldots,N\},  \\
    &z_{ij}=0 \implies \sum_{k=1}^n l_{ij}^k \le (n-1) \quad  i, j \in \{1,\ldots,N\},
\end{align*}
and these can be formulated as algebraic constraints
\begin{align}
    &\sum_{k=1}^n l_{ij}^k \ge nz_{ij} \quad  i, j \in \{1,\ldots,N\}, \label{eqn:zl1} \\
    &\sum_{k=1}^n l_{ij}^k  \le (n-1) + z_{ij} \quad  i, j \in \{1,\ldots,N\},   \label{eqn:zl2}
\end{align}
where $n \in \{x \in \mathbb{N}: x \ge 3 \} $.

In this way, we have converted the original logical constraints \cref{eq:zab1} and \cref{eq:zab2}, into equivalent algebraic inequalities \cref{eqn:lab1}, \cref{eqn:lab2}, \cref{eqn:zl1} and \cref{eqn:zl2}. These constraints give the relation between $2n$ continuous variables $ a_k, b_k \in \mathbb{R}, \; k \in \{ 1,\ldots,n \} $ and $N^2n$ binary variables $l_{ij}^k \in \{0,1\}, \; \ i, j \in \{1,\ldots,N\}, \; k \in \{ 1,\ldots,n \}$, $N^2$ binary variables $z_{ij} \in \{0,1\}, \;  i, j \in \{1,\ldots,N\}$.

\subsection{Constraints on Discrete Variables}
\label{subsecdv}

We would like to enforce that the grid point $ \boldsymbol{g}_{\Bar{i}\,\Bar{j}} := (x_{\Bar{i}},y_{\Bar{j}}) $ lies inside the  polygon obtained in \cref{subsecconti}. Recall that, the binary variable $z_{ij}=1$ if and only if the grid point $\boldsymbol{g}_{ij}$ lies inside the polygon. Therefore, we impose that
\begin{equation}
    z_{\bar{i}\,\bar{j}} = 1.
    \label{zeq1}
\end{equation}

In fact, \cref{zeq1} is unnecessary as $(x_{\Bar{i}}, y_{\Bar{j}}) \in S := \{(x,y): \bigwedge\limits_{k=1}^n a_k(x-x_{\Bar{i}}) + b_k(y-y_{\Bar{j}}) - 1 \le 0 \}$ given by \cref{corollary} in \cref{subsecconti}, which means this grid point must lie inside the  polygon. However, we explicitly list it here for better clarity.

The convex polygon is intended to approximate the PRS. To guarantee a level of confidence of the PRS, we impose that the sum of the normalized weights $w_{ij}$ of the grid points which lie inside the polygon should be greater than this confidence level $\alpha$. That is,
\begin{equation}
    \sum_{i=1}^N \sum_{j=1}^N w_{ij} z_{ij} \geq \alpha.
    \label{confidencegua}
\end{equation}

\cref{zeq1} and \cref{confidencegua} lead to two constraints that we formulate for $N^2$ binary variables $z_{ij} \in \{0,1\}, \;  i, j \in \{1,\ldots,N\}$.

\subsection{Formulation of MINLP Optimization Framework}

To solve the problem described in \cref{subsecdp}, we have chosen the objective function $\min \sum\limits_{i=1}^N \sum\limits_{j=1}^N z_{ij}$ in \cref{eqn:obj}, and introduced the following three types of decision variables with $n \in \{x \in \mathbb{N}: x \ge 3 \} $ in \cref{subsecdp}:
\begin{itemize}
    \item $2n$ continuous variables $ a_k, b_k \in \mathbb{R}, \;  k \in \{ 1,\ldots,n \} $;  
    
    \item $N^2n$ binary variables $l_{ij}^k \in \{0,1\}, \,  i, j \in \{1,\ldots,N\}, \,  k \in \{ 1,\ldots,n \}$;
    
    \item $N^2$ binary variables $z_{ij} \in \{0,1\}, \;  i, j \in \{1,\ldots,N\}$.
\end{itemize}

%Putting together all onstraints, we have 
%and formulated all needed constraints \cref{eqn:detcon}, \cref{eqn:no1cons} from \cref{subsecconti}; \cref{eqn:lab1}, \cref{eqn:lab2}, \cref{eqn:zl1}, \cref{eqn:zl2} from \cref{subsecbridge}; \cref{zeq1}, \cref{confidencegua} from \cref{subsecdv}.

Collect the previous constraints \cref{eqn:detcon}, \cref{eqn:no1cons} from \cref{subsecconti}; \cref{eqn:lab1}, \cref{eqn:lab2}, \cref{eqn:zl1}, \cref{eqn:zl2} from \cref{subsecbridge}; \cref{zeq1}, \cref{confidencegua} from \cref{subsecdv}. 

We formally formulate the problem stated in\cref{Problem_Statement} as an MINLP optimization framework below. All coordinates are expressed in a global coordinate system. If not specified, $ \forall i,j \in \{1,\ldots,N\}, \;  \forall k \in \{1,\ldots,n\}$.
% end section

\begin{equation}
\begin{aligned}
    & \min \sum_{i=1}^N \sum_{j=1}^N z_{ij}  \\[0.3cm]
    & \text{s.t.} \quad a_ib_j-b_ia_j > 0, \quad  \forall i \in \{ 1,\ldots,n \}, j = i^{\oplus}; \\[0.3cm]
    & \phantom{s.t.} \quad -a_k\left(b_i-b_j \right)+b_k\left(a_i-a_j\right) < a_ib_j-b_ia_j, \quad \forall i \in \\ 
    & \phantom{s.t.} \quad \{ 1,\ldots,n \}, \; j = i^{\oplus}, \; \forall k \in \{ 1,\ldots,n \}-\{i, j\}; \\[0.3cm]
    & \phantom{s.t.} \quad a_k (x_i-x_{\Bar{i}}) +b_k (y_j-y_{\Bar{j}}) -1 \le M_{ijk}^1(1 - l_{ij}^k), \quad  \forall i,j,k; \\[0.3cm]
    & \phantom{s.t.} \quad -a_k (x_i-x_{\Bar{i}}) -b_k (y_j-y_{\Bar{j}}) + 1 < M_{ijk}^2 l_{ij}^k, \quad  \forall i,j,k; \\[0.3cm]
    & \phantom{s.t.} \quad \sum_{k=1}^n l_{ij}^k \ge n z_{ij}, \quad  \forall i, j; \\[0.3cm]
    & \phantom{s.t.} \quad \sum_{k=1}^n l_{ij}^k \le (n-1) + z_{ij}, \quad  \forall i, j; \\[0.3cm]
    & \phantom{s.t.} \quad z_{\bar{i} \, \bar{j}} = 1; \\[0.3cm]
    & \phantom{s.t.} \quad \sum_{i=1}^N \sum_{j=1}^N w_{ij} z_{ij} \geq \alpha; \\[0.3cm]
    & \phantom{s.t.} \quad a_k, b_k \in \mathbb{R}, \quad  \forall k; \\[0.3cm]
    & \phantom{s.t.} \quad l_{ij}^k \in \{0,1\}, \quad  \forall i, j, k; \\[0.3cm]
    & \phantom{s.t.} \quad z_{ij} \in \{0,1\}, \quad  \forall i, j.
    \label{eqn:framework}
\end{aligned}
\end{equation}

\section{Solution Method}
\label{Solution_Method}
\medskip

In this section, we develop a heuristic algorithm to efficiently solve the previous MINLP optimization.

\subsection{MINLP Optimal Algorithm}

Although the optimization framework \cref{eqn:framework} is an MINLP problem with nonlinear constraints, all constraints except for \cref{eqn:detcon} and \cref{eqn:no1cons} are in fact linear ones. Fortunately, even the quadratic constraints \cref{eqn:detcon} and \cref{eqn:no1cons} are bilinear constraints that only involve the product of disjoint pairs of variables. Gurobi is well-suited for addressing such constraints. It employs cutting planes and branching algorithm; see~\cite{mitchell2002branch} for more information. Leveraging this solver, \textbf{\cref{alg:opt}} finds an optimal solution to \cref{eqn:framework}, which results in a convex approximation of the PRS computed by \cref{alg:kde}.

\begin{algorithm} [!h]
\caption{MINLP Optimal Algorithm}
\label{alg:opt}
\begin{algorithmic} [1] 
    \Function{MINLPSolver}{$\alpha$, $\boldsymbol{g}$, $\boldsymbol{w}$} 
        \State $\bar{i}$, $\bar{j}$ = $\arg \max (\boldsymbol{w})$
        \State Determine $x_{\Bar{i}}$, $y_{\Bar{j}}$ according to $\bar{i}$, $\bar{j}$ and $\boldsymbol{g}$
        \State Formulate \cref{eqn:framework} for grid points $\boldsymbol{g}$ with weights $\boldsymbol{w}$ \Statex[1] given confidence level $\alpha$
        \State Implement cutting planes and branching algorithm \Statex[1] to solve \cref{eqn:framework} for $a_k$, $b_k$
        \State \Return $a_k$, $b_k$, $x_{\Bar{i}}$, $y_{\Bar{j}}$
    \EndFunction
    \Statex
    
    \Function{FindPolygon}{$\boldsymbol{g}$, $\boldsymbol{w}$, $\alpha$} % tell me or I do operation myself  % remember the original data samples not parameter here
        \State $a_k$, $b_k$, $x_{\Bar{i}}$, $y_{\Bar{j}}$ = \Call{MINLPSolver}{$\alpha$, $\boldsymbol{g}$, $\boldsymbol{w}$}
 
        \State \Return $a_k$, $b_k$, $x_{\Bar{i}}$, $y_{\Bar{j}}$
    \EndFunction
    \Statex
    % \ means lineskip, \mathrm \operatorname \text, function name capital,  = symbol in mathmode or textmode, direct text or text in mathmode
    
    \State \Call{FindPolygon}{$\boldsymbol{g}$, $\boldsymbol{w}$, $\alpha$}

\end{algorithmic}  
\end{algorithm}

% \footnote{\magenta{since this is a built in algorithm, we don't have to spend space here for details. I would omit the figure as well, reviewers will probably complain this is not done by us, so we have to remove it.}}. 

%\begin{figure}[!h] %[!h]  [hbt!]
%\centering
%\includegraphics[width=.35\textwidth]%{figs/cutting_plane.png}
%\caption{Cutting Planes and Branching algorithm %implemented by Gurobi} 
%\label{fig:cuttingplane}
%\end{figure}
%% https://assets.gurobi.com/pdfs/user-events/2016-frankfurt/Die-Algorithmen.pdf

\subsection{MINLP Heuristic Algorithm \label{sec:heuristic}}

However, \cref{alg:opt} is computationally expensive and scales poorly with a large number of grid points. To address this issue, we develop \textbf{\cref{alg:heuristic}} that can efficiently solve \cref{eqn:framework} while ensuring accuracy. \cref{alg:opt} just serves as a benchmark that provides the optimal solution to \cref{eqn:framework} which is used only for comparison purposes. As explained in \cref{fig:diagram_alg2}, this algorithm performs weighted sampling based on the KDE values to select representative grid points $\boldsymbol{g}'$ from all grid points $\boldsymbol{g}$. Unlike the original MINLP problem, which uses all grid points $\boldsymbol{g}$, a new MINLP problem can be formulated using representative grid points $\boldsymbol{g}'$. Then, the cutting planes and branching algorithm built in Gurobi can be applied to solve this newly formulated MINLP problem. As illustrated in \cref{fig:diggridsheu}, the optimal solution to the new MINLP problem is an  approximation to the solution of the original MINLP problem. By doing so, efficiency comes at the cost of optimality. Since the size of $\boldsymbol{g}'$ can be far less than $\boldsymbol{g}$, the new MINLP problem reduces the number of decision variables and constraints, greatly contributing to reducing computational time.

\begin{figure}[!h] %[!h]  [hbt!]
\centering
    \includegraphics[width=.45\textwidth]{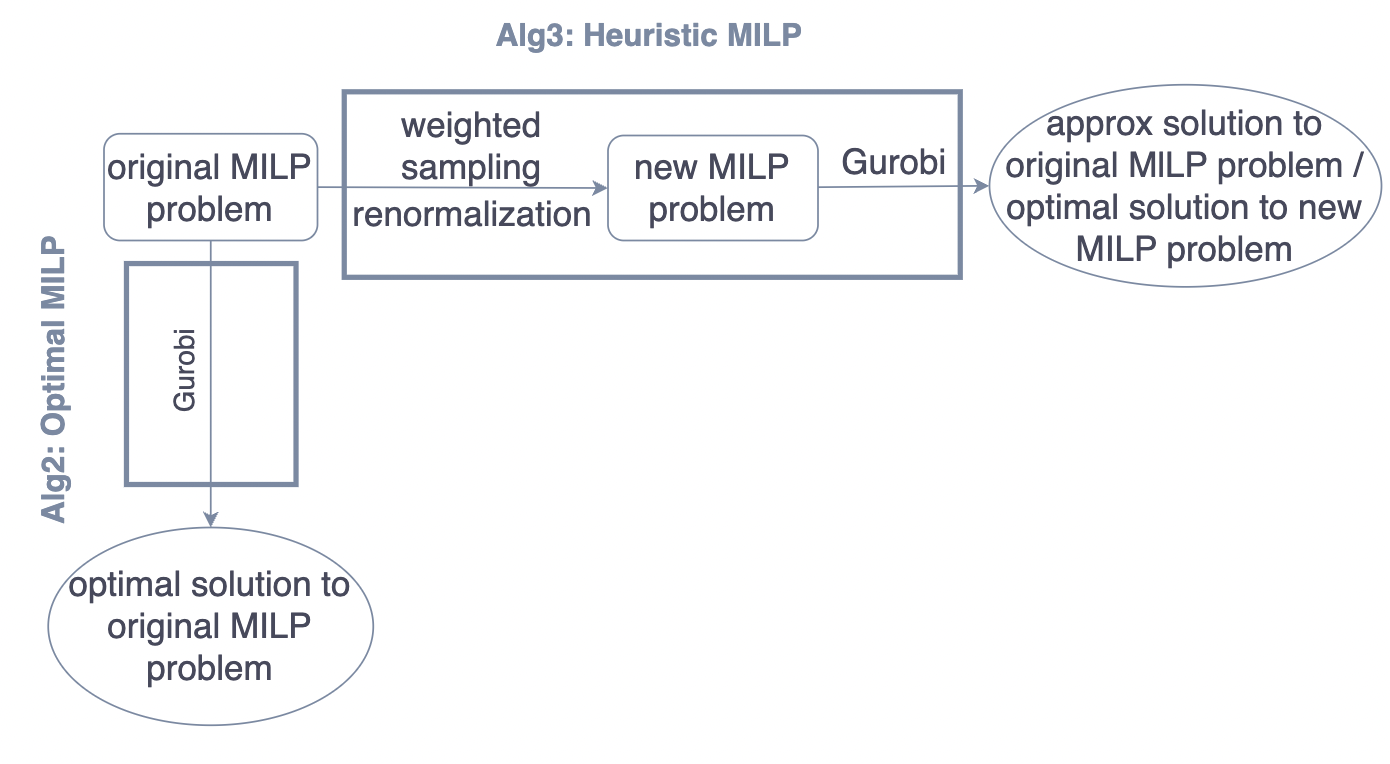}
\caption{Procedural difference between MINLP Heuristic and MINLP Optimal}
\label{fig:diagram_alg2}
\end{figure}

\begin{figure}[!h] %[!h]  [hbt!]
\centering
    \includegraphics[width=.4\textwidth]{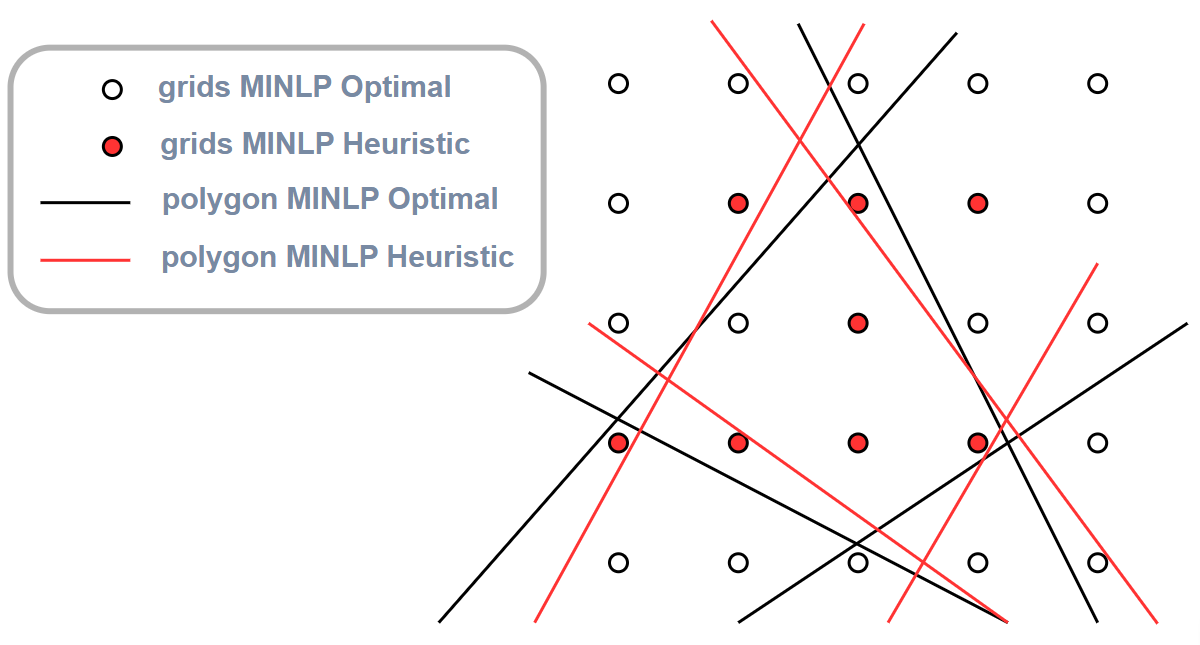}
\caption{Results difference between MINLP Heuristic and MINLP Optimal} 
\label{fig:diggridsheu}
\end{figure}

In the pseudo-code of \cref{alg:heuristic},  we use $N_s$ to represent the number of representative grid points $\boldsymbol{g}'$. From Lines 1 through 6, we use the AES algorithm proposed by \cite{efraimidis2006weighted} to randomly select $N_s$ representative grid points $\boldsymbol{g}'$ from all $N^2$ grid points $\boldsymbol{g}$ without replacement, according to their normalized weights $\boldsymbol{w}$. Recall that $\boldsymbol{g}$ is the output of \cref{alg:kde} and $\boldsymbol{w}:= [w_{ij}]$ is obtained through \cref{eqn:normw}. From Line 7 through 9, the original weights $\boldsymbol{w}'$ of representative grid points $\boldsymbol{g}'$ are re-normalized to $\hat{\boldsymbol{w}}'$, so that their sum is equal to one. The process of re-normalization is crucial to guarantee the accuracy of the solution obtained from \cref{alg:heuristic}, when compared with the solution after implementing \cref{alg:opt}.

\begin{algorithm} [!h]
\caption{MINLP Heuristic Algorithm}
\label{alg:heuristic}
\begin{algorithmic} [1] 
	\Function{WgtSamp}{$\boldsymbol{g}$, $\boldsymbol{w}$, $N_s$} 
    	\State $u_{ij} = \operatorname{random}(0, 1)$
        \State $k_{ij} = u_{ij}^{1/w_{ij}}$
    	\State Take first $N_s$ grids from $\boldsymbol{g}$  in descending order of $k_{ij}$ \Statex[1]  as representative grids $\boldsymbol{g}'$
    	\State Take the weights of those $N_s$ grids from $\boldsymbol{w}$  
    	\Statex[1] as representative weights $\boldsymbol{w}'$
    	\State \Return $\boldsymbol{g}'$, $\boldsymbol{w}'$
    \EndFunction 
    \Statex
    % \Statex[1] % \newline \par \\[]   \hspace*{}    % \phantom
    
    \Function{NormWgt}{$\boldsymbol{w}'$, $N_s$}
        \State $ \hat{\boldsymbol{w}}' = \boldsymbol{w}' + \operatorname{ones}(N_s) \cdot ( 1 - \operatorname{sum}(\boldsymbol{w}') ) / N_s $
        \State \Return $\hat{\boldsymbol{w}}'$
    \EndFunction
    \Statex
    
    \Function{MINLPSolver}{$\boldsymbol{g}'$, $\hat{\boldsymbol{w}}'$, $\alpha$, $\boldsymbol{g}$} 
        \State $\bar{i}$, $\bar{j} = \arg \max (\hat{\boldsymbol{w}}')$
        \State Determine $x_{\Bar{i}}$, $y_{\Bar{j}}$ according to $\bar{i}$, $\bar{j}$ and $\boldsymbol{g}$
        \State Formulate \cref{eqn:framework} for grid points $\boldsymbol{g}'$ with weights $\hat{\boldsymbol{w}}'$ 
        \Statex[1] given confidence level $\alpha$
        \State Implement cutting planes and branching algorithm \Statex[1] to solve \cref{eqn:framework} for $a_k$, $b_k$
        \State \Return $a_k$, $b_k$, $x_{\Bar{i}}$, $y_{\Bar{j}}$
    \EndFunction
    \Statex
    
    \Function{FindPolygon}{$\boldsymbol{g}$, $\boldsymbol{w}$, $N_s$, $\alpha$} % tell me or I do operation myself  % remember the original data samples not parameter here
            \State $\boldsymbol{g}'$, $\boldsymbol{w}'$ = \Call{WgtSamp}{$\boldsymbol{g}$, $\boldsymbol{w}$, $N_s$}
            \State $\hat{\boldsymbol{w}}'$ = \Call{NormWgt}{$\boldsymbol{w}'$, $N_s$}
            \State $a_k$, $b_k$, $x_{\Bar{i}}$, $y_{\Bar{j}}$ = \Call{MINLPSolver}{$\boldsymbol{g}'$, $\hat{\boldsymbol{w}}'$, $\alpha$, $\boldsymbol{g}$}
            \State Determine the polygon $S$ according to $a_k$, $b_k$, $x_{\Bar{i}}$, $y_{\Bar{j}}$
            \State \Return $S$
    \EndFunction
    \Statex
    % \ means lineskip, \mathrm \operatorname \text, function name capital,  = symbol in mathmode or textmode, direct text or text in mathmode
    
    \State \Call{FindPolygon}{$\boldsymbol{g}$, $\boldsymbol{w}$, $N_s$, $\alpha$}

\end{algorithmic}  
\end{algorithm}

\section{Main Results}
\label{Main_Results}
\medskip

In this section, we conduct comprehensive case studies to compare the performance of the MINLP Optimal algorithm (\cref{alg:opt}), and the MINLP Heuristic algorithm (\cref{alg:heuristic}), with the Bounding Box algorithm introduced in \cite{ericson2004real}, as well as the impact of different parameters on the performance of the algorithms. We also discuss the robustness of the MINLP Heuristic algorithm. The tests were implemented in Python 3.9 and on an Intel(R) Core(TM) i9-12900KF, 3187 Mhz, 16 Core(s), 24 Logical Processor(s) Desktop with 64GB RAM.

Case studies are divided into two stages:

1) Solution stage: According to KDE values and the PRS obtained from a collection of $M$ data samples running \cref{alg:kde}, we formulate an MINLP problem. Then, MINLP Optimal and MINLP Heuristics are implemented respectively to find convex approximations of the PRS. As a comparison, the Bounding Box algorithm is applied to find a convex quadrilateral bounding the PRS;

2) Testing stage: Given the obtained convex polygon in the previous stage, we can generate a new collection of $N_{\text{test}}$ ($N_{\text{test}} \gg M$) data samples, and evaluate the ratio of the number of data samples that fall inside the polygon to the total number of data samples generated. As $N_{\text{test}}$ increases, the ratio will converge to the true probability of the system state lying inside the polygon.

% \blue{ We are considering to combine case 1 and 2, capital title}

\subsection{Cases Settings}

\subsubsection{Case I}

In this case, we consider a Dubins vehicle model moving on a plane. The speed of the vehicle obeys a truncated Gaussian distribution $v \sim \mathcal{N}(\mu_1, \sigma_1)$ with $\mu_1 = \SI{190}{km/h}$ and $\sigma_1 = \SI{5}{km/h}$, and the speed falls inside the interval $[165, 220]\,\unit{km/h}$. The heading angle of the vehicle obeys another truncated Gaussian distribution $\theta \sim \mathcal{N}(\mu_2, \sigma_2)$ with $\mu_2 = \SI{10}{degs}$ and $\sigma_2 = \SI{30}{degs}$, and the heading angle falls inside the interval $[-50, 70] \,\unit{degs}$. 

% \footnote{\magenta{this is a single vehicle, it's not a multi-agent system}} 

% \footnote{\magenta{what do you mean by "model based and bounded"? let me know}}

The position $\boldsymbol{x}_{k}:= \begin{bmatrix}
    x_{k}, y_{k}
\end{bmatrix}^{\top}$ of the vehicle at time $k$ can be modeled by a discrete-time dynamic system
\begin{equation*}
\begin{bmatrix}
    x_k \\
    y_k
\end{bmatrix} = 
\begin{bmatrix}
    x_{k-1} \\
    y_{k-1}
\end{bmatrix} + 
\begin{bmatrix}
    \cos{\theta} \\
    \sin{\theta}
\end{bmatrix} v \Delta t,
    % \label{eqn:kinematic}
\end{equation*}
where $\boldsymbol{x}_{k-1}:= \begin{bmatrix}
    x_{k-1}, y_{k-1}
\end{bmatrix}^{\top}$ is the position of the vehicle at time $k-1$,  and the speed $v$ and heading angle $\theta$ are assumed to be constant during a time interval $\Delta t = \SI{1}{s} $. Due to the uncertainties arising from $v$ and $\theta$ of the vehicle, the position $\boldsymbol{x}_{k}$ is a random vector that obeys a non-Gaussian distribution, which we call a fan-shaped distribution. The shadowed region in \cref{fig:area} indicates the reachable set of the position of the vehicle at a time point. 
\begin{figure}[H] %[!h]
\centering
\includegraphics[width=.25\textwidth]{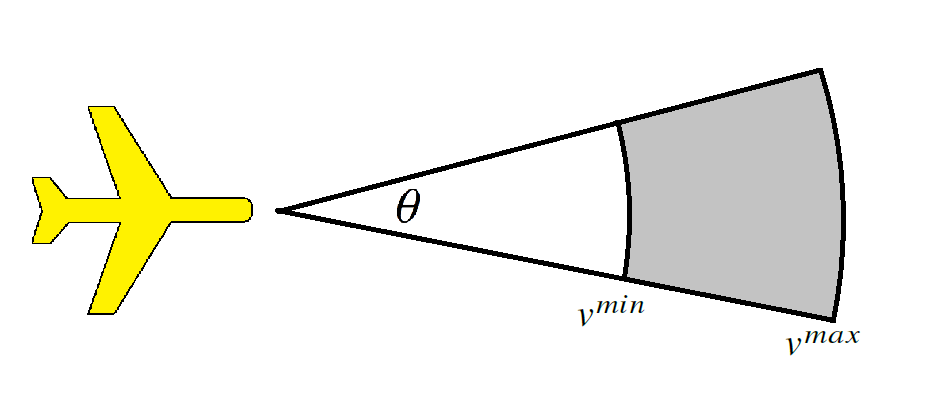}
\caption{Reachable set of the position of a vehicle at a time point}
\label{fig:area}
\end{figure}

% \footnote{\magenta{let's omit the "real" position, because this is a vector, and we have been saying that this vehicle was moving on the plane}}

% \blue{Case II: Bimodal distribution, Model-free, Unbounded}

% \blue{REAL EXPERIMENT DAI2021 PUBLISH TRO}

\subsubsection{Case II}

In this case, we display a scatter plot of the possible positions (data samples) $[x,y]^{\top}$ of a vehicle on the plane at a time point, which is generated by the marginal histograms for $x$ and $y$ respectively. As shown in \cref{fig:fighist}, the position of the vehicle at that time point obeys a bimodal distribution which is non-Gaussian.
\begin{figure}[h] %[!h]
\centering
\includegraphics[width=.35\textwidth]{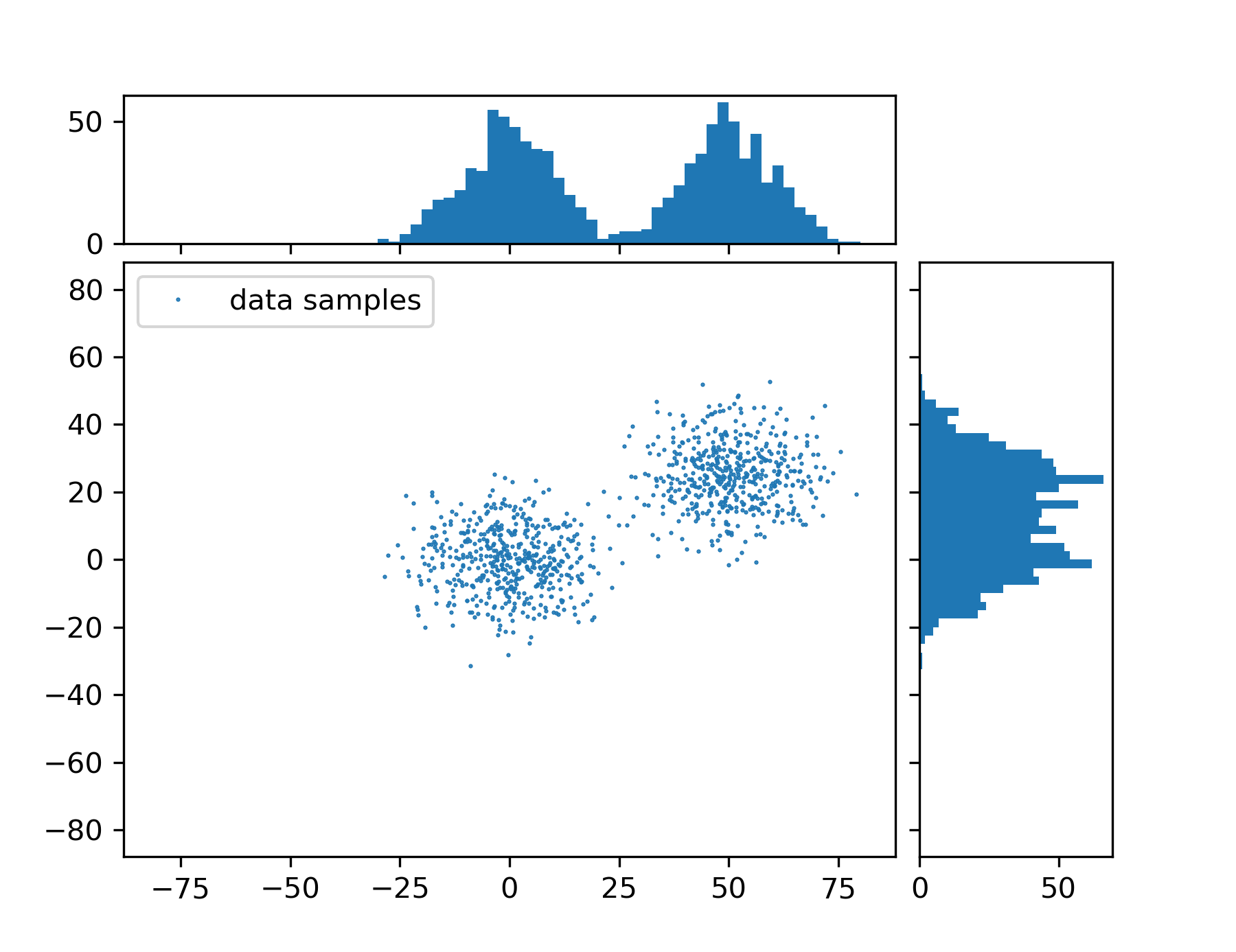}
\caption{Joint bimodal distribution generated by two marginal histograms}
\label{fig:fighist}
\end{figure}

In the following, the bandwidth matrix $\mathbf{H}$ in \cref{kernel_gaussian} is set to $\big[\begin{smallmatrix}
  0.2 & 0\\
  0 & 0.2
\end{smallmatrix}\big]$, a common constant \num{e4} is chosen for the parameters $M_{ijk}^1, M_{ijk}^2$ in \cref{eqn:lab1} and \cref{eqn:lab2}, and the number of data samples is $M = 1000$. The setting below is chosen as a baseline: The number of sides is $n = 4$, the confidence level is $\alpha = 90\%$, the number of grid points used by MINLP Optimal is $N^2 = 20^2$, and the number of grid points used by MINLP Heuristic is $N_s = 70$ for Case I and $N_s = 60$ for Case II. In the following figures and tables, the baseline is marked with the symbol ``*".

\subsection{Different Numbers of Sides}
\label{case1sides}
In this part, we compare the performance of different algorithms for different numbers of polygon sides: $n = 3, 4, 5$. The other parameters are the same as in the baseline. % analyze 

\begin{figure*}[htbp]
    \centering
    \subfloat[Case I: $n = 3$ ]{
    \label{sub-fig-fantri}
    %\begin{minipage}{ }
    %\centering
    \includegraphics[width=0.35\textwidth]{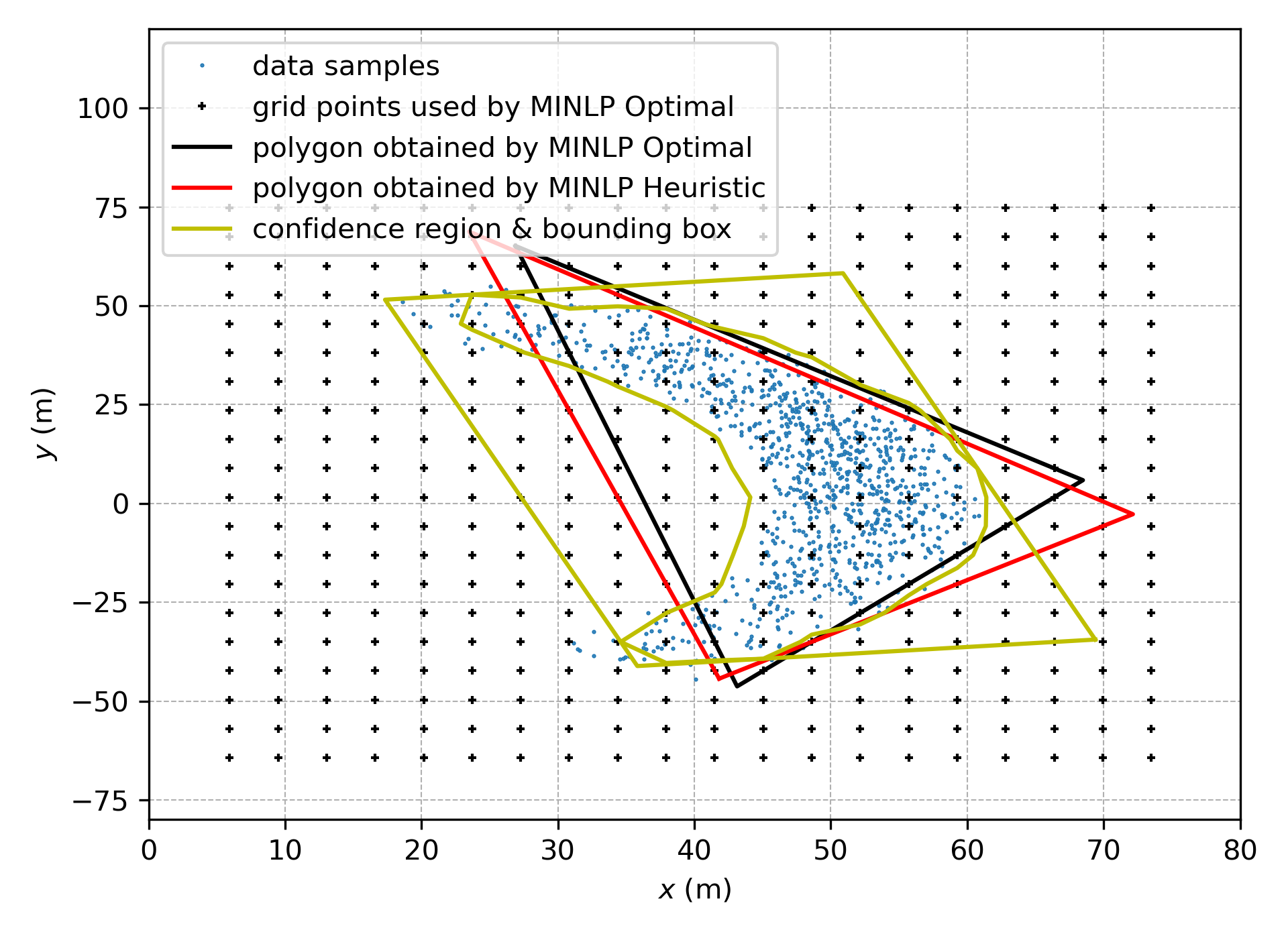}
    %\end{minipage}
    } \hspace{-0.2cm}
    \subfloat[Case II: $n = 3$ ]{
    \label{sub-fig-bitri}
    %\begin{minipage}{ }
    %\centering
    \includegraphics[width=0.35\textwidth]{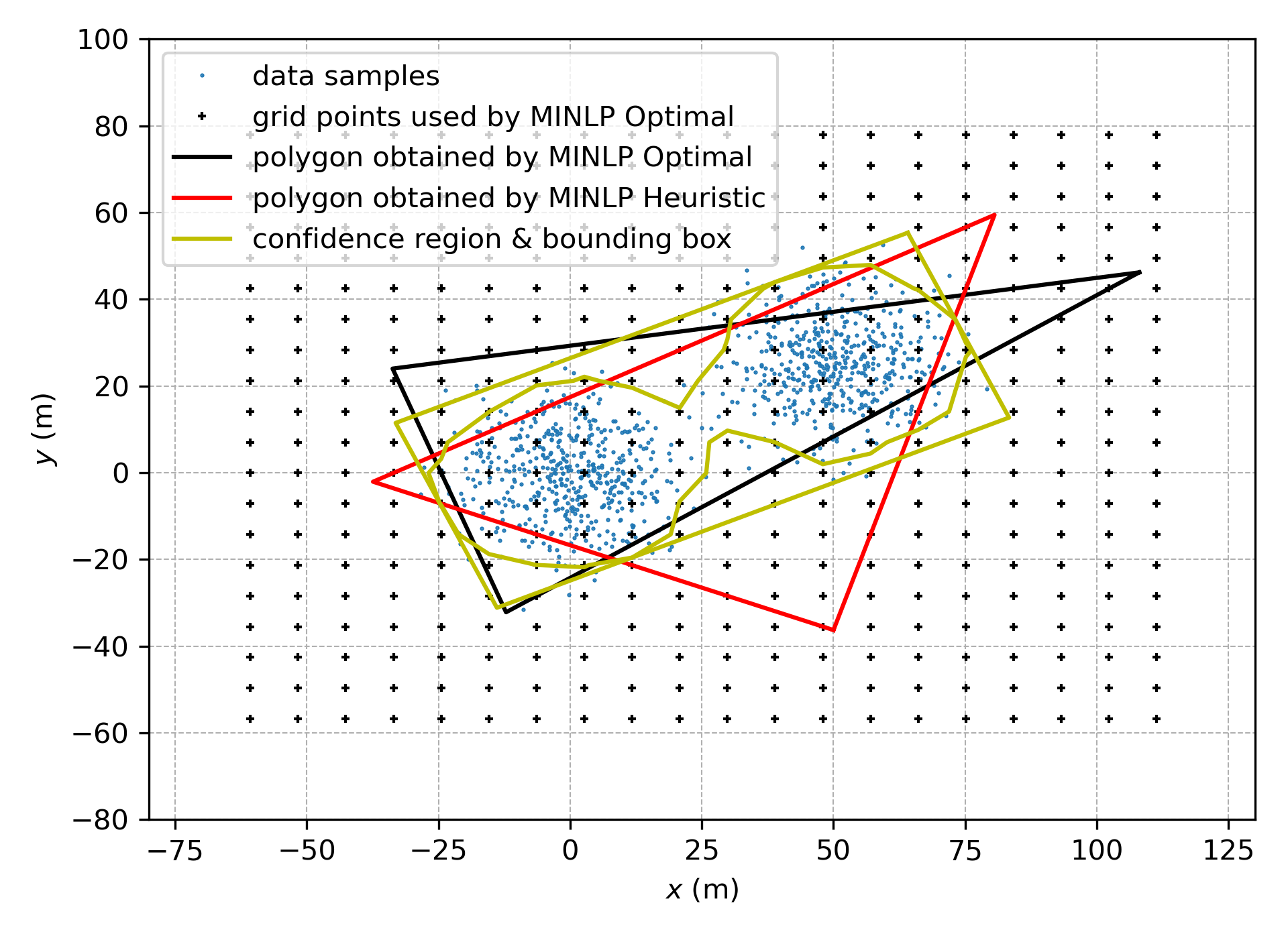}
    %\end{minipage}
    } \hspace{-0.2cm}
    \subfloat[Case I: $n = 4$ *]{
    \label{sub-fig-fanquad}
    %\begin{minipage}{ }
    %\centering
    \includegraphics[width=0.35\textwidth]{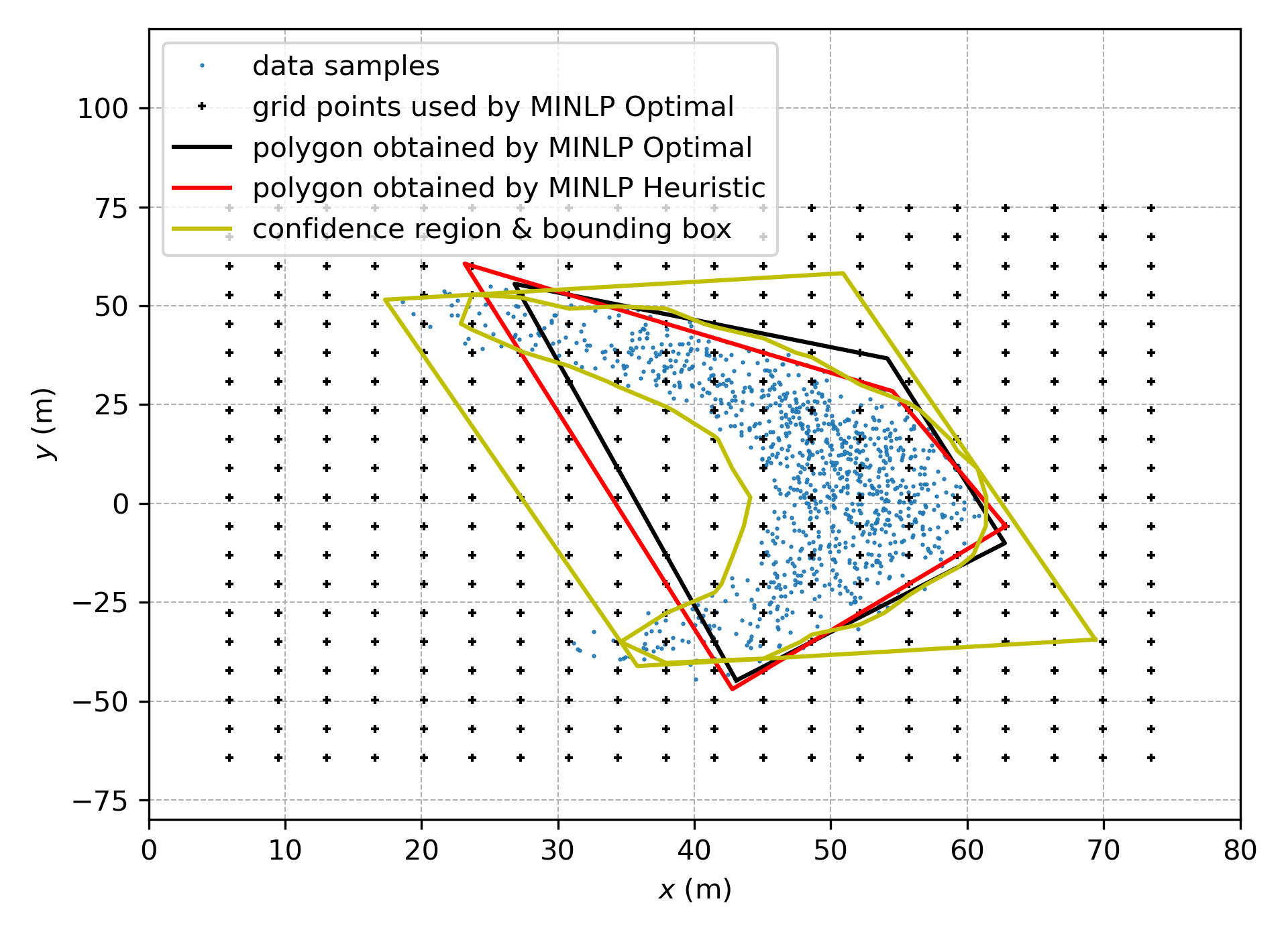}
    %\end{minipage}
    } \hspace{-0.2cm}%\quad
    \subfloat[Case II: $n = 4$ *]{
    \label{sub-fig-biquad}
    %\begin{minipage}{ }
    %\centering
    \includegraphics[width=0.35\textwidth]{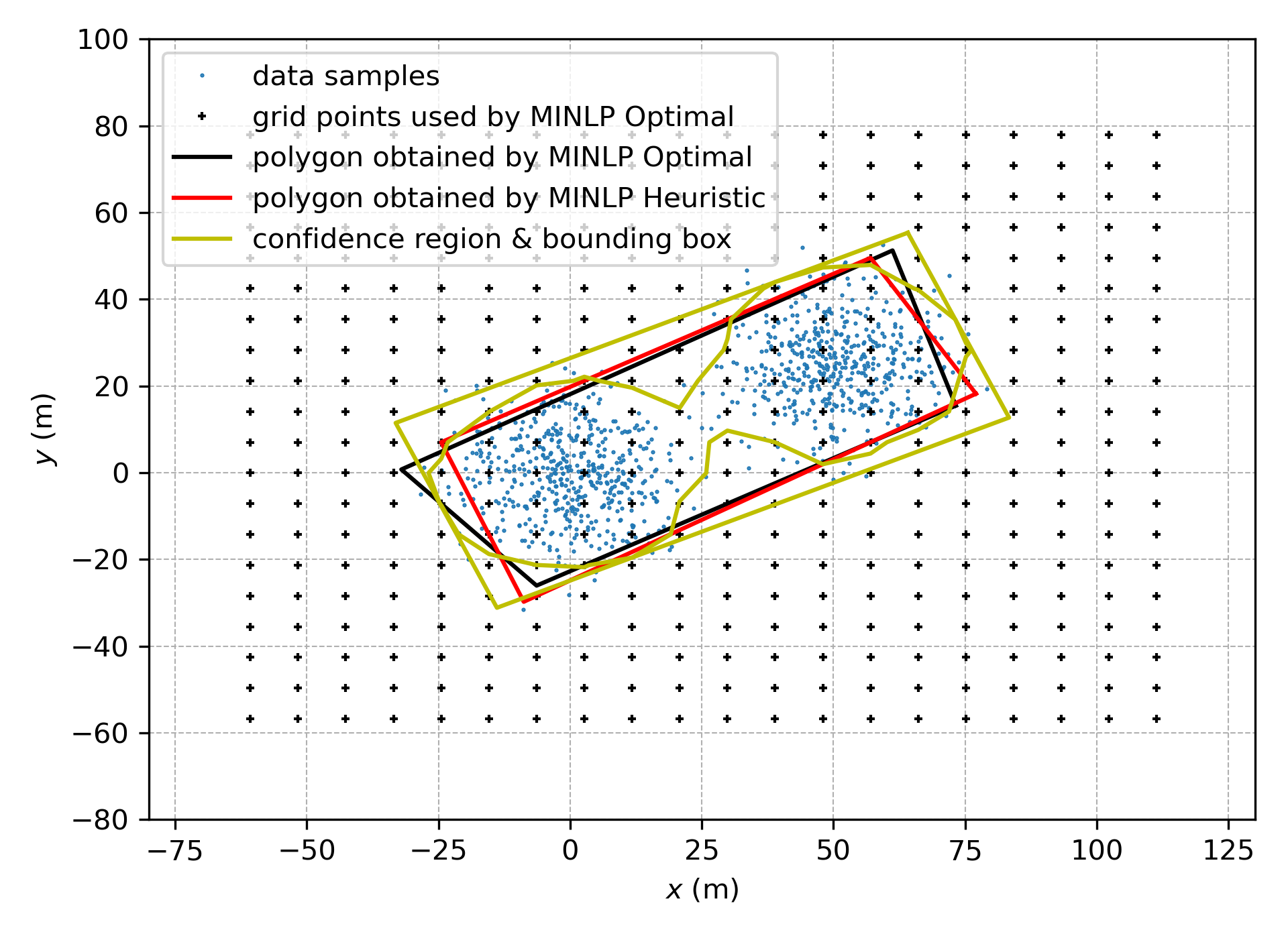}
    %\end{minipage}
    } \hspace{-0.2cm}%\quad
    \subfloat[Case I: $n = 5$]{
    \label{sub-fig-fanpentagon}
    %\begin{minipage}{ }
    %\centering
    \includegraphics[width=0.35\textwidth]{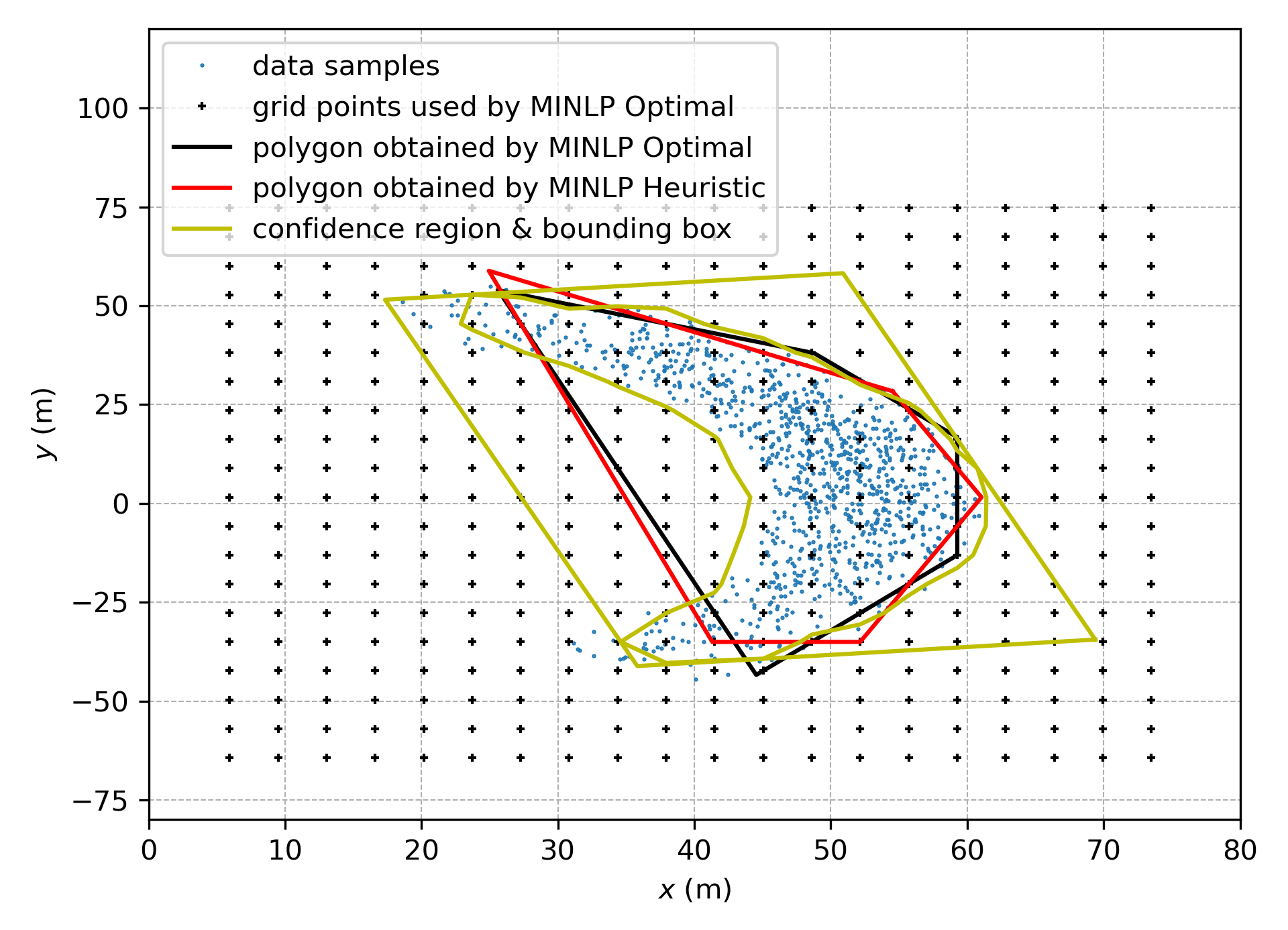}
    %\end{minipage}
    } \hspace{-0.2cm}
    \subfloat[Case II: $n = 5$]{
    \label{sub-fig-bipentagon}
    %\begin{minipage}{ }
    %\centering
    \includegraphics[width=0.35\textwidth]{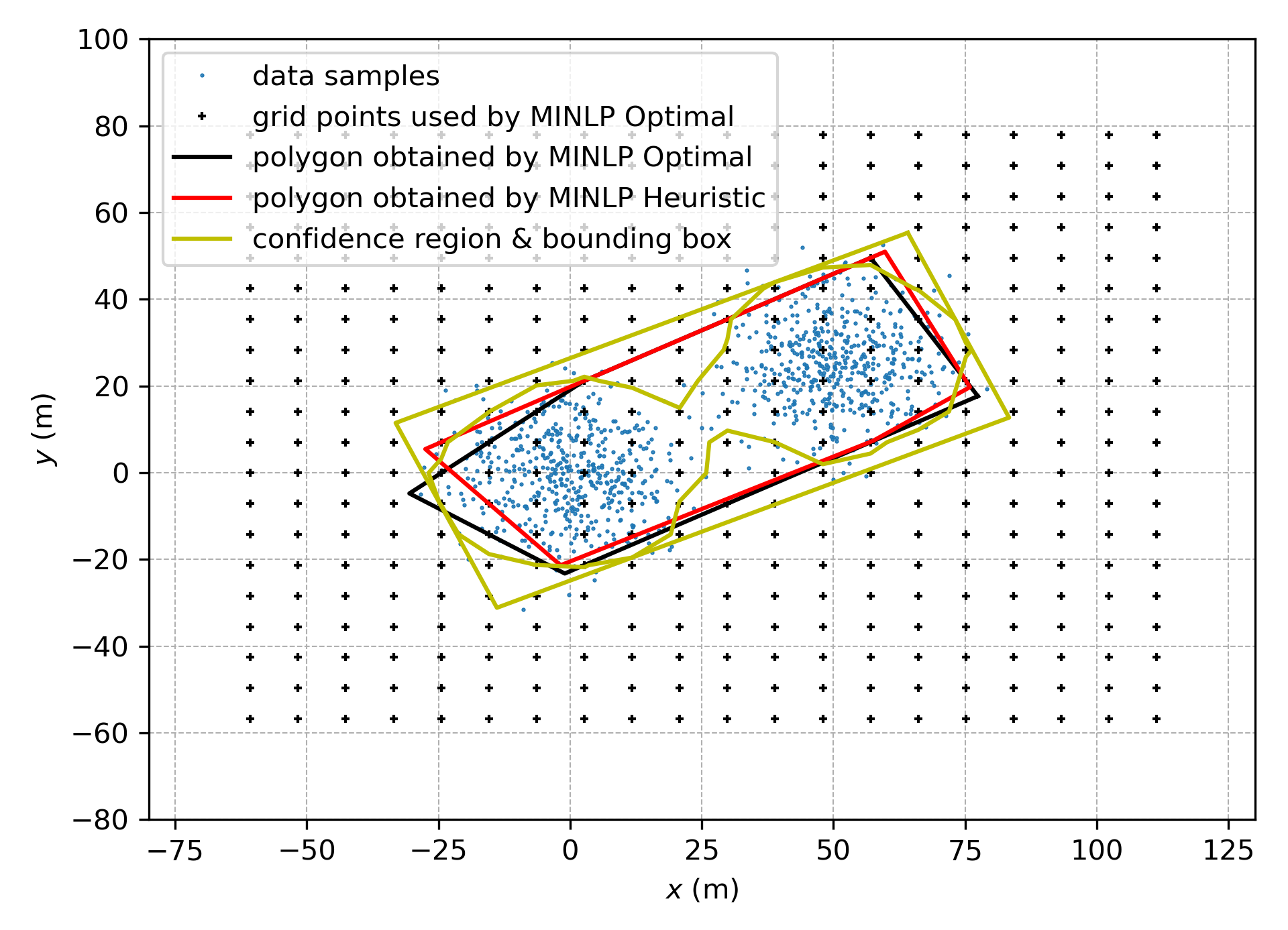}
    %\end{minipage}
    }
    \caption{Comparisons of different numbers of sides $n$ (For Bounding Box $ n \equiv 4$)}
    \label{fig:fancpgon}  %
\end{figure*} % end 

\begin{table*}[htbp]
\centering
\caption{Comparisons of different numbers of sides $n$ (For Bounding Box $ n\equiv 4$)}
\label{tab:fancpgon}
\begin{tabular}{llcccccc}
\toprule
\multicolumn{1}{c}{ \multirow{2}{*}{\# \textbf{Sides $n$}} } & \multicolumn{1}{c}{ \multirow{2}{*}{\textbf{Algorithm}} } & \multicolumn{3}{c}{\textbf{Case I}}  & \multicolumn{3}{c}{\textbf{Case II}}  \\ \cmidrule(lr){3-5} \cmidrule(lr){6-8}
{} & {} & \textbf{Ratio} & \textbf{Area (\unit{m^2})} & \textbf{Time (\unit{s})} & \textbf{Ratio} & \textbf{Area (\unit{m^2})} &\textbf{Time (\unit{s})} \\ \midrule
             & MINLP Optimal   & 90.1\%         & 1835.9              & 27.5   & 89.1\%         & 4220.7              & 13.7                   \\
3 (Triangle)  & MINLP Heuristic & 90.2\%         & 1911.9              & 0.533     & 90.1\%         & 4706.6              & 0.209                   \\
             & Bounding Box    & N/A         & N/A              & N/A      & N/A         & N/A              & N/A                \\ \midrule
             & MINLP Optimal   & 91.5\%         & 1827.0              & 70.5     & 91.4\%         & 3570.5              & 18.6                        \\
4 (Quadrilateral) *            & MINLP Heuristic & 90.7\%         & 1807.9              & 0.801    & 90.6\%         & 3457.1              & 0.361                  \\
             & Bounding Box    & 97.9\%         & 3235.1              & 0.216    & 97.7\%         & 4987.2              & 0.189                     \\ \midrule
             & MINLP Optimal   & 89.7\%         & 1667.9              & 316.3      & 91.3\%         & 3511.6              & 27.1                     \\
5 (Pentagon)           & MINLP Heuristic & 90.2\%         & 1721.9              & 1.145     & 90.5\%         & 3401.7              & 0.525                \\
             & Bounding Box    & N/A        & N/A              & N/A     & N/A        & N/A             & N/A                   \\ \bottomrule
\end{tabular}
\end{table*}

The results are shown in \cref{fig:fancpgon} and \cref{tab:fancpgon}. Note that changing of the number of sides $n$ is only applicable for MINLP Optimal and MINLP Heuristic. For MINLP Optimal or MINLP Heuristic, as $n$ increases from $3$ to $5$, the area of the polygon becomes smaller, while the ratio is just slightly different and always closer to the confidence level $\alpha$. This suggests that the results of  both algorithms become less conservative while ensuring accuracy. However, the associated computational time increases as the number of decision variables and constraints in \cref{eqn:framework} increases with an increase on  $n$. To achieve a balance between optimality and computational efficiency, we set $n=4$.

 % As the Bounding Box method requires a constant number of sides equal to 4, as shown in \cref{fig:fancpgon} and \cref{tab:fancpgon},  its results are always the same (ratio 97.9\%, area \SI{3235.1}{m^2}, time \SI{0.216}{s}).

When $n = 4$,  the area and ratio of either MINLP Optimal or MINLP Heuristic are much smaller than those of the Bounding Box algorithm, which suggests that the polygons obtained by both MINLP Optimal and MINLP Heuristic algorithms are more tight and accurate than the Bounding Box. This clearly benefits real motion planning. When considering collision avoidance between the vehicle and an obstacle, a less conservative convex approximation of the PRS allows a larger feasible search area which may provide a better-planned trajectory. Plus, the computational time of MINLP Heuristic is far less than MINLP Optimal  (\SI{0.801}{s} $<$ \SI{70.5}{s} for Case I and \SI{0.361}{s} $<$ \SI{18.6}{s} for Case II), while their results of area and ratio are still similar. This demonstrates the computational efficiency of MINLP Heuristic while ensuring accuracy, making online implementation feasible.

% because a less conservative set to be evaded often means less control effort to be spent.\footnote{\magenta{please elaborate, what are we evading?}} 

As stated in \cref{Solution_Method}, the polygon obtained by MINLP Optimal is optimal in the sense that it has a minimum area without violating constraints. However, for some cases in \cref{tab:fancpgon}, the area of MINLP Heuristic (e.g. $\SI{1807.9}{m^2}$) may be slightly less than MINLP Optimal (e.g. $\SI{1827.0}{m^2}$). The approximation error comes from the re-normalization with the newly selected $N_s$ grid points $\boldsymbol{g}'$. The polygon obtained by MINLP Heuristic is an optimal solution to a new MINLP problem based on those selected $N_s$ grid points, which is also an approximation to the original MINLP problem involving $N^2$ grid points.
Since the original and new MINLP problems have different grid settings, there is a chance that the polygon area of MINLP Heuristic is less than MINLP Optimal. However, those errors are relatively small according to our experiments.

\subsection{Different Numbers of Grid Points}
Here, we compare the performance of the algorithms when considering different numbers of grid points: $N^2 = 10^2, 20^2, 30^2, 40^2$. For MINLP Heuristic, $N_s =  20, 70, 160, 280$ grid points are randomly selected for Case I and $N_s = 15, 60, 135, 240$ grid points for Case II, respectively. The other parameters are the same as in the baseline. 

\begin{figure*}[htbp]
    \centering
    \subfloat[Case I:  $N^2 = 10^2$]{
    \label{sub-fig-grids10}
    %\begin{minipage}{ }
    %\centering
    \includegraphics[width=0.35\textwidth]{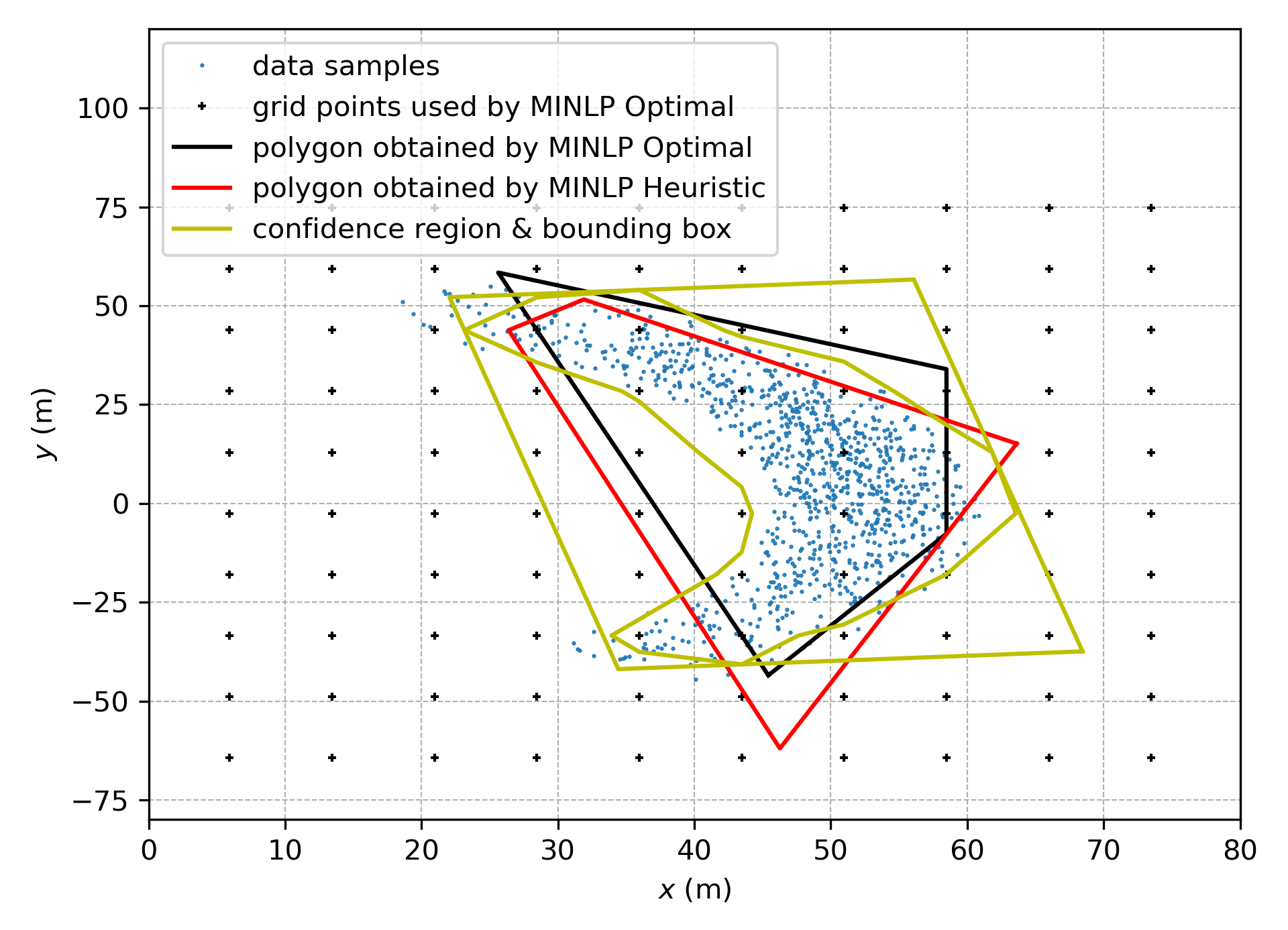}
    %\end{minipage}
    } \hspace{-0.2cm}
    \subfloat[Case II:  $N^2 = 10^2$]{
    \label{sub-fig-grids10bi}
    %\begin{minipage}{ }
    %\centering
    \includegraphics[width=0.35\textwidth]{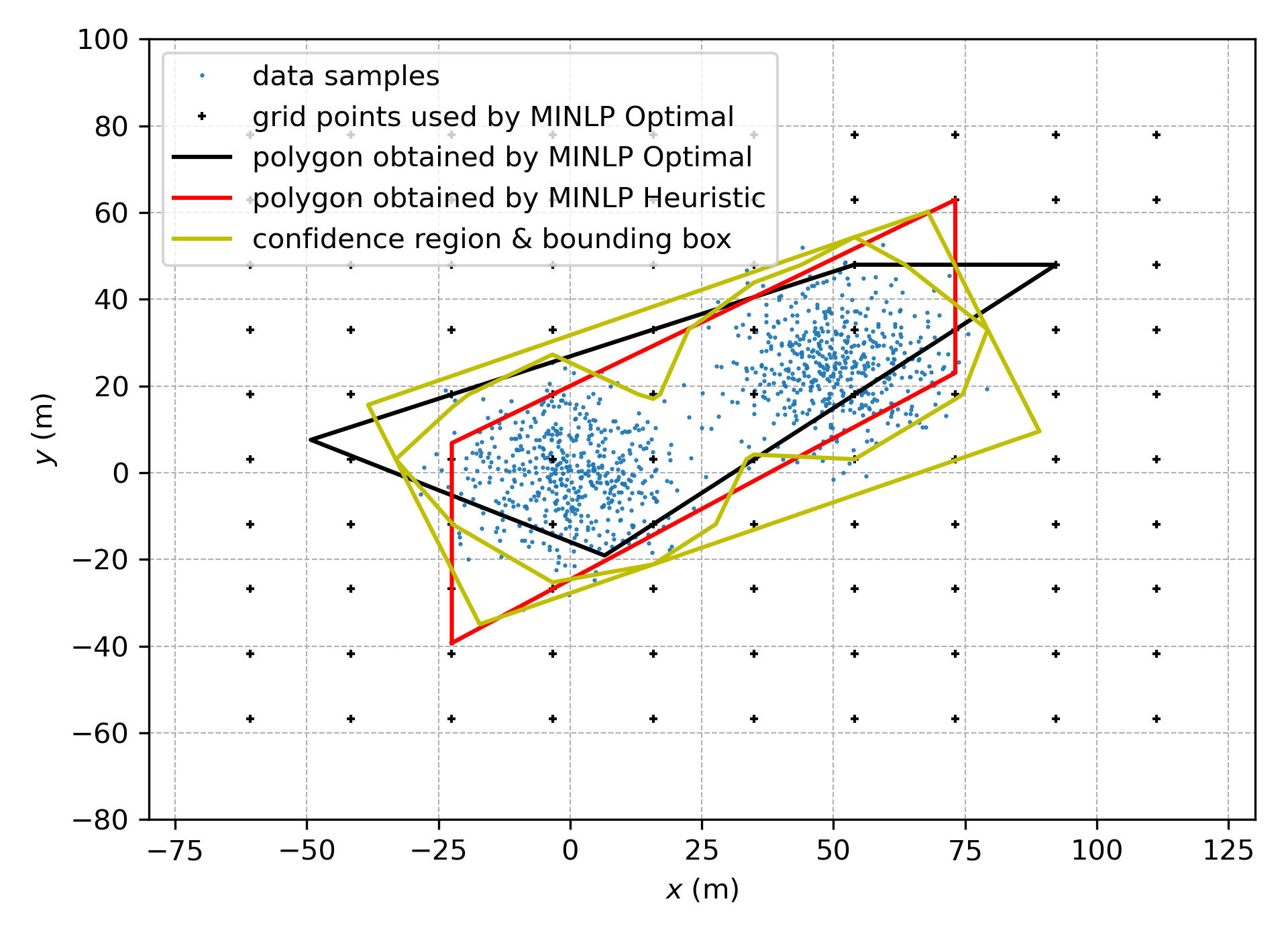}
    %\end{minipage}
    } \hspace{-0.2cm}
    \subfloat[Case I:  $N^2 = 20^2$ *]{
    \label{sub-fig-grids20}
    %\begin{minipage}{ }
    %\centering
    \includegraphics[width=0.35\textwidth]{figs/fan2020.png}
    %\end{minipage}
    } \hspace{-0.2cm}
    \subfloat[Case II:  $N^2 = 20^2$ *]{
    \label{sub-fig-grids20bi}
    %\begin{minipage}{ }
    %\centering
    \includegraphics[width=0.35\textwidth]{figs/fig2020.png}
    %\end{minipage}
    } \hspace{-0.2cm}
    \subfloat[Case I: \#  $N^2 = 30^2$]{
    \label{sub-fig-grids30}
    %\begin{minipage}{ }
    %\centering
    \includegraphics[width=0.35\textwidth]{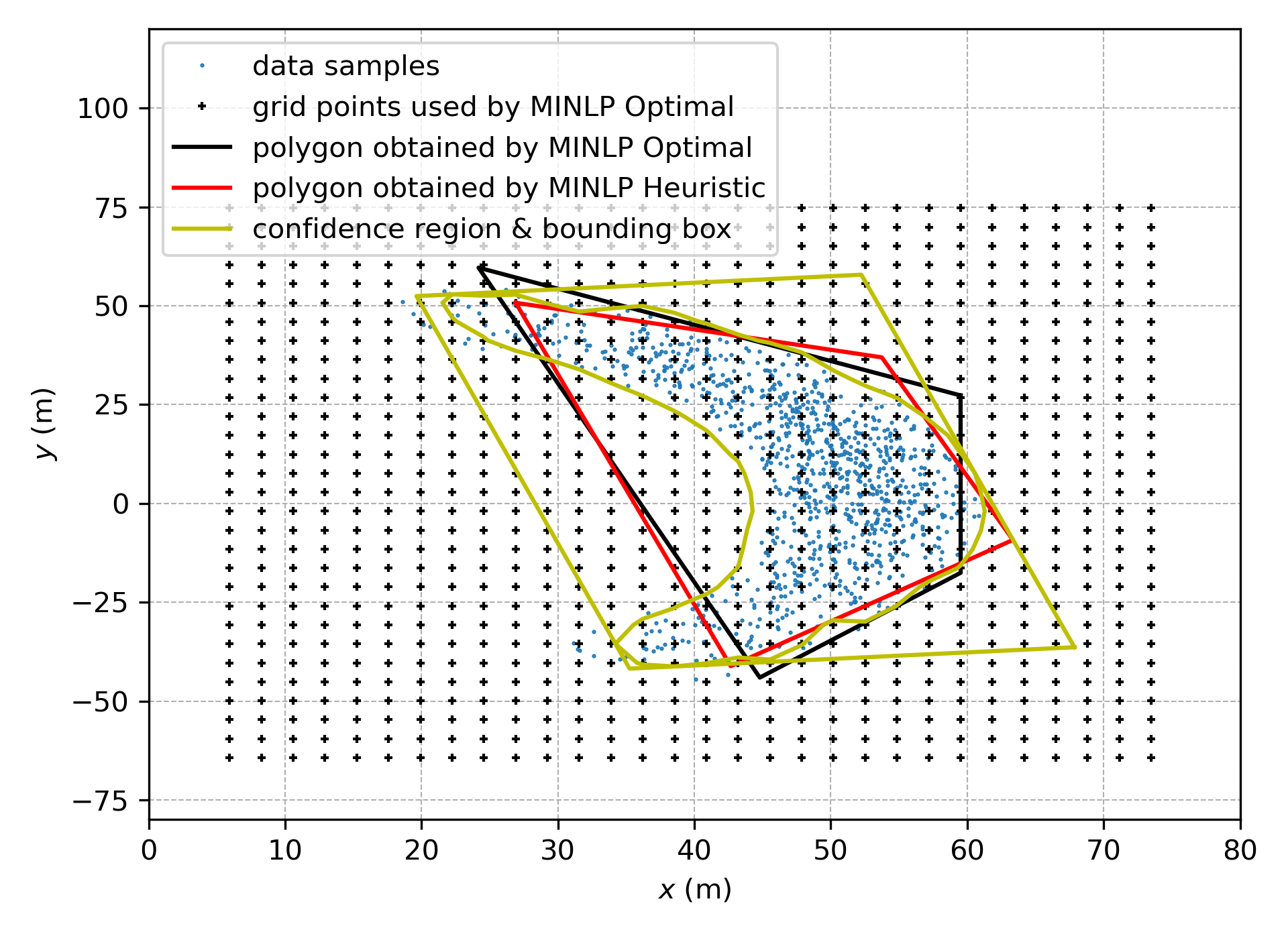}
    %\end{minipage}
    } \hspace{-0.2cm}%  
    \subfloat[Case II: $N^2 = 30^2$]{
    \label{sub-fig-grids30bi}
    %\begin{minipage}{ }
    %\centering
    \includegraphics[width=0.35\textwidth]{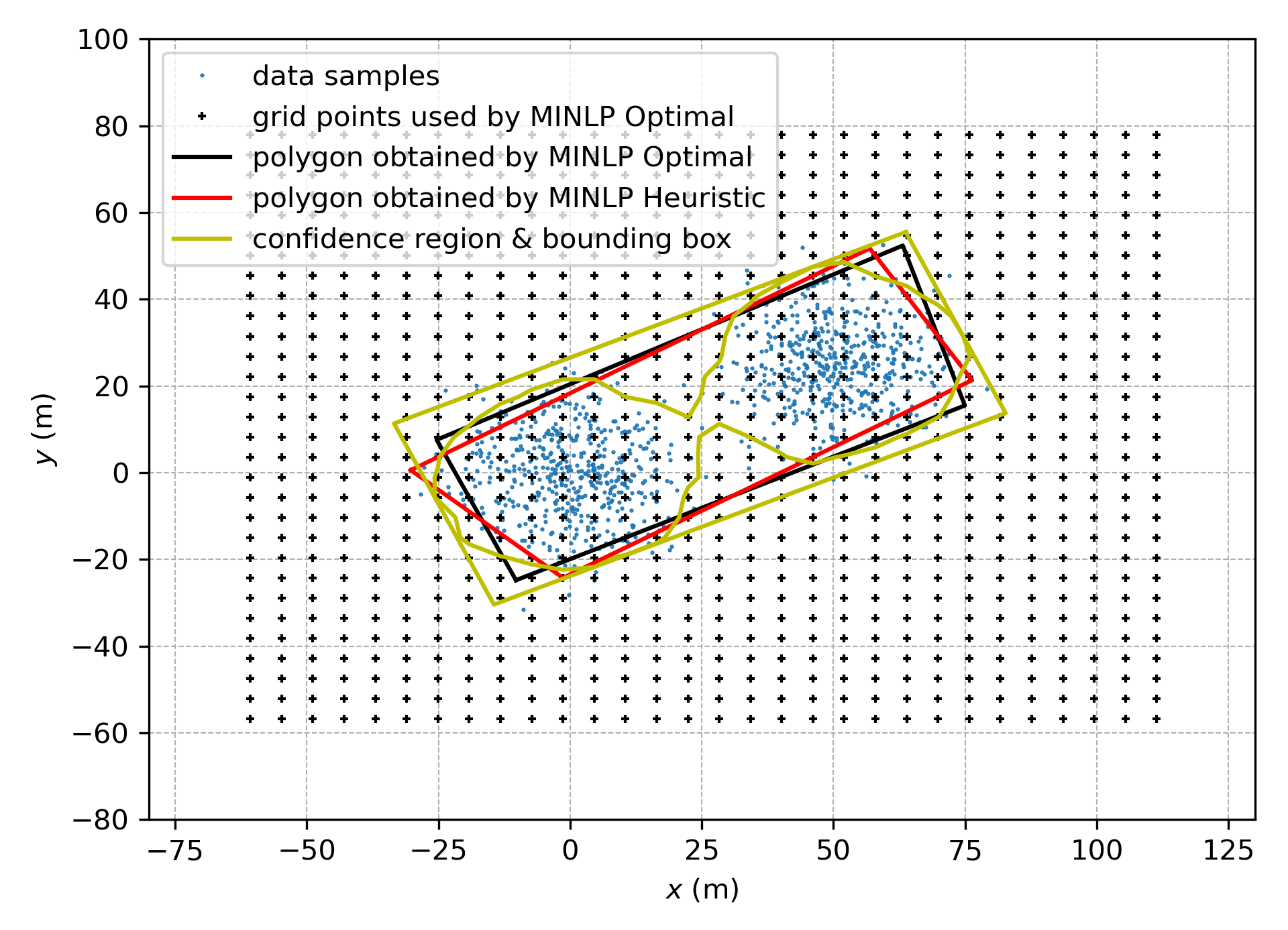}
    %\end{minipage}
    } \hspace{-0.2cm}%  
    \subfloat[Case I: $N^2 = 40^2$]{
    \label{sub-fig-grids40}
    %\begin{minipage}{ }
    %\centering
    \includegraphics[width=0.35\textwidth]{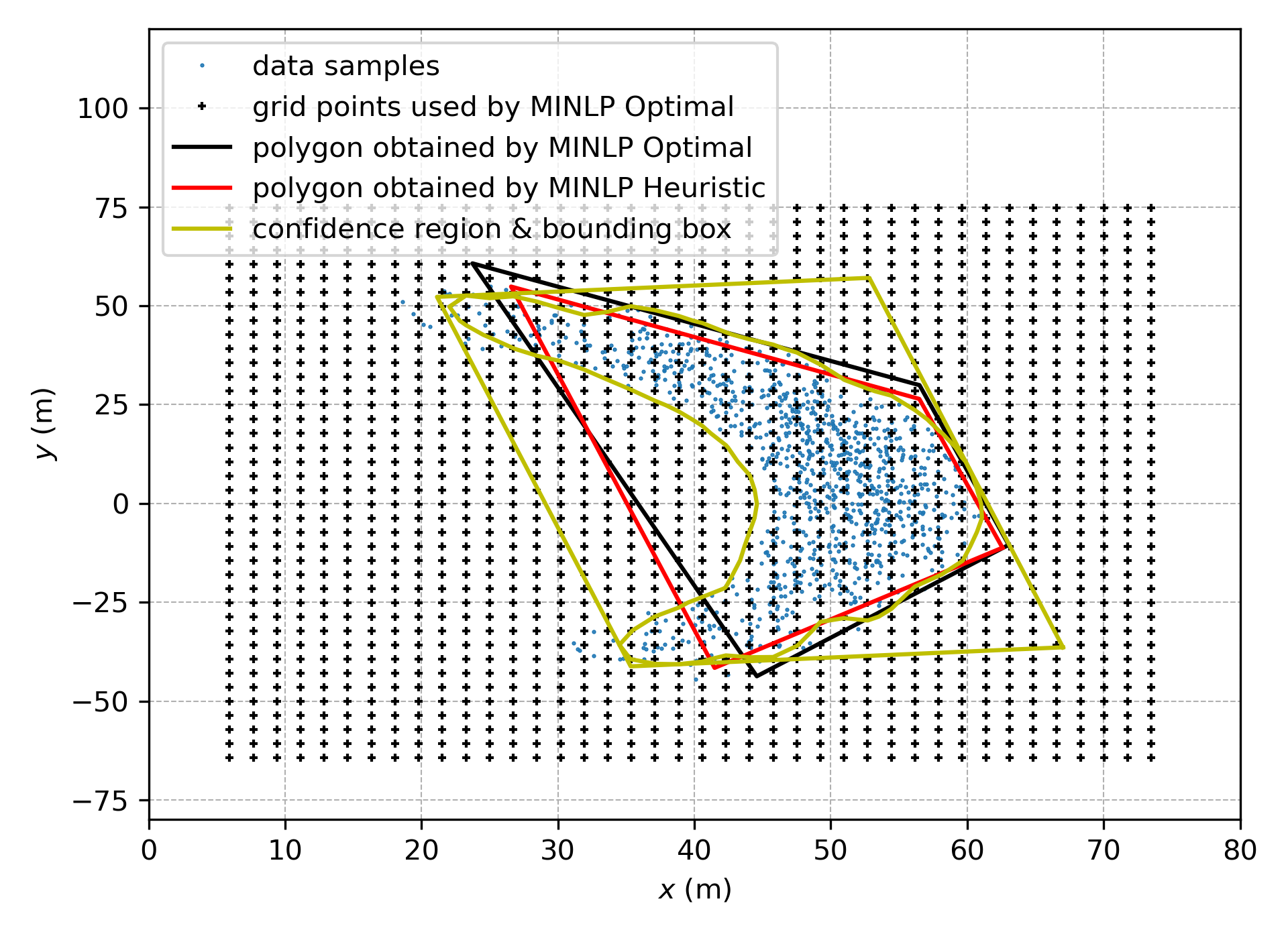}
    %\end{minipage}
    } \hspace{-0.2cm}% 
    \subfloat[Case II: \# $N^2 = 40^2$]{
    \label{sub-fig-grids40bi}
    %\begin{minipage}{ }
    %\centering
    \includegraphics[width=0.35\textwidth]{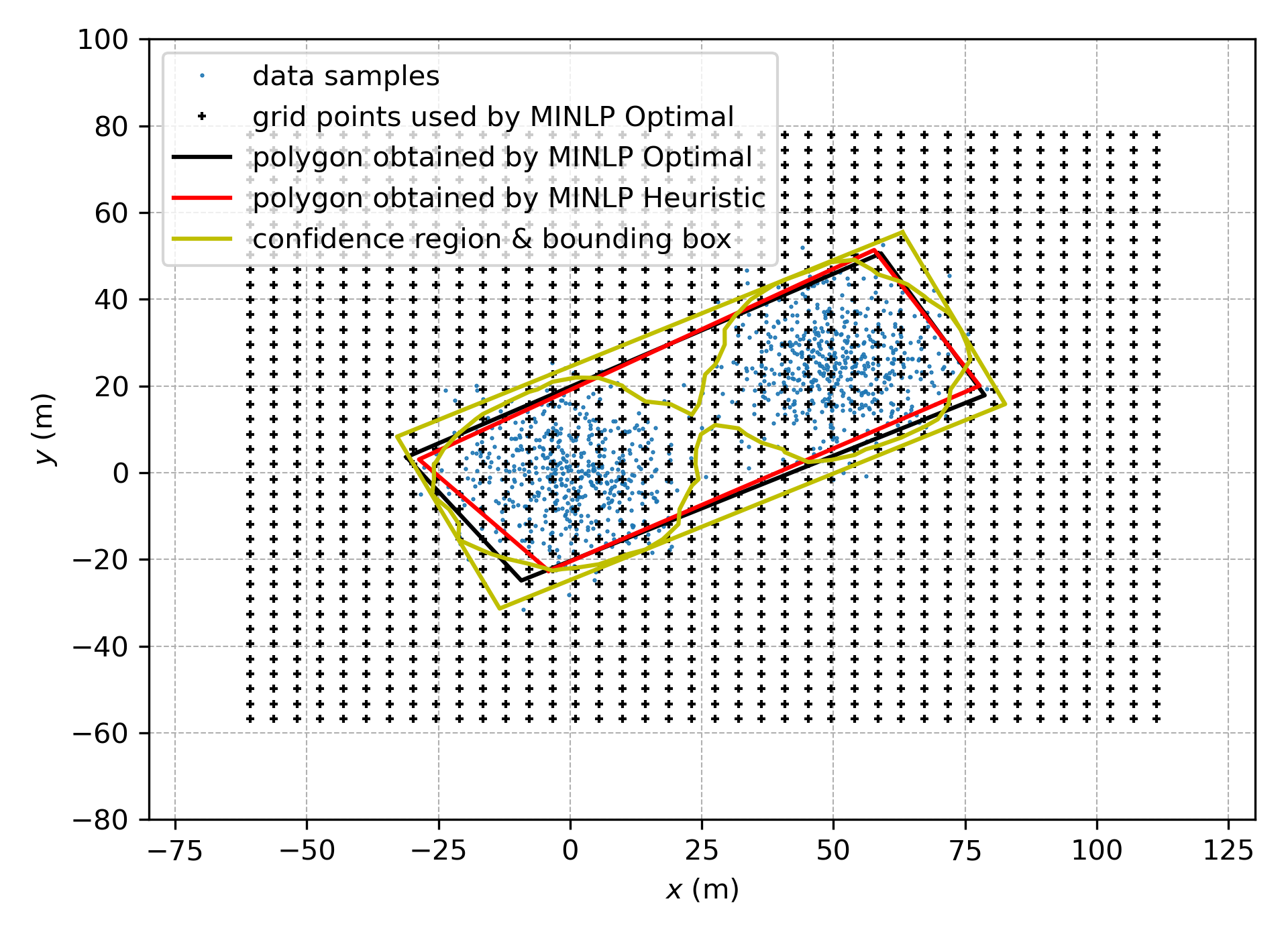}
    %\end{minipage}
    }
    \caption{Comparisons of different numbers of gird points $N^2$}
    \label{fig:fancpgrids}  %
\end{figure*} % end 

\begin{table*}[htbp]
\centering
\caption{Comparisons of different numbers of grid points}
\label{tab:fancpgrids}
\begin{tabular}{clcccccc}
\toprule
\multirow{2}{*}{\# \textbf{Grid Points $N^2$}}  & \multicolumn{1}{c}{ \multirow{2}{*}{\textbf{Algorithm}}  } & \multicolumn{3}{c}{\textbf{Case I}}  & \multicolumn{3}{c}{\textbf{Case II}}  \\ \cmidrule(lr){3-5} \cmidrule(lr){6-8}
{} & {} & \textbf{Ratio} & \textbf{Area (\unit{m^2})} & \textbf{Time (\unit{s})} & \textbf{Ratio} & \textbf{Area (\unit{m^2})} &\textbf{Time (\unit{s})} \\ \midrule
             & MINLP Optimal   & 87.2\%         & 1703.6              & 0.321         & 82.2\%         & 3784.7              & 0.082              \\
$10^2$          & MINLP Heuristic & 90.2\%         & 1908.7              & 0.064                    & 92.6\%         & 4110.5              & 0.015                  \\
             & Bounding Box    & 97.9\%         & 3257.2             & 0.209       & 99.6\%         & 6318.8              & 0.184                  \\ \midrule
             & MINLP Optimal   & 91.5\%         & 1872.0              & 70.5     & 91.4\%         & 3570.5              & 18.6                       \\
$20^2$ *           & MINLP Heuristic & 90.7\%         & 1807.9              & 0.801       & 90.6\%         & 3457.1              & 0.361                      \\
             & Bounding Box    & 97.9\%         & 3235.1              & 0.216     & 97.7\%         & 4987.2              & 0.189                  \\ \midrule
             & MINLP Optimal   & 91.4\%         & 1829.1              & 1150.2      & 92.0\%         & 3578.3              & 362.9               \\
$30^2$          & MINLP Heuristic & 90.4\%         & 1755.9              & 8.9     & 91.8\%         & 3443.1              & 6.9                 \\
             & Bounding Box    & 97.7\%         & 3157.7              & 0.225   & 97.6\%         & 4893.5              & 0.193                  \\ \midrule
             & MINLP Optimal   & 91.3\%         & 1827.0              & 2360.1    & 91.4\%         & 3554.8              & 1237.4                   \\
$40^2$          & MINLP Heuristic & 90.1\%         & 1721.1             & 125.5    & 90.9\%         & 3425.0              & 23.8                \\
             & Bounding Box    & 97.1\%         & 3030.2             & 0.239   & 97.3\%         & 4722.6              & 0.211                         \\ \bottomrule
\end{tabular}
\end{table*}

From \cref{fig:fancpgrids} and \cref{tab:fancpgrids}, we can draw some conclusions as follows. The results obtained by the implementation of all three algorithms given in \cref{sub-fig-grids10} and \cref{sub-fig-grids10bi} are visually very different from the other subfigures in \cref{fig:fancpgrids}. This is because when the number of grid points is too small ($ N^2 =  10^2$), KDE can not approximate PDF well. As indicated in \cref{tab:fancpgrids}, for every algorithm, as $N^2$ ranges from $20^2$ to $40^2$, the performance in ratio and area slightly improves. However, as $N^2$ increases, the increase in the number of decision variables and constraints leads to a significant increase in computational time. As shown in \cref{tab:fancpgrids}, the computational time of MINLP Heuristic goes from \SI{0.801}{s} to \SI{125.5}{s} for Case I and from \SI{0.361}{s} to \SI{23.8}{s} for Case II. Thus, there is a slight improvement in accuracy and optimality, but at huge cost of efficiency. To reach a compromise, the number of grid points can be chosen as $N^2 = 20^2$.

When $N^2$ is fixed to  $20^2$, the area and ratio of MINLP Optimal and MINLP Heuristic are much smaller than those of the Bounding Box algorithm, indicating MINLP Optimal and MINLP Heuristic significantly outperform Bounding Box in terms of optimality and accuracy. Further, for MINLP Optimal and MINLP Heuristic, their results of ratio and area are quite similar. However, the computational time of MINLP Heuristic greatly outperforms MINLP Optimal, i.e., \SI{0.801}{s} $<$ \SI{70.5}{s} for Case I and \SI{0.361}{s} $<$ \SI{18.6}{s} for Case II, demonstrating the efficiency of MINLP Heuristic. Therefore, MINLP Heuristic can guarantee near-optimality, accuracy, and efficiency simultaneously, which makes it very well suited for real-time implementation.

\subsection{Different Confidence Levels} 
In this part, we compare the performance of the algorithms relative to different confidence levels: $\alpha$ = $90\%$, $95\%$, and $99\%$. The other parameters are the same as in the baseline.

\begin{figure*}[htbp]
    \centering
    \subfloat[Case I: $\alpha = 90\%$ *]{
    \label{sub-fig-cl90}
    %\begin{minipage}{ }
    %\centering
    \includegraphics[width=0.35\textwidth]{figs/fan2020.png}
    %\end{minipage}
    } \hspace{-0.2cm}
    \subfloat[Case II: $\alpha = 90\%$ *]{
    \label{sub-fig-bicl90}
    %\begin{minipage}{ }
    %\centering
    \includegraphics[width=0.35\textwidth]{figs/fig2020.png}
    %\end{minipage}
    } \hspace{-0.2cm}
    \subfloat[Case I: $\alpha = 95\%$ ]{
    \label{sub-fig-cl95}
    %\begin{minipage}{ }
    %\centering
    \includegraphics[width=0.35\textwidth]{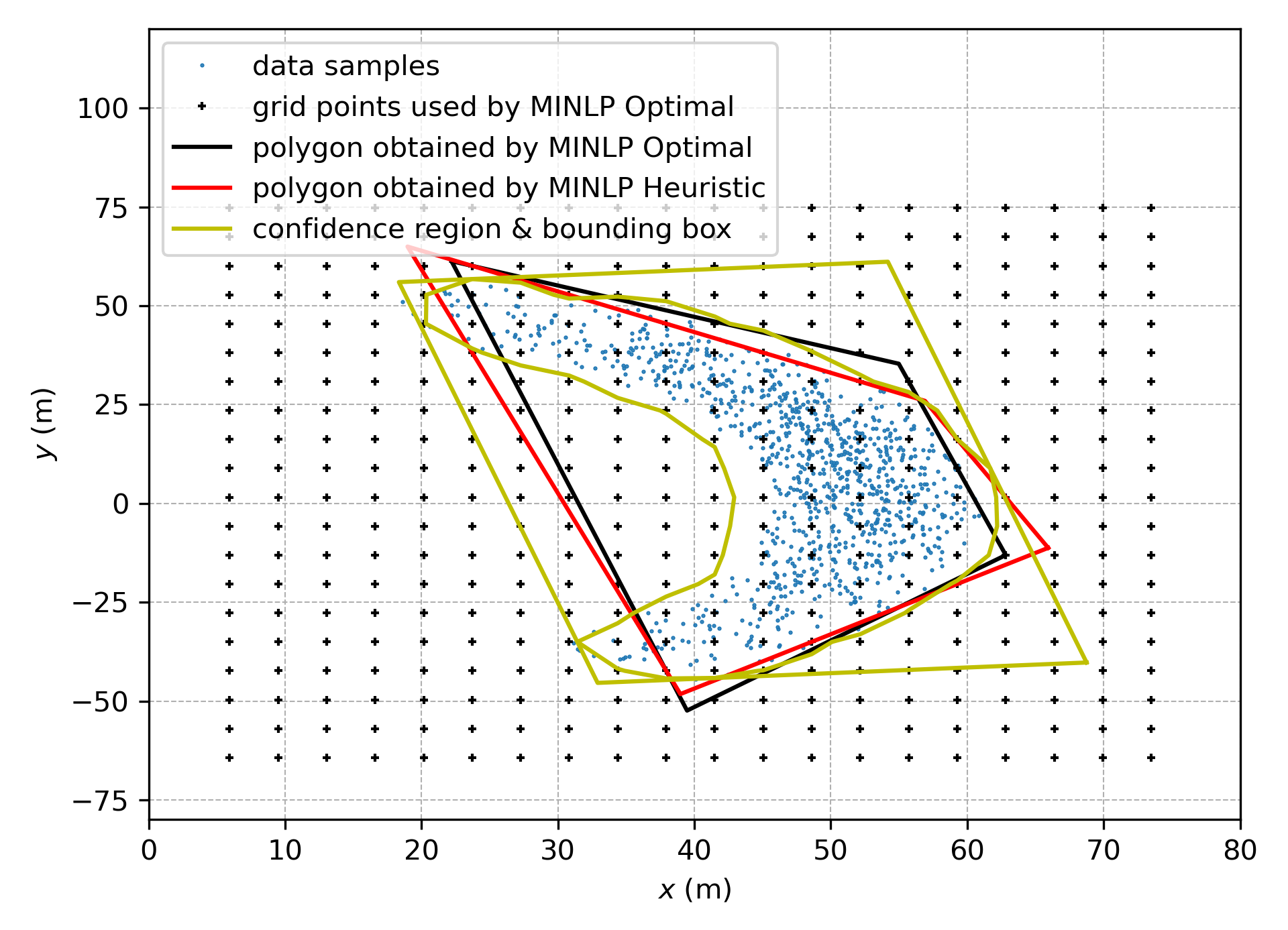}
    %\end{minipage}
    } \hspace{-0.2cm}
    \subfloat[Case II: $\alpha = 95\%$]{
    \label{sub-fig-bicl95}
    %\begin{minipage}{ }
    %\centering
    \includegraphics[width=0.35\textwidth]{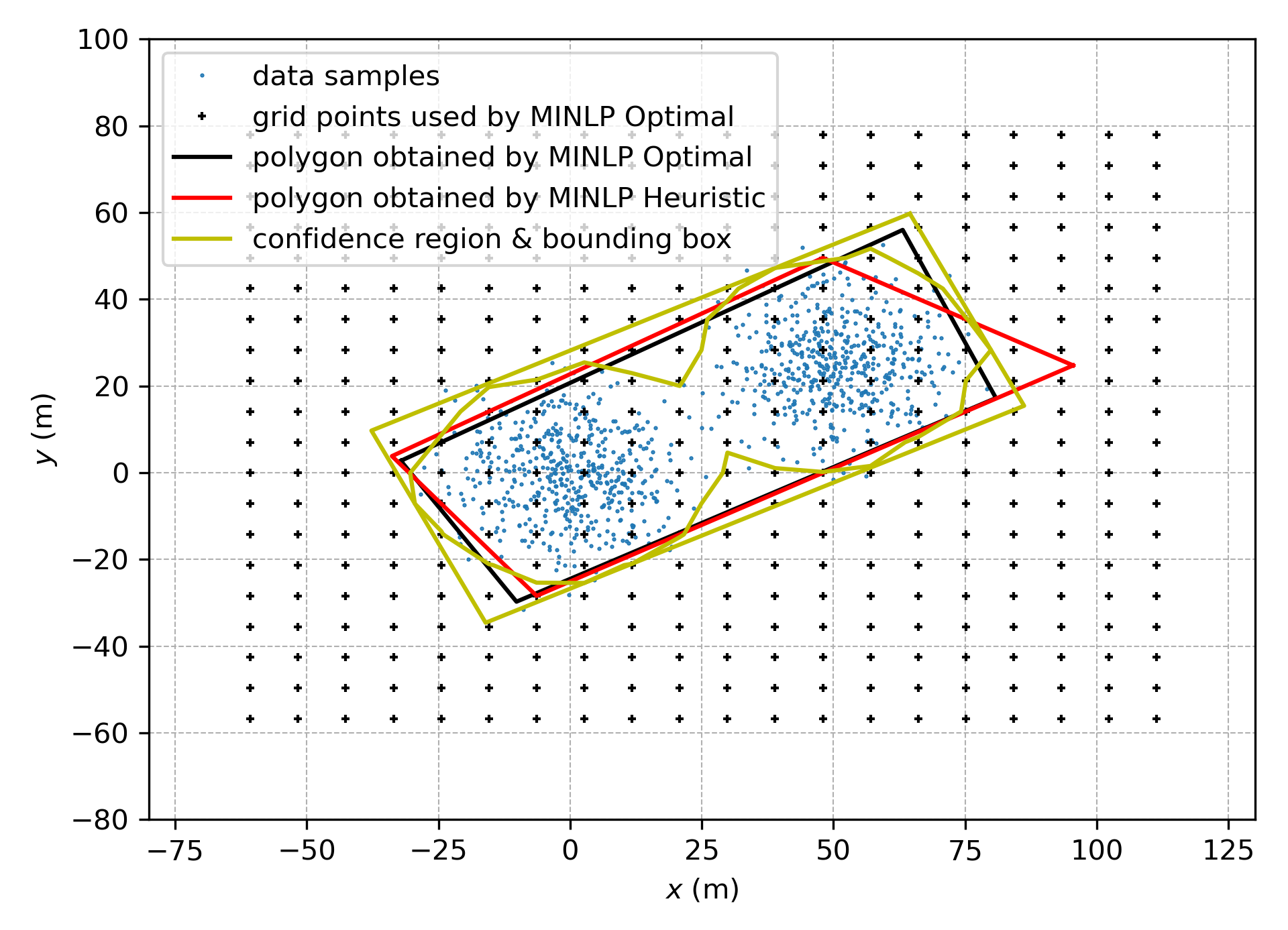}
    %\end{minipage}
    } \hspace{-0.2cm}
    \subfloat[Case I: $\alpha = 99\%$]{
    \label{sub-fig-cl100}
    %\begin{minipage}{ }
    %\centering
    \includegraphics[width=0.35\textwidth]{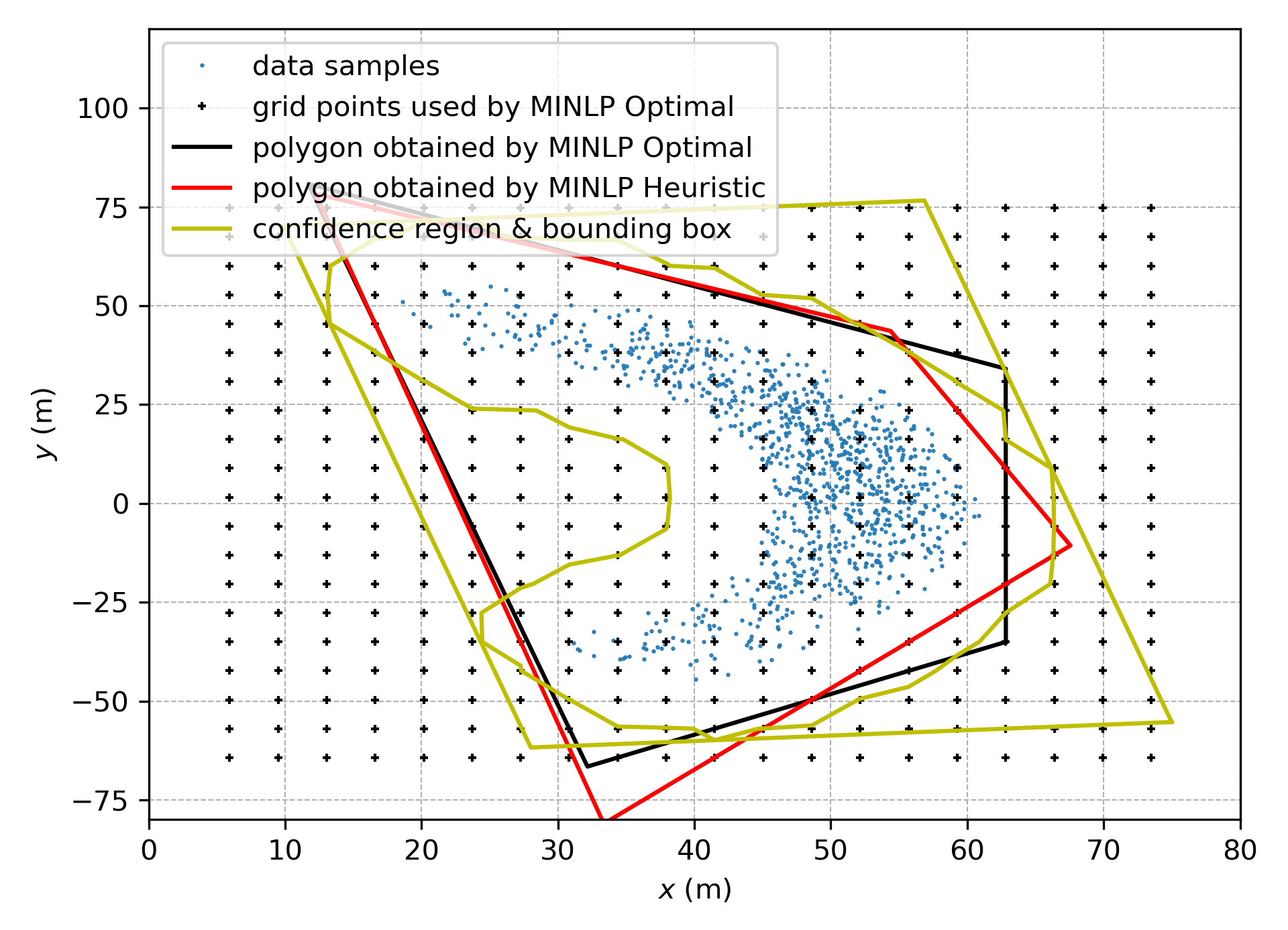}
    %\end{minipage}
    } \hspace{-0.2cm}
    \subfloat[Case II: $\alpha = 99\%$]{
    \label{sub-fig-bicl99}
    %\begin{minipage}{ }
    %\centering
    \includegraphics[width=0.35\textwidth]{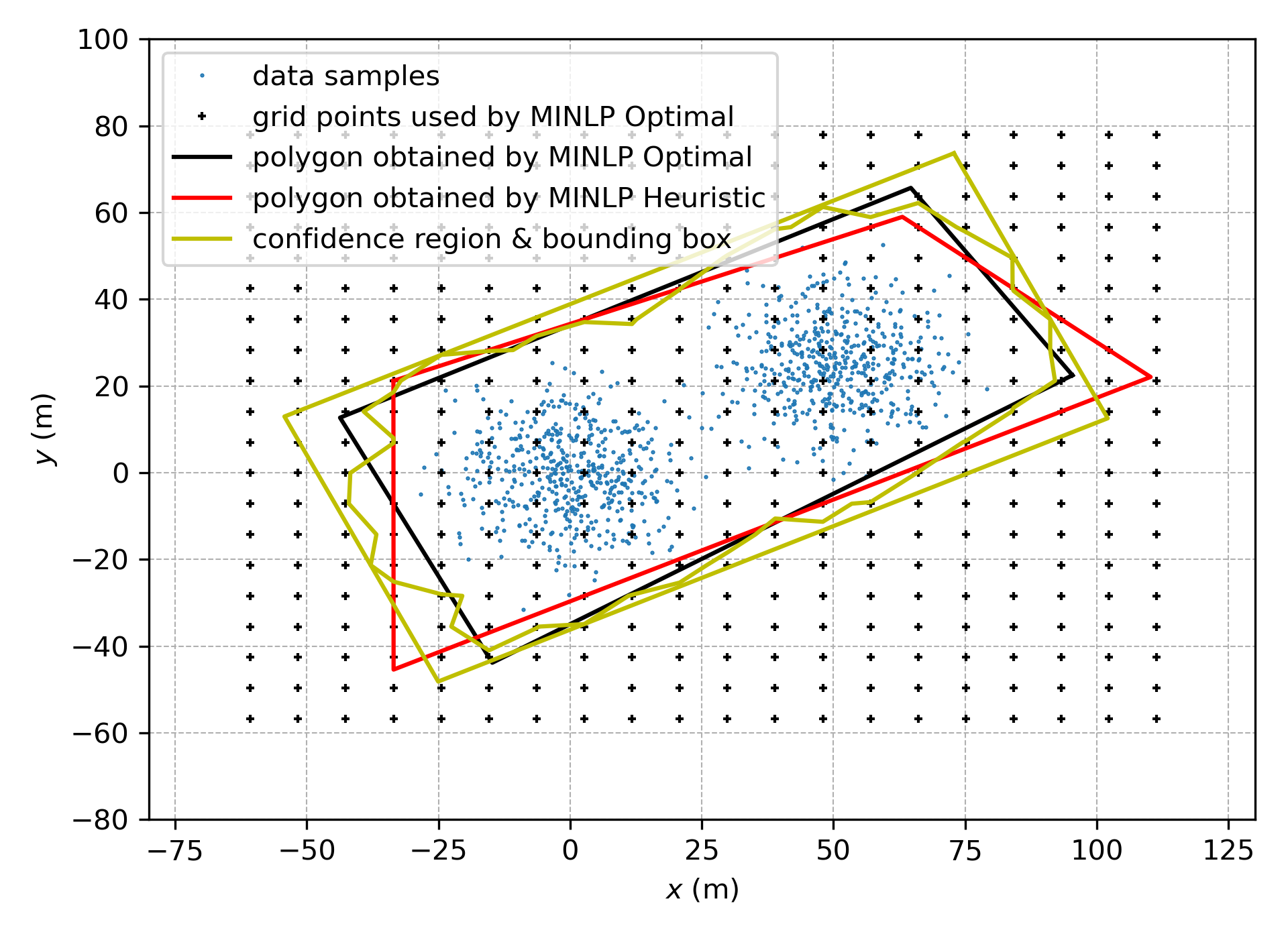}
    %\end{minipage}
    } 
    \caption{Comparisons of different confidence levels $\alpha$}
    \label{fig:fancpcl}  %
\end{figure*} % end

\begin{table*}[htbp]
\centering
\caption{Comparisons of different confidence levels}
\label{tab:fancpcl}
\begin{tabular}{clcccccc}
\toprule
\multirow{2}{*}{\textbf{Confidence Level $\alpha$}} & \multicolumn{1}{c}{ \multirow{2}{*}{\textbf{Algorithm}} } & \multicolumn{3}{c}{\textbf{Case I}} & \multicolumn{3}{c}{\textbf{Case II}} \\ \cmidrule(lr){3-5} \cmidrule(lr){6-8}
{} & {} & \textbf{Ratio} & \textbf{Area (\unit{m^2})} & \textbf{Time (\unit{s})} & \textbf{Ratio} & \textbf{Area (\unit{m^2})} &\textbf{Time (\unit{s})} \\ \midrule
             & MINLP Optimal   & 91.5\%         & 1872.0              & 70.5  & 91.4\% & 3570.5 & 18.6                   \\
90\% *           & MINLP Heuristic & 90.7\%         & 1807.9          & 0.801  & 90.6\%  & 3457.1 & 0.361                   \\
             & Bounding Box    & 97.9\%         & 3235.1              & 0.216 & 97.7\% & 4987.2 & 0.189                  \\      \midrule
             & MINLP Optimal   & 95.7\%         & 2359.6              & 53.5    &96.1\% &4312.6 &8.8                  \\
95\%            & MINLP Heuristic & 96.1\%         & 2424.7          & 0.417  &96.2\% & 4471.5 & 0.112                   \\
             & Bounding Box    & 99.4\%         & 3709.1              & 0.243     & 99.0\%  &5613.5 &0.235              \\      \midrule
             & MINLP Optimal   & 99.9\%         & 4350.6              & 50.6  & 99.9\%      & 7223.7 & 7.8               \\
99\%            & MINLP Heuristic & 99.9\%         & 4443.5         & 0.377    & 99.9\%  & 7460.1     & 0.090            \\
             & Bounding Box    & 100\%         & 6314.8              & 0.303    & 100\% & 9539.5 & 0.271               \\ \bottomrule
\end{tabular}
\end{table*}

In \cref{fig:fancpcl} and \cref{tab:fancpcl}, for both Case I and II, as $\alpha$ increases, the area of the polygon obtained by MINLP Optimal, MINLP Heuristic or Bounding Box respectively increases while the computational time of each algorithm slightly fluctuates. 

When $\alpha$ is fixed, the area of the polygon obtained by MINLP Optimal is far less than that of the Bounding Box. This means that the polygon obtained by MINLP Optimal is far less conservative than Bounding Box. However, the computational time of implementing MINLP Optimal is much longer than Bounding Box. That is, MINLP Optimal is inefficient to perform the online evaluation. Instead, the area of the polygon obtained by MINLP Heuristic is close to MINLP Optimal while the computational time of implementing MINLP Heuristic is short enough. Hence, MINLP Heuristic runs much faster than MINLP Optimal while ensuring accuracy.

\subsection{Robustness of the MINLP Heuristic Algorithm} 

As indicated in \cref{alg:heuristic}, every time we run MINLP Heuristic, we apply a weighted sampling to select representative grid points $\boldsymbol{g}'$ out of the grid points $\boldsymbol{g}$. Due to the randomness of weighted sampling, the selected grid points $\boldsymbol{g}'$ typically differ every time, and the polygon obtained by MINLP Heuristic changes every time it is applied. This leads to the following concerns regarding the robustness of MINLP Heuristic. 1) Space: this refers to whether the convex polygon obtained significantly differs every time; 2) Time: this refers to whether the computational time spent significantly differs every time. A major difference would mean that MINLP Heuristic is not robust enough for applications in real scenarios.

% be responsible for, affect, differ minorly

The similarity between two polygons can be quantified by means of the \textit{Jaccard distance}. %It compares members of two sets to see which members are shared and which are distinct. It's a measure of similarity between the two sets of data, with a range from 0 to 1. 
Formally, for any two sets $A$ and $B$, Jaccard distance is defined as $d(A,B) = 1 - {{\operatorname{Card}(A \cap B)}\over{\operatorname{Card}(A \cup B)}}$ \cite{wu2022online}, 
% \begin{equation*}
% d(A,B) = 1 - {{\operatorname{Card}(A \cap B)}\over{\operatorname{Card}(A \cup B)}},
% 	\label{jaccard}
% \end{equation*}
where the operation $\operatorname{Card}$ finds the cardinality of a set. A smaller Jaccard distance means the two sets are more similar to each other.

% \footnote{\magenta{provide a reference}}

We analyze the robustness of MINLP Heuristic with respect to the varying number $N_s$ of grid points $\boldsymbol{g}'$ used by MINLP Heuristic for Cases I and II mentioned above. For either Case I or Case II, all the parameters, except for the number $N_s$, are constant and the same as in the respective baselines.

\begin{figure}[!h]
    \centering
    \subfloat[Space robustness]{
    \label{sub-fig-sr}
    %\begin{minipage}{ }
    %\centering
    \includegraphics[width=0.35\textwidth]{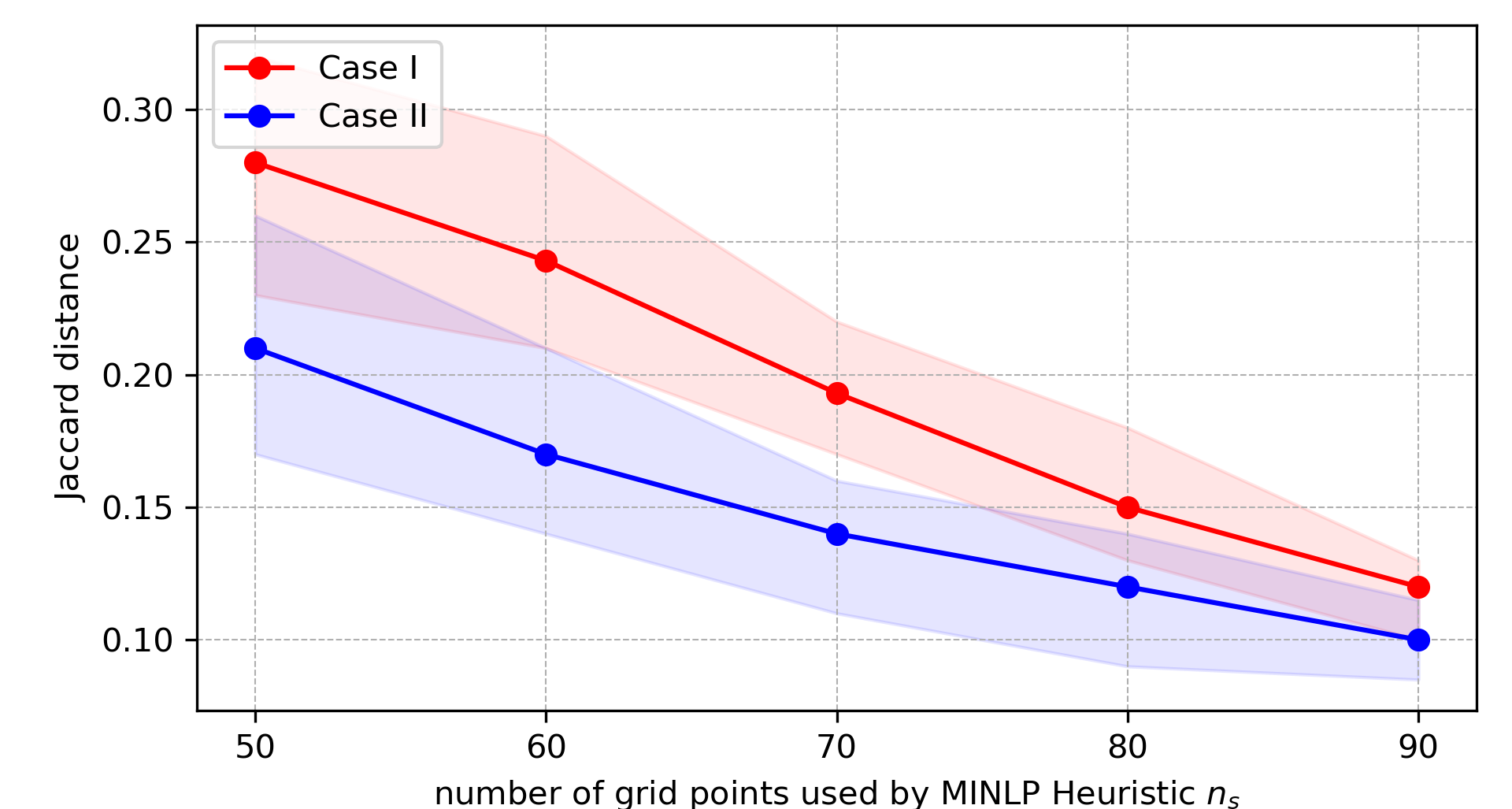}
    %\end{minipage}
    } \hspace{0.5cm}
    \subfloat[Time robustness]{
    \label{sub-fig-tr}
    %\begin{minipage}{ }
    %\centering
    \includegraphics[width=0.35\textwidth]{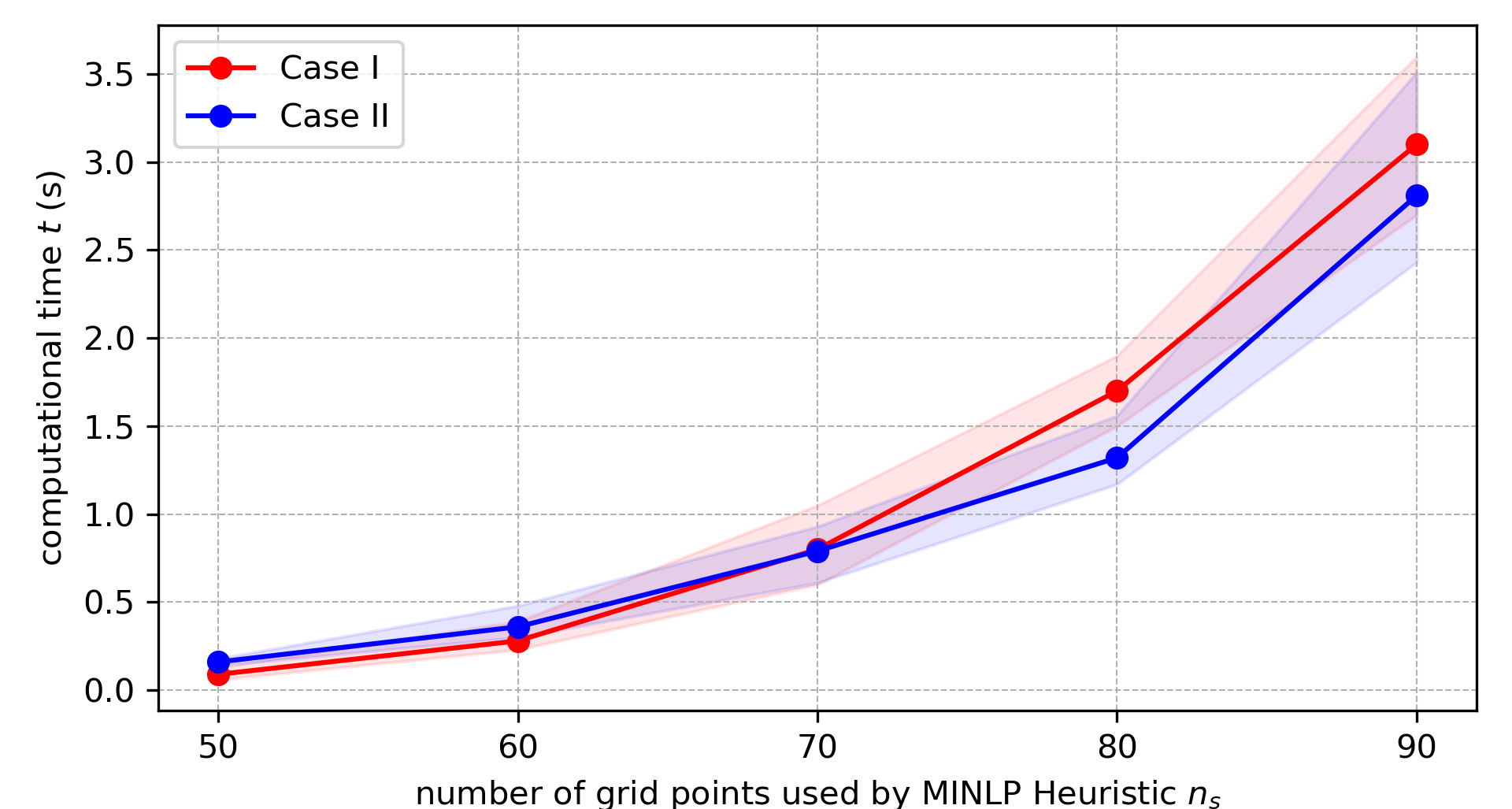}
    %\end{minipage}
    } 
    \caption{ Robustness analysis of MINLP Heuristic algorithm}
    \label{fig:robustheu}  %
\end{figure} % end 

For either Case I or Case II, given a fixed number $N_s$, we run MINLP Optimal once to obtain a constant optimal polygon and run MINLP Heuristic ten times to obtain ten approximate polygons which may differ every time. We evaluate the Jaccard distance between every approximate polygon and the optimal polygon and obtain ten distances eventually. In addition, we evaluate the computational time of running MINLP Heuristic every time. We repeat the operation above for $N_s = 50, 60, 70, 80, 90$,  respectively. The results are depicted in \cref{fig:robustheu}. As illustrated in \cref{sub-fig-sr}, for both cases, when $N_s$ is too small, like $N_s = 50$, different Jaccard distances fluctuate significantly around the average value, suggesting that MINLP Heuristic has poor robustness under this circumstance. As $N_s$ increases from $50$ to $90$, the average Jaccard distance decreases from \num{0.28}, \num{0.21} to \num{0.12}, \num{0.10},  respectively. This means the approximate polygon obtained by MINLP Heuristic becomes more near-optimal with respect to an increasing $N_s$. Moreover, the variance of ten Jaccard distances also decreases as $N_s$ grows. That is, MINLP Heuristic becomes more robust with respect to the increase of $N_s$. As suggested in \cref{sub-fig-tr}, as $N_s$ increases, the average computational time increases from \SI{0.09}{s} to \SI{3.1}{s} for Case I, and from \SI{0.16}{s} to \SI{2.81}{s} for Case II. Additionally, as $N_s$ grows, the variance of computational time increases. %Therefore, an increase of grids used by MINLP Heuristic contributes to the burden of computational effort.

In summary, the robustness of MINLP Heuristic comes at the cost of losing computational efficiency. For the case studies under consideration, a balance is achieved for %Therefore, to achieve a balance, we want to choose an appropriate parameter of the number of grids used by MINLP Heuristic. For example, it can be 
$N_s = 70$ for Case I or $N_s = 60$ for Case II. This makes the MINLP Heuristic robust enough and its applications promising. 
% TUNE parameter 

% end of section results

\section{Conclusion}
\label{Conclusion}
\medskip

This paper proposes an efficient convex approximation of the PRS of an uncertain dynamic system. For arbitrary unknown uncertainties, especially unbounded ones, the notion of the PRS is introduced as an extension of the traditional reachable set, connecting chance constraint formulation and reachability analysis. In this paper, the PRS is obtained through a data-driven approach employing a KDE method. FFT is then customized to accelerate the computation of KDE to make it computationally attractive. The irregularity or non-convexity of the PRS refrains its use in optimal design applications. To address this issue, we formulate a MINLP problem whose solution accounts for a convex approximation of the PRS. We then develop a MINLP Heuristic algorithm to solve it. Comprehensive case studies demonstrate that our proposed algorithm enjoys the benefits of accuracy, efficiency, near-optimality, as well as robustness while providing a convex approximation for the PRS.

This work is limited to a 2D system. Future work will be devoted to solving higher dimensional problems, and applying the results to the design of safety-critical real-time motion planning algorithms for uncertain robotic systems. 

% The advantages of this work pave the way for its promising applications to safety-critical real-time motion planning of uncertain systems.

% section conclusion ends

% or


% \appendices

% \section{}

\section{Appendix}
\label{Appendix}
\medskip

\subsection{Proof of \cref{lemma:4}}
\label{subsec:lemma4}
\textbf{I)} First, we prove that any two formally different representations of lines refer to different lines.

Given $n \ge 3$ planar lines not passing through the origin $l^k := a_kx + b_ky - 1 = 0,  k \in \{ 1,\ldots,n \}$, we can construct a sequence of lines $(l^1, l^2, \ldots, l^n, l^1)$. Two formally different lines $l^i, l^j$ in the sequence are defined to be consecutive if and only if $j= i^{\oplus} \lor j= i^{\ominus}$. Since there are $n$ intersection points $V_{ij} := l^i \cap l^j,  i \in \{ 1,\ldots,n \}, j = i^{\oplus}$, it follows that $l^i \ne l^j, \;  i \in \{ 1,\ldots,n \}, j = i^{\oplus}$. Thus, any two formally different lines consecutive in the sequence are indeed different from each other. 

For $n=3$, the sequence of lines is $(l^1, l^2, l^3, l^1)$. Since any two formally different lines in the sequence $(l^1, l^2, l^3, l^1)$ are consecutive, thus $l^1, l^2, l^3$ are indeed different from each other. For $n \ge 4$, there are at least one pair of formally different lines not consecutive in the sequence. We have already shown that any two formally different lines consecutive in the sequence are indeed different from each other, and therefore it remains to show that any two formally different lines not consecutive in the sequence are also indeed different from each other, given the conditions of this lemma. We prove this by contradiction. Suppose that $l^u, l^v, u \in \{ 1,\ldots,n \}, v \in \{ 1,\ldots,n \} - \{u^{\ominus}, u, u^{\oplus}\}$ are two formally different lines not consecutive in the sequence but are indeed an identical line. Then $V_{v^{\ominus}v}$ and $V_{vv^{\oplus}}$ are both on the line $l^u$, which means
\begin{equation}
    \begin{aligned} 
    & -a_u\left(b_{v^{\ominus}}-b_v \right)+b_u\left(a_{v^{\ominus}}-a_v\right) - (a_{v^{\ominus}}b_v-b_{v^{\ominus}}a_v) = 0, \\
    & -a_u\left(b_v-b_{v^{\oplus}} \right)+b_u\left(a_v-a_{v^{\oplus}}\right) - (a_vb_{v^{\oplus}}-b_va_{v^{\oplus}}) = 0. 
    % &  i \in \{ 1,\ldots,n \}, j = i^{\oplus},  k \in \{ 1,\ldots,n \}-\{i, j\}. \notag
    \label{eqn:diffline}
    \end{aligned}
\end{equation}
However, since $v \notin \{u^{\ominus}, u, u^{\oplus}\}$, it follows that the subscript of $V_{v^{\ominus}v}$ is formally different from $V_{u^{\ominus}u}$ and $V_{uu^{\oplus}}$, and the subscript of $V_{vv^{\oplus}}$ is also formally different from $V_{u^{\ominus}u}$ and $V_{uu^{\oplus}}$. Thus, since the constraints \cref{eqn:no1cons} are satisfied, we obtain two instances
\begin{equation*}
    \begin{aligned}
    & -a_u\left(b_{v^{\ominus}}-b_v \right)+b_u\left(a_{v^{\ominus}}-a_v\right) - (a_{v^{\ominus}}b_v-b_{v^{\ominus}}a_v) < 0, \\
    & -a_u\left(b_v-b_{v^{\oplus}} \right)+b_u\left(a_v-a_{v^{\oplus}}\right) - (a_vb_{v^{\oplus}}-b_va_{v^{\oplus}}) < 0, 
    % &  i \in \{ 1,\ldots,n \}, j = i^{\oplus},  k \in \{ 1,\ldots,n \}-\{i, j\} \notag
\end{aligned}
\end{equation*}
which contradicts \cref{eqn:diffline}. %Therefore, any two formally different lines not consecutive in the sequence are indeed different from each other, given the conditions of this lemma.

Thus, if constraints \cref{eqn:detcon} and \cref{eqn:no1cons} are satisfied, any two formally different lines are indeed different from each other. 

\textbf{II)} Next, we show that if there are $n$ intersection points $V_{ij} := l^i \cap l^j,  i \in \{ 1,\ldots,n \}, j = i^{\oplus}$ (namely, $D_{ij}:=a_ib_j-b_ia_j > 0, \;  i \in \{ 1,\ldots,n \}, j = i^{\oplus}$ according to \cref{eqn:detcon}), and the constraints \cref{eqn:no1cons} are satisfied, then any two formally different intersection points are indeed different from each other.

We can prove this by contradiction. Assume that there are two formally different intersection points $V_{ij}, V_{uv}, i \in \{ 1,\ldots,n \}, j = i^{\oplus}, u \in \{ 1,\ldots,n \} - \{ i \}, v = u^{\oplus}$ which are indeed an identical intersection point. Since $n \ge 3$, it follows that $(u \ne i \land u \ne j) \lor (v \ne i \land v \ne j)$, and without loss of generality, we assume $v \ne i \land v \ne j$. Thus, the constraints \cref{eqn:no1cons} apply to $V_{ij}$ and $l^v$, which yields an instance
\begin{equation}
    -a_v\left(b_i-b_j \right)+b_v\left(a_i-a_j\right) - (a_ib_j-b_ia_j) < 0.
    \label{eqn:-av}
\end{equation}
However, since $V_{ij}, V_{uv}$ are indeed an identical intersection point, $V_{ij}$ is on the line $l^v$, meaning that
\begin{equation*}
    -a_v\left(b_i-b_j \right)+b_v\left(a_i-a_j\right) - (a_ib_j-b_ia_j) = 0,
\end{equation*}
which contradicts \cref{eqn:-av}. %Thus, any two formally different intersection points are indeed different from each other.

Combining I) and II), we conclude that if constraints \cref{eqn:detcon} and \cref{eqn:no1cons} are satisfied, any two formally different representations of intersection points (resp.~lines) refer to different intersection points (resp.~lines), and therefore all $n$ intersection points (resp.~lines) are indeed different from each other.

\subsection{Proof of \cref{theorem:1}}
\label{subsec:theorem1}
\textbf{I)} Since there are $n \ge 3$ intersection points $V_{ij} := l^i \cap l^j, \;   i \in \{ 1,\ldots,n \}, j = i^{\oplus}$, we can enumerate them in a sequence of intersection points $S_g := (V_{12}, V_{23}, \ldots, V_{(n-1)n}, V_{n1}, V_{12})$. According to \cref{lemma:4}, since the coefficients of these lines satisfy constraints \cref{eqn:detcon} and \cref{eqn:no1cons}, then any two formally different intersection points (resp.~lines) are indeed different from each other.

\textbf{II)} We can also construct a finite sequence of symbols $S_s := (``V_{12}", ``V_{23}", \ldots, ``V_{(n-1)n}", ``V_{n1}", ``V_{12}")$. In $S_s$, 1) since $n \ge 3$, there are at least three formally different symbols $``V_{12}"$, $``V_{23}"$ and $``V_{31}"$; 2) No symbols appear more than once except for the case in which only the first symbol $``V_{12}"$ and last symbol $``V_{12}"$ are formally identical. Therefore, \cref{def:Formallyenclosing} applies and $S_s$ is formally enclosing.

Consider a map ${f}_{\text{tag}} $ from $S_s$ to $S_g$ given by the operation of labelling the intersection points, i.e., $V_{ij} := l^i \cap l^j, \;   i \in \{ 1,\ldots,n \}, j = i^{\oplus}$. It can be verified that $f_{\text{tag}}$ is injective, surjective and order-preserving. The map $f_{\text{tag}}$ also satisfies Condition 4) in \cref{def:isomorphism}. This is because in $S_s$, for any two consecutive symbols $``V_{i^{\ominus}i}"$ and $``V_{ii^{\oplus}}"$, and any third symbol $``V_{uv}"$ which is formally different from each of $``V_{i^{\ominus}i}"$ and $``V_{ii^{\oplus}}"$, according to \cref{eqn:no1cons} we obtain
\begin{equation*}
    \begin{aligned}
    & -a_i\left(b_u-b_v \right)+b_i\left(a_u-a_v\right) - (a_ub_v-b_ua_v) \ne  0, \\
    &  u \notin \{i^{\ominus}, i\}, v= u^{\oplus},
    % \label{eqn:neline}
    \end{aligned}
\end{equation*}
which means that the intersection point $V_{uv}$ is not on the line $l^i$. According to I), since formally different symbols $``V_{i^{\ominus}i}"$ and $``V_{ii^{\oplus}}"$ refer to different intersection points $V_{i^{\ominus}i}$ and $V_{ii^{\oplus}}$, thus these two intersection points determine a non-degenerate line segment. Further, the line segment determined by these two points is a segment of the line $l^i$. Hence, the intersection point $V_{uv}$ that the third symbol $``V_{uv}"$ refers to is not on the edge $(V_{i^{\ominus}i}, V_{ii^{\oplus}})$. In summary, the map $f_{\text{tag}}$ satisfies all the conditions in \cref{def:isomorphism} and therefore is an isomorphism type A.

According to \cref{remark_enclose}, since $S_s$ is formally enclosing, and $f_{\text{tag}}$ is an isomorphism type A, it follows that the graph $G$ formed by $S_g$ is enclosing.

\textbf{III)} Since $n \ge 3$, the number of formally different symbols in  $S_s$ is $n \ge 3$. 

As concluded in II), the map $f_{\text{tag}}$ is injective, surjective, and order-preserving. The map $f_{\text{tag}}$ also satisfies Condition 4) in \cref{def:isomorphism_type_b}. This is because, for any three consecutive symbols $``V_{i^{\ominus}i}"$, $``V_{ii^{\oplus}}"$, $``V_{i^{\oplus}i^{\oplus \oplus}}"$ in $S_s$, applying \cref{eqn:no1cons} yields
\begin{equation*}
    -a_i\left(b_{i^{\oplus}}-b_{i^{\oplus \oplus}} \right)+b_i\left(a_{i^{\oplus}}-a_{i^{\oplus \oplus}} \right) - (a_{i^{\oplus}}b_{i^{\oplus \oplus}} - b_{i^{\oplus}}a_{i^{\oplus \oplus}}) \ne  0,
\end{equation*}
which means that the intersection point $V_{i^{\oplus}i^{\oplus \oplus}}$ is not on the line $l^i$. From I), since formally different symbols $``V_{i^{\ominus}i}"$ and $``V_{ii^{\oplus}}"$ refer to different intersection points $V_{i^{\ominus}i}$ and $V_{ii^{\oplus}}$, then $V_{i^{\ominus}i}$ and $V_{ii^{\oplus}}$ are different from each other. Further, the intersection point $V_{i^{\ominus}i}$ and $V_{ii^{\oplus}}$ are on the line $l^i$. Hence, the three intersection points $V_{i^{\ominus}i}$, $V_{ii^{\oplus}}$, $V_{i^{\oplus}i^{\oplus \oplus}}$ are not collinear. In summary, the map $f_{\text{tag}}$ satisfies all the conditions in \cref{def:isomorphism_type_b} and therefore is an isomorphism type B.

According to \cref{remark_degenerate}, since $f_{\text{tag}}$ is an isomorphism type B, and there are $n \ge 3$ formally different symbols in $S_s$, it follows that the graph $G$ formed by $S_g$ is non-degenerate.

\textbf{IV)} We have shown that the graph $G$ formed by $S_g$ is an enclosing $n$ sided polygon. For a line $l^k, k \in \{ 1,\ldots,n \}$, once the constraints \cref{eqn:no1cons} are satisfied, it follows that
\begin{equation}
    \begin{cases}
    -a_k\left(b_i-b_j \right)+b_k\left(a_i-a_j\right) - (a_ib_j-b_ia_j) < 0, \\
     i \in \{ 1,\ldots,n \} - \{k^{\ominus},k\}, j=i^{\oplus},  \\ \\
    -a_k\left(b_i-b_j \right)+b_k\left(a_i-a_j\right) - (a_ib_j-b_ia_j) = 0, \\
     i \in \{k^{\ominus},k\}, j=i^{\oplus}.
    \end{cases}
    \label{eqn:bianxing}
\end{equation}
Since there are $n$ intersection points $V_{ij}, \;   i \in \{ 1,\ldots,n \}, j = i^{\oplus}$, according to \cref{eqn:detcon}, we have $D_{ij}:=a_ib_j-b_ia_j > 0, \;  i \in \{ 1,\ldots,n \}, j = i^{\oplus}$. Thus, the inequalities \cref{eqn:bianxing} can be transformed to 
\begin{equation}
    \begin{cases}
    a_k\frac{-\left(b_i-b_j \right)}{D_{ij}}+b_k\frac{\left(a_i-a_j\right)}{D_{ij}} - 1 < 0, \\
     i \in \{ 1,\ldots,n \} - \{k^{\ominus},k\}, \,j=i^{\oplus},  \\ \\
    a_k\frac{-\left(b_i-b_j \right)}{D_{ij}}+b_k\frac{\left(a_i-a_j\right)}{D_{ij}} - 1 = 0, \\ 
     i \in \{k^{\ominus},k\}, \, j=i^{\oplus},
    \end{cases}
\label{eqn:conv0}
\end{equation}
% \begin{equation}
%     \begin{cases}
%     a_k\frac{-\left(b_i-b_j \right)}{D_{ij}}+b_k\frac{\left(a_i-a_j\right)}{D_{ij}} - 1 < 0 &  i \in \{ 1,\ldots,n \} - \{k^{\ominus},k\}, \\
%     &  j=i^{\oplus}  \\
%     a_k\frac{-\left(b_i-b_j \right)}{D_{ij}}+b_k\frac{\left(a_i-a_j\right)}{D_{ij}} - 1 = 0 &  i \in \{k^{\ominus},k\}, j=i^{\oplus}
%     \end{cases}
% \label{eqn:conv0}
% \end{equation}
which means that all $n$ intersection points $V_{ij},  i \in \{ 1,\ldots,n \}, j = i^{\oplus}$ are on the same side of the line $l^k$. 
Therefore, in $S_g$, for any two consecutive intersection points $V_{ij},V_{uv}$ where $i \in \{ 1,\ldots,n \} - \{k^{\ominus}\}, j = i^{\oplus}, u=i^{\oplus}, v=i^{\oplus \oplus}$, we can define a set constructed by the convex combination of $V_{ij}$ and $V_{uv}$ excluding these two boundary points, which is
\begin{equation*}
    \begin{aligned}
    \Big \{ & t(-\frac{b_i-b_j}{D_{ij}}, \frac{a_i-a_j}{D_{ij}}) + (1-t)(-\frac{b_u-b_v}{D_{uv}}, \frac{a_u-a_v}{D_{uv}}) : \\
    & t \in \mathbb{R} \land 0 < t < 1 \Big \}. 
    \end{aligned}
\end{equation*}
It is a subset of the half plane $\{(x,y): a_kx + b_ky - 1 < 0 \}$. In other words, the edge $(V_{ii^{\oplus}},V_{i^{\oplus}i^{\oplus \oplus}}), i \in \{ 1,\ldots,n \} - \{k^{\ominus}\}$, is on the strict inner side of the line $l^k$. This is because 
\begin{equation}
    \begin{aligned}
    & \, a_k \left(t \frac{-\left(b_i-b_j \right)}{D_{ij}} + (1-t) \frac{-\left(b_u-b_v \right)}{D_{uv}} \right) +  \\
    & \, b_k \left(t \frac{\left(a_i-a_j \right)}{D_{ij}} + (1-t) \frac{\left(a_u-a_v\right)}{D_{uv}} \right) - 1  \\ 
    % = \, & a_k t \frac{-\left(b_i-b_j \right)}{D_{ij}} + a_k (1-t) \frac{-\left(b_u-b_v \right)}{D_{uv}} \notag \\
    % & + b_k t \frac{\left(a_i-a_j \right)}{D_{ij}} + b_k (1-t) \frac{\left(a_u-a_v\right)}{D_{uv}} - 1 \notag \\ \notag \\
    = \, & \, t \left(a_k \frac{-\left(b_i-b_j \right)}{D_{ij}} + b_k \frac{\left(a_i-a_j \right)}{D_{ij}} \right) +  \\
    & \, (1-t) \left(a_k \frac{-\left(b_u-b_v \right)}{D_{uv}} + b_k \frac{\left(a_u-a_v\right)}{D_{uv}} \right) - 1  \\ 
    < \, & \, 0,
    \label{eqn:convex}
    \end{aligned}
\end{equation}
according to \cref{eqn:conv0}.
For the polygon $G$ formed by $S_g$, according to \cref{eqn:convex}, we have shown that every edge $(V_{ii^{\oplus}},V_{i^{\oplus}i^{\oplus \oplus}}), i \in \{ 1,\ldots,n \} - \{k^{\ominus}\}$ of the polygon is on the same side of the line $l^k, k \in \{ 1,\ldots,n \}$ that the edge $(V_{k^{\ominus}k},V_{kk^{\oplus }}), k \in \{ 1,\ldots,n \}$ defines. Recall the definition of a convex polygon: For each edge of the polygon, all the other edges are on the same side of the line that the edge defines. Therefore, the polygon $G$ formed by $S_g$ is a convex polygon. 

\textbf{V)} Recall that for an arbitrary point $(u, v)$ and an enclosing $n$ sided convex polygon determined by $n$ lines $l^k := a_kx + b_ky - 1 = 0,  k \in \{1,\ldots,n\}$, if $\bigwedge\limits_{k=1}^n a_ku + b_kv - 1 < 0 $ holds, then the point $(u, v)$ lies inside the polygon. We have shown above that if constraints \cref{eqn:detcon} and \cref{eqn:no1cons} are satisfied, then the graph $G$ formed by $S_g$ is an enclosing $n$ sided convex polygon. Since $a_k\cdot 0 + b_k\cdot 0 - 1  \equiv -1 < 0,  k \in \{ 1,\ldots,n \}$, the origin $(0,0)$ lies inside the polygon $G$ formed by $S_g$. 

From the above, we conclude that if constraints \cref{eqn:detcon} and \cref{eqn:no1cons} are satisfied, then the graph $G$ formed by $S_g$ is an enclosing $n$ sided convex polygon with the origin $(0,0)$ inside. For such $n \ge 3$ planar lines $l^k := a_kx + b_ky - 1 = 0, \;  k \in \{ 1,\ldots,n \}$ not passing through the origin satisfying constraints \cref{eqn:detcon} and \cref{eqn:no1cons}, the boundary of the region $S$ generated by these lines, as defined in \cref{eqn:ss}, is exactly the graph $G$ formed by $S_g$. Since that graph $G$ is an enclosing $n$ sided convex polygon with the origin in its interior under the conditions of \cref{eqn:detcon} and \cref{eqn:no1cons}, it follows that the boundary of the region $S$ generated by these lines is an enclosing $n$ sided convex polygon with the origin in its interior.

\subsection{Proof of \cref{corollary}}
\label{subsec:corollary1}
This follows from the application of an affine coordinate transformation using $(\Bar{x},\Bar{y})$. 
More precisely, given coordinates $x\mbox{-}y$, consider the transformation $ (x^{\text{new}}, y^{\text{new}}) = (x - \Bar{x}, y - \Bar{y})$. 
%In addition to the global coordinate system $x\mbox{-} y$, we introduce a 
%consider the new local coordinate system $x^{\text{local}}\mbox{-} y^{\text{local}}$.
In the new coordinates, the origin $(0^{\text{new}}, 0^{\text{new}})$ is located at $(\Bar{x},\Bar{y})$. In addition, the representation of lines changes to
\begin{equation*}
    \left \{a(x-\Bar{x}) + b(y-\Bar{y}) - 1 = 0: a,b \in \mathbb{R} \land  a^2+b^2 \ne 0 \right \}.
    % \label{representation}
\end{equation*}

From \cref{theorem:1} it follows that \cref{eqn:detcon} and \cref{eqn:no1cons} are sufficient to guarantee that the boundary of the region $ S := \left \{(x^{\text{new}},y^{\text{new}}): \bigwedge\limits_{k=1}^n a_kx^{\text{new}} + b_ky^{\text{new}} - 1 \le 0 \right \} $ is an enclosing $n$ sided convex polygon with the origin $(0^{\text{new}}, 0^{\text{new}})$ in its interior. 
%Since $a(x-\Bar{x}) + b(y-\Bar{y}) - 1 = 0$ in the global system and $ax^{\text{local}} + by^{\text{local}} - 1 = 0$ in the local system refer to an identical line, it follows that the region $ S_2 := \left \{(x,y): \bigwedge\limits_{k=1}^n a_k(x-\Bar{x}) + b_k(y-\Bar{y}) - 1 \le 0 \right \} = S_1 $. 
By changing back to the original coordinate system, the result follows. 

%corollary\footnote{\magenta{We can abbreviate this part. It is very clear.}}.%, since $S_1 = S_2$, \cref{eqn:detcon} and \cref{eqn:no1cons} are sufficient to guarantee the boundary of the region $S_2$ is an enclosing $n$ sided convex polygon with the point $(\Bar{x},\Bar{y})$ inside.

%end of all proofs

% \begin{figure}
% \centering
%   \includegraphics[width=0.8\linewidth]{polygon2.pdf}
% \caption{Characterization of each point $\bar{\mathbf{x}}$ in the planning space}
% \label{fig:point}
% \end{figure}

\section*{Acknowledgment}
\medskip

This work was supported by the National Science Foundation under Grants CMMI-2138612. Any opinions, findings and conclusions or recommendations expressed in this paper are those of the authors and do not reflect the views of NSF. The authors would like to thank Lin Li, Paul Lathrop from UC San Diego and Jun Xiang from San Diego State University for their valuable comments and discussions.

% \blue{adding tase refs} 

% formal sort-order isomorphism composite-map

% fig change confidece region to prs

% alg3 loop

% reviewers and editors

\ifCLASSOPTIONcaptionsoff
  \newpage
\fi



%\bibliographystyle{IEEEtran}

\bibliographystyle{IEEEtran}
\bibliography{main}

\end{document}